\documentclass[twoside,11pt]{article}

\usepackage{jmlr2e}
\usepackage{microtype}
\usepackage{hyperref}
\usepackage[utf8]{inputenc} % allow utf-8 input
\usepackage[T1]{fontenc}    % use 8-bit T1 fonts

\usepackage{booktabs}       % professional-quality tables
\usepackage{amsfonts}       % blackboard math symbols
\usepackage{nicefrac}       % compact symbols for 1/2, etc.
\usepackage{microtype}      % microtypography
\usepackage{mathtools}
\usepackage{thmtools,thm-restate}
\usepackage{latexsym,amsmath,amssymb,bm}
\usepackage{graphicx}
\usepackage[font=footnotesize, skip=5pt]{caption}
\usepackage[font=footnotesize, skip=5pt]{subcaption}
\usepackage{times}
\usepackage{amsthm}
\usepackage[retainorgcmds]{IEEEtrantools}
\usepackage{mdwlist}
\usepackage{ctable}
\usepackage{enumitem}
\usepackage{listings}
\usepackage{multirow}
\DeclarePairedDelimiter{\ceil}{\lceil}{\rceil}

\newtheorem{theorem}{Theorem}
\newtheorem{defn}{Definition}
\newtheorem{prop}{Proposition}
\newtheorem{lma}{Lemma}

\newtheorem{cor}{Corollary}

\DeclareMathOperator*{\argmin}{arg\,min}
\DeclareMathOperator*{\arginf}{arg\,inf}

\jmlrheading{}{2020}{}{}{}{}

\ShortHeadings{Towards a Unified Analysis of Random Fourier Features}{Li, Ton, Oglic and Sejdinovic}
\firstpageno{1}

\begin{document}

\title{Towards a Unified Analysis of Random Fourier Features}

\author{\name Zhu Li \email zhu.li@stats.ox.ac.uk \\
       \addr Department of Statistics\\
       University of Oxford\\
       Oxford, OX1 3LB, UK
       \AND
       \name Jean-Francois Ton \email jean-francois.ton@spc.ox.ac.uk \\
       \addr Department of Statistics\\
       University of Oxford\\
       Oxford, OX1 3LB, UK
       \AND
       \name Dino Oglic \email dino.oglic@kcl.ac.uk \\
       \addr Department of Informatics\\
       King's College London\\
       London, WC2R 2LS, UK
       \AND 
       \name Dino Sejdinovic \email dino.sejdinovic@stats.ox.ac.uk \\
       \addr Department of Statistics\\
       University of Oxford\\
       Oxford, OX1 3LB, UK}

\editor{}

\maketitle

\begin{abstract}
Random Fourier features is a widely used, simple, and effective technique for scaling up kernel methods. The existing theoretical analysis of the approach, however, remains focused on specific learning tasks and typically gives pessimistic bounds which are at odds with the empirical results. We tackle these problems and provide the first unified risk analysis of learning with random Fourier features using the squared error and Lipschitz continuous loss functions. In our bounds, the trade-off between the computational cost and the learning risk convergence rate is problem specific and expressed in terms of the regularization parameter and the \emph{number of effective degrees of freedom}. We study both the standard random Fourier features method for which we improve the existing bounds on the number of features required to guarantee the corresponding minimax risk convergence rate of kernel ridge regression, as well as a data-dependent modification which samples features proportional to \emph{ridge leverage scores} and further reduces the required number of features. As ridge leverage scores are expensive to compute, we devise a simple approximation scheme which provably reduces the computational cost without loss of statistical efficiency.
\end{abstract}

\begin{keywords}
  Kernel methods, random Fourier features, stationary kernels, kernel ridge regression, Lipschitz continuous loss, support vector machines, logistic regression, ridge leverage scores.
  %Kernel Methods, 
\end{keywords}

\section{Introduction}

Kernel methods are one of the pillars of machine learning \citep{Scholkopf01,scholkopf2004kernel}, as they give us a flexible framework to model complex functional relationships in a principled way and also come with well-established statistical properties and theoretical guarantees \citep{caponnetto2007optimal,steinwart2008support}. The key ingredient, known as \emph{kernel trick}, allows implicit computation of an inner product between rich feature representations of data through the kernel evaluation $k(x,x') = \langle \varphi(x),\varphi(x') \rangle_{\mathcal H}$, while the actual feature mapping $\varphi:\mathcal X\to \mathcal H$ between a data domain $\mathcal X$ and some high and often infinite dimensional Hilbert space $\mathcal H$ is never computed. However, such convenience comes at a price: due to operating on all pairs of observations, kernel methods inherently require computation and storage which is at least quadratic in the number of observations, and hence often prohibitive for large datasets. In particular, the kernel matrix has to be computed, stored, and often inverted. As a result, a flurry of research into scalable kernel methods and the analysis of their performance emerged \citep{rahimi2007random,mahoney2009cur,bach2013sharp,alaoui2015fast,rudi2015less,rudi2017generalization,rudi2017falkon,zhang2015divide}. Among the most popular frameworks for fast approximations to kernel methods are random Fourier features (RFF) due to \citet{rahimi2007random}.
The idea of random Fourier features is to construct an explicit feature map which is of a dimension much lower than the number of observations, but with the resulting inner product which approximates the desired kernel function $k(x,y)$. In particular, random Fourier features rely on Bochner's theorem \citep{Bochner32,rudin2017fourier} which tells us that any bounded, continuous and shift-invariant kernel is the Fourier transform of a bounded positive measure, called spectral measure. The feature map is then constructed using samples drawn from the spectral measure. Essentially, any kernel method can then be adjusted to operate on these explicit feature maps (i.e., primal representations), greatly reducing the computational and storage costs, while in practice mimicking performance of the original method.

Despite their empirical success, the theoretical understanding of statistical properties of random Fourier features is incomplete, and the question of how many features are needed, in order to obtain a method with performance provably comparable to the original one, remains without a definitive answer. Currently, there are two main lines of research addressing this question. The first line considers the approximation error of the kernel matrix itself~\citep[e.g., see][and references therein]{rahimi2007random,sriperumbudur2015optimal,sutherland2015error} and bases performance guarantees on the accuracy of this approximation. However, all of these works require $\Omega(n)$ features ($n$ being the number of observations), which translates to no computational savings at all and is at odds with empirical findings. Realizing that the approximation of kernel matrices is just a means to an end, the second line of research aims at directly studying the risk and generalization properties of random Fourier features in various supervised learning scenarios. Arguably, first such result is already in \citet{rahimi2009weighted}, where supervised learning with Lipschitz continuous loss functions is studied. However, the bounds therein still require a pessimistic $\Omega(n)$ number of features and cannot demonstrate the efficiency of random Fourier features theoretically. In \citet{bach2017equivalence}, the generalization properties are studied from a function approximation perspective, showing for the first time that fewer features could preserve the statistical properties of the original method, but in the case where a certain data-dependent sampling distribution is used instead of the spectral measure. These results also do not apply to kernel ridge regression and the mentioned sampling distribution is typically itself intractable.~\citet{avron2017random} study the empirical risk of kernel ridge regression and show that it is possible to use $o(n)$ features and have the empirical risk of the linear ridge regression estimator based on random Fourier features close to the empirical risk of the original kernel estimator, also relying on a modification to the sampling distribution. However, this result does not provide any learning risk convergence rates, and a tractable method to sample from a modified distribution is proposed for the Gaussian kernel only. A highly refined analysis of kernel ridge regression is given by \citet{rudi2017generalization}, where it is shown that $\Omega(\sqrt{n}\log n)$ features suffices for an optimal $\mathcal{O}(1/\sqrt{n})$ learning rate in a minimax sense \citep{caponnetto2007optimal}. Moreover, the number of features can be reduced even further if a data-dependent sampling distribution is employed. While these are groundbreaking results, guaranteeing computational savings without any loss of statistical efficiency, they require some technical assumptions that are difficult to verify. Moreover, to what extent the bounds can be improved by utilizing data-dependent distributions still remains unclear. Finally, it does not seem straightforward to generalize the approach of \citet{rudi2017generalization} to kernel support vector machines (SVM) and/or kernel logistic regression (KLR). Recently, \citet{sun2018but} have provided novel bounds for random Fourier features in the SVM setting, assuming the Massart's low noise condition and that the target hypothesis lies in the corresponding reproducing kernel Hilbert space. The bounds, however, require the sample complexity and the number of features to be exponential in the dimension of the instance space and this can be problematic for high dimensional instance spaces. The theoretical results are also restricted to the hinge loss (without means to generalize to other loss functions) and require optimized features.

In this paper, we address the gaps mentioned above by making the following contributions:
\begin{itemize}%[noitemsep]
\item We devise a simple framework for the unified analysis of generalization properties of random Fourier features, which applies to kernel ridge regression, as well as to kernel support vector machines and logistic regression. %Through studying the role of the regularization parameter $\lambda \geq 0$, we establish an explicit trade-off between computational cost and the (expected) risk convergence rate for counterparts of these kernel methods based on random Fourier features, which can both be expressed as a function of $\lambda$.\vspace{-3ex}
\item For the plain random Fourier features sampling scheme, we provide, to the best of our knowledge, the sharpest results on the number of features required. In particular, we show that already with $\Omega(\sqrt{n}\log d_{\mathbf{K}}^{\lambda})$ features, we obtain the same learning rate with kernel ridge regression in the minimax sense \citep{caponnetto2007optimal}, where $d_{\mathbf{K}}^{\lambda}$ corresponds to the notion of \emph{the number of effective degrees of freedom}~\citep{bach2013sharp} with $d_{\mathbf{K}}^{\lambda} \ll n$ and $\lambda \coloneqq \lambda(n)$ is the regularization parameter. In addition, $\Omega(1/\lambda)$ features is sufficient to ensure $\mathcal{O}(\sqrt{\lambda})$ learning risk rate in kernel support vector machines and kernel logistic regression.

\item In the case of a modified data-dependent sampling distribution, the so called \emph{empirical ridge leverage score distribution}, we demonstrate that $\Omega(d_{\mathbf{K}}^{\lambda})$ features suffice for the learning risk to converge at $\mathcal{O}(\lambda)$ rate in kernel ridge regression. Moreover, the same number of feature is sufficient to guarantee $\mathcal{O}(\sqrt{\lambda})$ risk convergece rate in kernel support vector machines and kernel logistic regression.

\item In our refined analysis of kernel ridge regression, we show that the excess risk convergence rate of the estimator based on random Fourier features can (depending on the decay rate of the spectrum of the kernel function) be upper bounded by $\mathcal{O}(\nicefrac{\log n}{n})$ or even $\mathcal{O}(\nicefrac{1}{n})$, which implies much faster convergence than the standard $\mathcal{O}(\nicefrac{1}{\sqrt{n}})$ rate featuring in most of previous bounds.

\item Similarly, under the low noise assumption, our refined analysis for Lipschitz continuous loss demonstrates that it is possible to achieve $\mathcal{O}(1/n)$ excess risk convergence rate. The required number of features can be $\Omega(\log n\log \log n)$ using the empirical leverage score distribution, or even constant number of features in some benign cases. To the best of our knowledge, this is the first result offering non-trivial computational savings for approximations in problems with Lipschitz loss functions.

\item Finally, as the empirical ridge leverage scores distribution is typically costly to compute, we give a fast algorithm to generate samples from the approximated empirical leverage distribution. Utilizing these samples one can significantly reduce the computation time during the in-sample prediction ($\mathcal{O}(n\log n \log\log n)$) and testing stages ($\mathcal{O}(\log n \log\log n)$). We also include a proof that gives a trade-off between the computational cost and the learning risk of the algorithm, showing that the statistical efficiency can be preserved while provably reducing the required computational cost.
\end{itemize}

\section{Background}\label{sec:background}
In this section, we provide some notation and preliminary results that will be used throughout the paper. Henceforth, we denote the Euclidean norm of a vector $\mathbf{a} \in \mathbb{R}^{n}$ with $\|\mathbf{a}\|_2$ and the operator norm of a matrix $A \in \mathbb{R}^{n_1 \times n_2}$ with $\|A\|_2$. Let $\mathcal{H}$ be a Hilbert space with $\langle \cdot,\cdot\rangle_{\mathcal{H}}$ as its inner product and $\|\cdot\|_{\mathcal{H}}$ as its norm. We use $\text{Tr}(\cdot)$ to denote the trace of an operator or a matrix. Given a measure $d\rho$, we use $L_2(d\rho)$ to denote the space of square-integrable functions with respect to $d\rho$.

\subsection{Supervised Learning with Kernels} \label{sec:supervised_learning}
We first briefly review the standard problem setting for supervised learning with kernel methods. Let $\mathcal{X}$ be an instance space, $\mathcal{Y}$ a label space, and $P(x,y)=P_xP(y \mid x)$ a joint probability density function on $\mathcal{X}\times \mathcal{Y}$ defining the relationship between an instance $x \in \mathcal{X}$ and a label $y \in \mathcal{Y}$. A training sample is a set of examples $\{(x_i,y_i)\}_{i=1}^n$ sampled independently from $P(x,y)$. The value $P_x$ is called the marginal distribution of an instance $x \in \mathcal{X}$. The goal of a supervised learning task defined with a kernel function $k$ (and the associated reproducing kernel Hilbert space $\mathcal{H}$) is to find a hypothesis %\footnote{Throughout the paper, we assume (without loss of generality) that our hypothesis space is the unit ball in a reproducing kernel Hilbert space $\mathcal{H}$, i.e., $\|f\|_{\mathcal{H}} \leq 1$. This is a standard assumption, characteristic to the analysis of random Fourier features~\citep[e.g., see][]{rudi2017generalization}}  
$f \colon \mathcal{X} \rightarrow \mathcal{Y}$ such that $f \in \mathcal{H}$ and $f(x)$ is a good estimate of the label $y \in \mathcal{Y}$ corresponding to a previously unseen instance $x \in \mathcal{X}$. While in regression tasks $\mathcal{Y} \subset \mathbb{R}$, in classification tasks it is typically the case that $\mathcal{Y}=\{-1, 1\}$. As a result of the representer theorem, an empirical risk minimization problem in this setting can be expressed as~\citep{Scholkopf01}
\begin{IEEEeqnarray}{rCl}
	\hat{f}^{\lambda} &&:= \argmin_{f \in\mathcal{H}}  \frac{1}{n}\sum_{i=1}^n l(y_i,f(x_i))+ \lambda \|f\|_{\mathcal{H}}^2 \nonumber \\ 
	&&=\argmin_{\alpha} \ \frac{1}{n}\sum_{i=1}^n l(y_i,(\mathbf{K}\alpha)_i) + \lambda \alpha^T \mathbf{K}\alpha, \label{main:krl_opm}
\end{IEEEeqnarray}
where $f=\sum_{i=1}^n \alpha_ik(x_i,\cdot)$ with $\alpha \in \mathbb{R}^n$, $l:\mathcal{Y}\times\mathcal{Y}\rightarrow \mathbb{R}_{+}$ is a loss function, $\mathbf{K}$ is the kernel matrix, and $\lambda$ is the regularization parameter. The hypothesis $\hat{f}^{\lambda}$ is an empirical estimator and its ability to capture the relationship between instances and labels given by $P$ is measured by the learning risk~\citep{caponnetto2007optimal}
\[\mathbb{E}_{P}(l_{\hat{f}^{\lambda}}) = \int_{\mathcal{X}\times \mathcal{Y}} l(y,\hat{f}^{\lambda}(x))dP(x,y).\]
where we use $l_{f}$ to denote $l(y,f(x))$. When the context is clear, we will omit $P$ from the expectation, i.e., writing $\mathbb{E}_{P}(l_{\hat{f}^{\lambda}})$ as $\mathbb{E}(l_{\hat{f}^{\lambda}})$.
The empirical distribution $P_n(x,y)$ is given by a sample of $n$ examples drawn independently from $P(x,y)$. The empirical risk is used to estimate the learning risk $\mathbb{E}(l_{\hat{f}^{\lambda}})$ and it is given by \[ \mathbb{E}_n(l_{\hat{f}^{\lambda}}) = \frac{1}{n} \sum_{i=1}^n l(y_i,\hat{f}^{\lambda}(x_i)).\] Similar to \citet{rudi2017generalization} and \citet{caponnetto2007optimal}, we will assume \footnote{The existence of $f_{\cal H}$ depends on the complexity of $\cal H$ which is related to the data distribution $P(y|x)$. For more details, please see \citet{caponnetto2007optimal} and \citet{rudi2017generalization}.} the existence of $f_\mathcal{H} \in \mathcal{H}$ such that $f_\mathcal{H} = \arginf_{f \in \mathcal{H}} \mathbb{E}(l_f)$. The assumption implies that there exists some ball of radius $R>0$ containing $f_\mathcal{H}$ in its interior. Our theoretical results do not require prior knowledge of this constant and hold uniformly over all finite radii. 
Furthermore, for all the estimators returned by the empirical risk minimization, we assume that they have bounded reproducing kernel Hilbert space norms. As a result, to simplify our derivations and constant terms in our bounds, we have (without loss of generality) assumed that all the estimators appearing in the rest of the manuscript are within the \textbf{unit ball} of our reproducing kernel Hilbert space.

Note that $\mathbb{E}(l_{f_{\mathcal{H}}})$ is the lowest learning risk one can achieve in the reproducing kernel Hilbert space $\mathcal{H}$. Hence, the theoretical studies of the estimator $\hat{f}^{\lambda}$ often concern how fast its learning risk $\mathbb{E}(l_{\hat{f}^{\lambda}})$ converges to $\mathbb{E}(l_{f_{\mathcal{H}}})$, in other words, how fast the excess risk $\mathbb{E}(l_{\hat{f}^{\lambda}}) - \mathbb{E}(l_{f_{\mathcal{H}}})$ converges to zero. In the remainder of the manuscript, we will refer the rate at which the excess risk converges to zero as the learning rate.

\subsection{Random Fourier Features}
Random Fourier features is a widely used, simple, and effective technique for scaling up kernel methods. The underlying principle of the approach is a consequence of Bochner's theorem~\citep{Bochner32}, which states that any bounded, continuous and shift-invariant kernel is the Fourier transform of
a bounded positive measure. This measure can be transformed/normalized into a probability measure which is typically called the spectral measure of the kernel. Assuming the spectral measure $d\tau$ has a density function $p(\cdot)$, the corresponding shift-invariant kernel can be written as  
\begin{align}
\label{main:krl_boc_dec}
\begin{aligned}
& k(x,y) =  \int_{\mathcal{V}}e^{-2\pi iv^{T}(x-y)}d\tau(v)= \int_{\mathcal{V}} \big(e^{-2\pi i v^{T}x}\big)\big(e^{-2\pi i v^{T}y}\big)^{*}p(v)dv  ,&
\end{aligned}
\end{align}
%\begin{IEEEeqnarray}{rCl}
%k(x,y)& =& \int_{\mathcal{V}}e^{-2\pi iv^{T}(x-y)}d\tau(v) \nonumber\\
%&=& \int_{\mathcal{V}} \big(e^{-2\pi i v^{T}x}\big)^{*} \big(e^{-2\pi i v^{T}y}\big)p(v)dv %\label{main:krl_boc_dec}
%\end{IEEEeqnarray}%Eq.(\ref{main:krl_boc_dec})
where $c^{*}$ denotes the complex conjugate of $c \in \mathbb{C}$. Typically, the kernel is real valued and we can ignore the imaginary part in this equation~\citep[e.g., see][]{rahimi2007random}. The principle can be further generalized by considering the class of kernel functions which can be decomposed as
\begin{IEEEeqnarray}{rCl}
k(x,y) = \int_{\mathcal{V}}z(v,x)z(v,y)p(v)dv,\label{main:krl_dec}
\end{IEEEeqnarray}
where $z \colon \mathcal{V}\times \mathcal{X}\rightarrow \mathbb{R}$ is a continuous and bounded function with respect to $v$ and $x$. %(i.e., there exists a constant $z_0>0$ such that $| z(v,x) | \leq z_0$ for all $v$ and $x$). 
The main idea behind random Fourier features is to approximate the kernel function by its Monte-Carlo estimate
\begin{IEEEeqnarray}{rCl}
\tilde{k}(x,y) = \frac{1}{s}\sum_{i=1}^sz(v_i,x)z(v_i,y),\label{main:krl_appx}
\end{IEEEeqnarray}
with the reproducing kernel Hilbert space $\tilde{\mathcal{H}}$ (note that in general $\tilde{\mathcal{H}} \nsubseteq \mathcal{H}$) and $\{v_i\}_{i=1}^s$ sampled independently from the spectral measure. In~\citet[Appendix A]{bach2017breaking}, it has been established that a function $f \in \mathcal{H}$ can be expressed as:~\footnote{It is not necessarily true that for any $g \in L_2(d\tau)$, there exists a corresponding $f \in \mathcal{H}$.}
\begin{IEEEeqnarray}{rCl}
f(x) = \int_{\mathcal{V}}g(v)z(v,x)p(v)dv \qquad (\forall x \in \mathcal{X}) \label{main:fun_appx}
\end{IEEEeqnarray}
where $g \in L_2(d\tau)$ is a real-valued function such that $\|g\|_{L_2(d\tau)}^2 < \infty$ and $\|f\|_{\mathcal{H}}= \min_{g} \|g\|_{L_2(d\tau)}$, where the minimum is taken over all possible decompositions of $f$. Thus, one can take an independent sample $\{v_i\}_{i=1}^s \sim p(v)$ (we refer to this sampling scheme as \emph{plain RFF}) and approximate a function $f \in \mathcal{H}$ at a point $x_j \in \mathcal{X}$ by 
\begin{align*}
   \tilde{f}(x_j) = \sum_{i=1}^s\alpha_i z(v_i,x_j) \coloneqq \mathbf{z}_{x_j}(\mathbf{v})^T\alpha \quad \text{with} \quad \alpha \in \mathbb{R}^s .
\end{align*}
In standard estimation problems, it is typically the case that for a given set of instances $\{x_i\}_{i=1}^n$ one approximates $\mathbf{f}_x = [f(x_1),\cdots,f(x_n)]^T$ by 
\begin{align*}
  \tilde{\mathbf{f}}_x =[\mathbf{z}_{x_1}(\mathbf{v})^T\alpha,\cdots,\mathbf{z}_{x_n}(\mathbf{v})^T\alpha]^T \coloneqq \mathbf{Z}\alpha  ,
\end{align*}
where $\mathbf{Z}\in \mathbb{R}^{n\times s}$ with $\mathbf{z}_{x_j}(\mathbf{v})^T$ as its $j$th row.

As the latter approximation is simply a Monte Carlo estimate, one could also select an importance weighted probability density function $q(\cdot)$ and sample features $\{v_i\}_{i=1}^s$ from $q$ (we refer to this sampling scheme as \emph{weighted RFF}). %, i.e., $\{v_i\}_{i=1}^s \sim q(v)$. 
Then, the function value $f(x_j)$ can be approximated by 
\begin{align*}
  \tilde{f}_{q}(x_j) = \sum_{i=1}^s\beta_i z_q(v_i,x_j)\coloneqq \mathbf{z}_{q,x_j}(\mathbf{v})^T\beta  ,
\end{align*}
with $z_q(v_i,x_j)= \sqrt{p(v_i)/q(v_i)}z(v_i,x_j)$ and $\mathbf{z}_{q,x_j}(\mathbf{v}) = [z_q(v_1,x_j),\cdots,z_q(v_s,x_j)]^T$. Hence, a Monte-Carlo estimate of $\mathbf{f}_x$ can be written in the matrix form as $\tilde{\mathbf{f}}_{q,x} = \mathbf{Z}_{q}\beta$, where $\mathbf{Z}_q\in \mathbb{R}^{n\times s}$ with $\mathbf{z}_{q,x_j}(\mathbf{v})^T$ as its $j$th row.

Let $\tilde{\mathbf{K}}$ and $\tilde{\mathbf{K}}_q$ be Gram-matrices with entries $\tilde{\mathbf{K}}_{ij} = \tilde{k}(x_i,x_j)$ and $\tilde{\mathbf{K}}_{q,ij} = \tilde{k}_q(x_i,x_j)$ such that \[\tilde{\mathbf{K}}= \frac{1}{s}\ \mathbf{Z}\mathbf{Z}^T \qquad \wedge \qquad \tilde{\mathbf{K}}_q = \frac{1}{s}\ \mathbf{Z}_{q}\mathbf{Z}_{q}^T.\] If we now denote the $j$th column of $\mathbf{Z}$ by $\mathbf{z}_{v_j}(\mathbf{x})$ and the $j$th column of $\mathbf{Z}_q$ by $\mathbf{z}_{q,v_j}(\mathbf{x})$, then the following equalities can be derived easily from Eq.~(\ref{main:krl_appx}):
%\begin{IEEEeqnarray}{rCl}
\begin{align*}
\begin{aligned}
%\mathbb{E}(\tilde{k}(x,y)) = k(x,y), \  
 \mathbb{E}_{v\sim p}(\tilde{\mathbf{K}}) = \mathbf{K} =  \mathbb{E}_{v\sim q}(\tilde{\mathbf{K}}_q) \quad \wedge \quad \mathbb{E}_{v \sim p}\big[\mathbf{z}_{v}(\mathbf{x})\mathbf{z}_{v}(\mathbf{x})^T\big]= \mathbf{K} =  \mathbb{E}_{v\sim q}\big[\mathbf{z}_{q,v}(\mathbf{x})\mathbf{z}_{q,v}(\mathbf{x})^T\big]. %\nonumber
\end{aligned}
\end{align*}
%\end{IEEEeqnarray}
%\subsection{The Empirical Leverage Function}\label{bg:emp_ridge}

Sampling features from the importance weighted probability density function $q(\cdot)$ has led to much interest in literature \citep{bach2017equivalence,alaoui2015fast,avron2017random,rudi2017generalization} as it often leads to huge computation savings. In particular, an importance weighted density function based on the notion of \emph{ridge leverage scores} is introduced in \citet{alaoui2015fast} for landmark selection in the Nystr{\"o}m method~\citep{Nystrom30,Smola00,Williams01}. For landmarks selected using that sampling strategy,~\citet{alaoui2015fast} establish a sharp convergence rate of the low-rank estimator based on the Nystr\"om method. This result motivates the pursuit of a similar notion for random Fourier features. Indeed,~\citet{bach2017equivalence} propose a leverage score function based on an integral operator defined using the kernel function and the marginal distribution of a data-generating process. Building on this work,~\citet{avron2017random} propose the ridge leverage function with respect to a fixed input dataset, i.e., %More specifically, data-dependent ridge leverage scores can be defined as~\citep{avron2017random}
\begin{IEEEeqnarray}{rCl}
l_{\lambda}(v) = p(v)\mathbf{z}_{v}(\mathbf{x})^{T}(\mathbf{K}+n\lambda\mathbf{I})^{-1}\mathbf{z}_{v}(\mathbf{x}) . \label{main:lev_fun}
\end{IEEEeqnarray}
From our assumption on the decomposition of a kernel function, it follows that there exists a constant $z_0$ such that $| z(v,x) | \leq z_0$ (for all $v$ and $x$) and $\mathbf{z}_{v}(\mathbf{x})^{T}\mathbf{z}_{v}(\mathbf{x}) \leq nz_0^2$. We can now deduce the following inequality using a result from~\citet[Proposition 4]{avron2017random}:
\begin{align*}
\begin{aligned}
l_{\lambda}(v)\leq p(v)\frac{z_0^2}{\lambda} .
\end{aligned}
\end{align*}
The function $l_{\lambda}(v)$ is important in the sense that it is related to the effective number of parameters in the following sense:
\begin{align*}
\begin{aligned}
\int_{\mathcal{V}}l_{\lambda}(v)dv = \text{Tr}\big[\mathbf{K}(\mathbf{K}+n\lambda \mathbf{I})^{-1} \big]:= d_{\mathbf{K}}^{\lambda},
\end{aligned}
\end{align*}
where $d_{\mathbf{K}}^{\lambda}$ is known for implicitly determining the number of independent parameters in a learning problem and, thus, it is called the \emph{effective dimension of the problem}~\citep{caponnetto2007optimal} or the \emph{number of effective degrees of freedom}~\citep{bach2013sharp,hastie2017generalized}.

We can now observe that $q^*(v) = l_{\lambda}(v)/d_{\mathbf{K}}^{\lambda}$ is a probability density function. In \citet{avron2017random}, it has been established that sampling according to $q^*(v)$ requires fewer Fourier features compared to the standard spectral measure sampling. We refer to $q^*(v)$  as the \emph{empirical ridge leverage score distribution} and, in the remainder of the manuscript, refer to this sampling strategy as \emph{leverage weighted RFF}.

\subsection{Rademacher Complexity}
To characterize the performance of a learning algorithm, we need to take into account the complexity of its hypothesis space. Below, we first introduce a particular measure of the complexity over function spaces known as \textit{Rademacher complexity} \citep{bartlett2002rademacher}. Then, we give two lemmas that demonstrate how Rademacher complexity of a reproducing kernel Hilbert space can be linked to the corresponding kernel and how the excess risk can be computed via Rademacher complexity. 

\begin{defn}\label{def:rade}
Suppose that $\{x_1\cdots,x_n\}$ are independent samples selected according to $P_x$. Let $\mathcal{H}$ be a class of functions mapping $\mathcal{X}$ to $\mathbb{R}$. Then, the random variable known as the \textit{empirical Rademacher complexity} is defined as 
\begin{IEEEeqnarray}{rCl}
\hat{R}_n(\mathcal{H}) = \mathbb{E}_{\sigma}\Bigg[\sup_{f\in\mathcal{H}}\Bigg|\frac{2}{n}\sum_{i=1}^n\sigma_if(x_i)\Bigg|\mid x_1,\cdots,x_n   \Bigg]\nonumber
\end{IEEEeqnarray}
where $\sigma_1,\cdots,\sigma_n$ are independent uniform $\{\pm 1\}$-valued random variables. The corresponding \textit{Rademacher complexity} is then defined as the expectation of the empirical Rademacher complexity $$R_n(\mathcal{H}) = \mathbb{E}\Big[\hat{R}_n(\mathcal{H})\Big].$$  
\end{defn} 

The following lemma provides the Rademacher complexity for a certain RKHS with kernel $k$.

\begin{lma}\citep{bartlett2002rademacher}\label{apn:rad_rkhs}
Let $\mathcal{H}_0$ be the unit ball of the RKHS $\mathcal{H}$ associated with kernel $k$, centered at the origin. Then, we have that $R_n(\mathcal{H}_0) \leq (1/n)\mathbb{E}_X\sqrt{\text{Tr}(\mathbf{K})}$, where $\mathbf{K}$ is the Gram matrix for kernel $k$ over an independent and identically distributed sample $X=\{x_1,\cdots,x_n\}$.
\end{lma}

Lemma \ref{apn:risk_rad} states that the expected excess risk convergence rate of a particular estimator in $\mathcal{H}$ not only depends on the number of data points, but also on the complexity of $\mathcal{H}$ and how it interacts with the loss function.

\begin{lma}\citep[Theorem 8]{bartlett2002rademacher}\label{apn:risk_rad}
Let $\{x_i,y_i\}_{i=1}^n$ be i.i.d samples from $P$ and let $\mathcal{H}$ be the space of functions mapping from $\mathcal{X}$ to $\mathbb{R}$. Denote a loss function with $l:\mathcal{Y}\times \mathbb{R} \rightarrow [0,1]$ and recall the learning risk function for all $f\in \mathcal{H}$ is $\mathbb{E}(l_f)$, together with the corresponding empirical risk function $\mathbb{E}_n(l_f) = (1/n)\sum_{i=1}^nl(y_i,f(x_i))$. Then, for a sample of size $n$, for all $f\in \mathcal{H}$ and $\delta \in (0,1)$, with probability $1-\delta$, we have that
\begin{IEEEeqnarray}{rCl}
\mathbb{E}(l_f)\leq \mathbb{E}_n(l_f) + R_n(l \circ \mathcal{H}) +\sqrt{\frac{8\log(2/\delta)}{n}}\nonumber
\end{IEEEeqnarray} 
where $l \circ \mathcal{H} = \{(x,y) \rightarrow l(y,f(x))-l(y,0) \mid f\in \mathcal{H}\}$.
\end{lma}
Note that the risk bound is given by the Rademacher complexity term $R_n(l \circ \mathcal{H})$ defined on the transformed space $l \circ \mathcal{H}$ which is obtained via composition of $f \in \mathcal{H}$ and the loss function $l$. This term is, in general, different from $R_n(\mathcal{H})$ but in the case when $l$ is Lipschitz continuous the two can be related by following derivations in \cite{bartlett2002rademacher}.

\subsection{Local Rademacher Complexity}
When characterizing the finite sample behaviour of learning risk, Rademacher complexity introduced in the previous section does not typically give the optimal convergence rates. This is because Rademacher complexity considers the behaviour of the empirical learning risk over the whole hypothesis space, while the estimator returned by the regression is typically in a neighbourhood around the optimal estimator. Hence, in our refined analysis, we rely on the so called \textit{local Rademacher complexity}. Before illustrating the concept, we first recall that given $f \in \mathcal{H}$, we denote its expectation and finite sample average with $\mathbb{E}(f)$ and $\mathbb{E}_n(f)$, respectively. The notion of local Rademacher complexity is typically introduced  via the so called sub-root function. Below, we first give the definition and a useful property of the sub-root function. We then review a theorem that relates the notion of local Rademacher complexity and learning risk.
\begin{defn}\label{apn: sub-root}
Let $\psi: [0,\infty) \rightarrow [0,\infty)$ be a function. Then, $\psi(r)$ is called the \textit{sub-root} function if, for all $r>0$, $\psi(r)$ is non-decreasing and $\nicefrac{\psi(r)}{r}$ is non-increasing.
\end{defn}
A sub-root function has the following property.
\begin{lma}\citep[Lemma 3.2]{bartlett2005local} \label{apn:sub-pro}
If $\psi(r)$ is a sub-root function, then $\psi(r) = r$ has a unique positive solution $r^*$. In addition, we have that $\psi(r) \leq r$ for all $r >0$ if and only if $r^* \leq r$.
\end{lma}

In Lemma \ref{apn:risk_rad}, we can see that the difference between learning risk $\mathbb{E}(l_f)$ and empirical learning risk $\mathbb{E}_n(l_f)$ is upper bounded by $\mathcal{O}(\nicefrac{1}{\sqrt{n}})$. This rate can be further improved with local Rademacher complexity. The reason for the slow learning rate is because we bound the difference between $\mathbb{E}(l_f)$ and $\mathbb{E}_n(l_f)$ using the global Rademacher complexity. Inspecting the definition of $R_n(\mathcal{H})$ (Definition \ref{def:rade}), we can see that $R_n(\mathcal{H})$ is defined by considering the whole hypothesis space as we are taking $\sup$ across all functions in $\mathcal{H}$. However, as discussed before, learning algorithms typically return functions that are in the neighbourhood around the optimal estimator. Hence, using $R_n(\mathcal{H})$ unnecessarily enlarges the space that we are interested in. 

Since empirical estimators returned by learning algorithms often have low learning risk and hence, low variance, we could instead consider the alternative space $\mathcal{H}_r := \{f \in \mathcal{H}: \mathbb{E}(f^2) \leq r\}$ for some given value $r \in \mathbb{R}$. In this way, we greatly reduce the complexity of the function space at hand and can provide a sharper convergence rate. The following results from \cite{bartlett2005local} details how this idea can be used to describe the learning risk behaviour. 
% \begin{lma}\citep[Theorem 3.3]{bartlett2005local} \label{apn:exp_local_rade}
% Let $\mathcal{H}$ be a class of functions with ranges in $[a,b]$ and $B > 0$ be a constant. Assume that there is a functional $T : \mathcal{H} \rightarrow \mathbb{R}^+$ such that $\text{Var}(f) \leq T(f) \leq B\mathbb{E}_P(f)$ for all $f \in \mathcal{H}$ and for $\alpha \in [0,1]$, $T(\alpha f) \leq \alpha^2 T(f)$. Suppose $\psi$ is a sub-root function with fixed point $r^*$ and satisfies for all $r \geq r^*$, \[\psi(r) \geq B R_n\{f\in \text{star}(\mathcal{H},0): T(f) \leq r \},\]
% and
% \begin{align*}
%     \mathrm{star}(\mathcal{H}, f_0)=\{ f_0 + \alpha (f-f_0)\ \mid\ f \in \mathcal{H}\ \wedge\ \alpha \in [0,1] \} .
% \end{align*}
% then we have for all $D > 1$ and $\delta \in (0,1)$, with probability greater than $1-\delta$,
% \begin{IEEEeqnarray}{rCl}
% \mathbb{E}_P(f) \leq \frac{D}{D-1} \mathbb{E}_n(f) + \frac{6D}{B}r^* + \frac{c_1 }{n}\log\frac{1}{\delta}, ~~~\forall f\in \mathcal{H}.  \nonumber
% \end{IEEEeqnarray} 
% \end{lma}

\begin{lma}\citep[Theorem 4.1]{bartlett2005local}\label{apn:emp_local_rade}
Let $\mathcal{H}$ be a class of functions with bounded ranges and assume that there is some constant $B > 0 $ such that for all $f \in \mathcal{H}$, $\mathbb{E}(f^2) \leq B \mathbb{E}(f)$. Let $\hat{\psi}_n$ be a sub-root function and let $\hat{r}^*$ be the fixed point of $\hat{\psi}_n$, i.e., $\hat{\psi}_n(\hat{r}^*) = \hat{r}^*$. Fix any $\delta \in (0,1)$, and assume that for any $r \geq \hat{r}^*$,\[\hat{\psi}_n(r) \geq c_1 \hat{R}_n\{f\in \mathrm{star}(\mathcal{H},0) \mid \mathbb{E}_n(f^2) \leq r\} + \frac{c_2}{n}\log\frac{1}{\delta} \ ,\]
where
\begin{align*}
    \mathrm{star}(\mathcal{H}, f_0)=\{ f_0 + \alpha (f-f_0)\ \mid\ f \in \mathcal{H}\ \wedge\ \alpha \in [0,1] \} .
\end{align*}
Then for all $D > 1$ and $f \in \mathcal{H}$, with probability greater than $1-\delta$,
\[\mathbb{E}(f) \leq \frac{D}{D-1}\mathbb{E}_n(f) + \frac{6D}{B}\hat{r}^* + \frac{c_3}{n}\log\frac{1}{\delta} \ ,\]
where $c_1$, $c_2$ and $c_3$ are some constants.
\end{lma}
Note that this theorem bounds the difference between $\mathbb{E}(f)$ and $\mathbb{E}_n(f)$. We will show later (Section \ref{sec:sq_ref_ana}), with some simple transformation, that this result can be used to bound the difference between the learning and empirical risk.

We have seen that in the above theorem, we can use the fixed point of the sub-root function to upper bound the learning rate. However, it is not clear how to obtain the explicit formula for the fixed point. Fortunately, in the setting of learning with kernel $k$ and the corresponding reproducing kernel Hilbert space, we can derive such results. The following lemma provides us with an upper bound on local Rademacher complexity through the eigenvalues of the Gram matrix.
\begin{lma}\citep[Lemma 6.6]{bartlett2005local}\label{apn:local_kernel}
Let $k$ be a positive definite kernel function with reproducing kernel Hilbert space $\mathcal{H}$ and let $\hat{\lambda}_1 \geq \cdots \geq \hat{\lambda}_n$ be the eigenvalues of the normalized Gram-matrix $(1/n)\mathbf{K}$. Then, for all $r >0$ and $f \in \mathcal{H}$, \[\hat{R}_n\{f\in \mathcal{H} \mid \mathbb{E}_n(f^2) \leq r\} \leq \left(\frac{2}{n}\sum_{i=1}^n\min\{r,\hat{\lambda}_i\}\right)^{1/2}.\]
\end{lma}

\section{Theoretical Analysis}\label{sec:theo_analysis}
In this section, we provide a unified analysis of the generalization properties of learning with random Fourier features. Our analysis is split into two cases/settings: \emph{i}) we start with a bound for learning with the squared error loss function (Section \ref{main:krr}) and \emph{ii}) then extend these results to learning problems with Lipschitz continuous loss functions (Section \ref{main:orig_sampling}). Before we present our analysis, we first enumerate the assumptions that we made in Section \ref{sec:background}:
\begin{itemize}
    \item[$1.$] For a learning problem with kernel $k$ (and corresponding reproducing kernel Hilbert space $\mathcal{H}$) defined as in Eq.~(\ref{main:krl_opm}), we assume that $f_{\mathcal{H}}=\arginf_{f \in \mathcal{H}} \mathbb{E}(l_f)$ always exists;
    \item[$2.$] We assume that the function $f_{\mathcal{H}}$ has bounded RKHS norm, and hence, without loss of generality, we restrict our analysis to the unit ball of $\mathcal{H}$, i.e., $\|f\|_{\mathcal{H}} \leq 1$; 
    \item[$3.$] We assume that the kernel $k$ has the decomposition as in Eq.~(\ref{main:krl_dec}) with $|z(w,x)| < z_0 \in (0,\infty)$;
    \item[$4.$] For kernel $k$, denote with $\lambda_1\geq \cdots \geq\lambda_n$ the eigenvalues of the kernel matrix $\mathbf{K}$. We assume that the regularization parameter satisfies $0 \leq n\lambda \leq \lambda_1$. 
\end{itemize}
For Assumption 4, intuitively speaking, it requires the signal $\lambda_1$ to be stronger than the added regularization term $n\lambda$. For example, the in-sample prediction of a kernel ridge regression problem is $\mathbf{K}(\mathbf{K}+ n\lambda I)^{-1}Y$. The largest eigenvalue of $\mathbf{K}(\mathbf{K}+ n\lambda I)^{-1}$ is $\nicefrac{\lambda_1}{(\lambda_1 + n\lambda)}$. If $n\lambda > \lambda_1$, then the in-sample prediction is essentially dominated by $n\lambda$ which leads to under-fitting.

Throughout the following analysis, we will use the above assumptions. Hence, for the sake of clarity, we will not repeat them, unless problem-specific clarifications are required. 

\subsection{Learning with the Squared Error Loss}\label{main:krr}
In this section, we consider learning with the squared error loss, i.e., $l(y,f(x)) = (y-f(x))^2$. For this particular loss function, the optimization problem from Eq.~(\ref{main:krl_opm}) is known as \emph{kernel ridge regression} (KRR). We make the following assumption specific for the KRR problem.
\begin{itemize}
    \item[A.$1$] $y = f^*(x) + \epsilon$, where $\mathbb{E}(\epsilon) = 0$ and $\text{Var}(\epsilon) = \sigma^2$. Furthermore, assume $y$ is bounded, i.e., $|y| \leq y_0$;
\end{itemize}
Note that $f^*$ in Assumption A.1 may be different from $f_{\mathcal{H}}$ as $f^*$ is not necessarily contained in our hypothesis space $\mathcal{H}$.

In the random Fourier feature setting, the KRR problem can be reduced to solving a linear system $(\mathbf{K} + n\lambda \mathbf{I})\alpha = Y$, with $Y = [y_1,\cdots,y_n]^T$. %Thus, to find an optimal solution to a kernel ridge regression problem one needs to either solve a linear system of equations or invert a positive definite matrix. 
Typically, an approximation of the kernel function based on random Fourier features is employed in order to effectively reduce the computational cost and scale kernel ridge regression to problems with a large number of examples. More specifically, for a vector of observed labels $Y$ the goal is to find a hypothesis $\tilde{\mathbf{f}}_x=\mathbf{Z}\beta$ that minimizes  $\|Y-\tilde{\mathbf{f}}_x\|_2^2$ while having good generalization properties. In order to achieve this, one needs to control the complexity of hypotheses defined by random Fourier features and avoid over-fitting. Hence, we would like to find out the norm of the function $\tilde{f} \in \tilde{\mathcal{H}}$ for the purpose of regularization. The next proposition gives an upper bound of its norm and the proof is in Section \ref{pf:func_norm}.
\begin{prop} \label{apn:func_norm}
Assume that the reproducing kernel Hilbert space $\mathcal{H}$ with kernel $k$ admits a decomposition as in Eq.~(\ref{main:krl_dec}) and denote by $\tilde{\mathcal{H}} \coloneqq \{\tilde{f} \mid \tilde{f} = \sum_{i=1}^s\alpha_i z(v_i,\cdot), \alpha_i \in \mathbb{R}\}$ the reproducing kernel Hilbert space with kernel $\tilde{k}$ (see Eq.~(\ref{main:krl_appx})). Then, for all $ \tilde{f} \in \tilde{\mathcal{H}}$ it holds that $\|\tilde{f}\|_{\tilde{\mathcal{H}}}^2 \leq s\|\alpha\|_2^2$.
\end{prop} 

According to Proposition \ref{apn:func_norm}, the learning problem with random Fourier features and the squared error loss can be cast as
\begin{IEEEeqnarray}{rCl}
\beta_{\lambda} := \argmin_{\beta \in \mathbb{R}^{s} }~~ \frac{1}{n}\|Y-\mathbf{Z}_q\beta\|_2^2 + \lambda s\|\beta\|_2^2.\label{main:krr_appx_opm}
\end{IEEEeqnarray}
This is simply a linear ridge regression problem in the space of Fourier features. We denote the optimal hypothesis function returned by Eq.~(\ref{main:krr_appx_opm}) to be $\tilde{f}_{\beta}^{\lambda}$. The function can be parameterized by $\beta_{\lambda}$ and its in-sample evaluation is given by $\tilde{\mathbf{f}}_{\beta}^{\lambda} = \mathbf{Z}_q\beta_{\lambda}$, where $\beta_{\lambda}= (\mathbf{Z}_q^T\mathbf{Z}_q + ns\lambda\mathbf{I})^{-1}\mathbf{Z}_q^TY$. Since $\mathbf{Z}_q \in \mathbb{R}^{n\times s}$, the computational and space complexities are $\mathcal{O}(s^3+ns^2)$ and $\mathcal{O}(ns)$. Thus, significant savings can be achieved using estimators with $s \ll n$. To assess the effectiveness of such estimators, it is important to understand the relationship between the excess learning risk and the choice of $s$.

\subsubsection{Worst Case Analysis}
%\vspace{-.5em}

In this section, we provide a bound on the required number of random Fourier features with respect to the worst case of the corresponding kernel ridge regression problem in the minimax rate sense ($\mathcal{O}(\nicefrac{1}{\sqrt{n}})$). The following theorem gives a general result while taking into account both the number of features $s$ and a sampling strategy for selecting them.

\begin{theorem}\label{main:krr_risk}
Under Assumption A.1, let $\tilde{l}: \mathcal{V} \rightarrow \mathbb{R}$ be a measurable function such that $\tilde{l}(v)\geq l_{\lambda}(v)$ ($\forall v \in \mathcal{V}$) and $d_{\tilde{l}} = \int_{\mathcal{V}}\tilde{l}(v)dv < \infty$. Suppose $\{v_i\}_{i=1}^s$ are sampled independently from the probability density function $q(v) = \tilde{l}(v)/d_{\tilde{l}}$. If \[s\ \geq \ 5d_{\tilde{l}}\log\frac{16d_{\mathbf{K}}^{\lambda}}{\delta},\] then for all $ \delta \in (0,1)$, with probability $1-\delta$, the excess risk of $\tilde{f}_{\beta}^{\lambda}$ can be upper bounded as
%\begin{IEEEeqnarray}{rCl}
\begin{align}
\mathbb{E}(l_{\tilde{f}_{\beta}^{\lambda}})-\mathbb{E}(l_{f_{\mathcal{H}}})\ \leq \ 4\lambda +\mathcal{O}\left(\frac{1}{\sqrt{n}}\right)+\mathbb{E}(l_{\hat{f}^{\lambda}})- \mathbb{E}(l_{f_{\mathcal{H}}}). \label{krr_risk_up}
\end{align}
%\end{IEEEeqnarray}
\end{theorem}

Theorem \ref{main:krr_risk} %links the excess risk of $f_{\beta}^{\lambda}$ to the excess risk of $\hat{f}^{\lambda}$ (the estimator from RKHS $\mathcal{H}$). In addition, it 
expresses the trade-off between the computational and statistical efficiency through the regularization parameter $\lambda$, the effective dimension of the problem $d_{\mathbf{K}}^{\lambda}$, and the normalization constant $d_{\tilde{l}}$ of the sampling distribution. The decay rate of regularization parameter is used as a key quantity \citep{caponnetto2007optimal,rudi2017generalization} and its choice can be linked to the complexity of the target regression function $f^*(x)=\int y d\rho(y \mid x)$. In particular, 
%Theorem \ref{main:krr_risk} derives the risk bound in a worst case scenario, i.e. we only assume that $f_{\mathcal{H}}$ exists. In this case,
\citet{caponnetto2007optimal} have shown that the minimax risk convergence rate for kernel ridge regression is $\mathcal{O}(\nicefrac{1}{\sqrt{n}})$. Setting $\lambda \propto \nicefrac{1}{\sqrt{n}}$, we observe that the estimator $\tilde{f}_{\beta}^{\lambda}$ attains the worst case minimax rate of kernel ridge regression.

As a consequence of Theorem~\ref{main:krr_risk}, we have the following bounds on the number of required features for the two strategies: \emph{leverage weighted} RFF (Corollary 1) and \emph{plain} RFF (Corollary 2).
\begin{cor}\label{main:krr_risk_cor1}
If the probability density function from Theorem~\ref{main:krr_risk} is the empirical ridge leverage score distribution $q^*(v)$, then the upper bound on the learning risk from Eq.~(\ref{krr_risk_up}) holds for all $s \geq  5d_{\mathbf{K}}^{\lambda}\log\frac{16d_{\mathbf{K}}^{\lambda}}{\delta}$.
\end{cor}
\begin{proof}
For Corollary \ref{main:krr_risk_cor1}, we set $\tilde{l}(v) = l_{\lambda}(v)$ and deduce \[d_{\tilde{l}} = \int_{\mathcal{V}}l_{\lambda}(v)dv = d_{\mathbf{K}}^{\lambda}.\] 
\end{proof}

Theorem \ref{main:krr_risk} and Corollary \ref{main:krr_risk_cor1} have several implications on the choice of $\lambda$ and $s$. First, we could pick $\lambda \in \mathcal{O}(n^{-1/2})$ that implies the worst case minimax rate for kernel ridge regression~\citep{caponnetto2007optimal,rudi2017generalization,bartlett2005local} and observe that in this case $s$ is proportional to $d_{\mathbf{K}}^{\lambda}\log d_{\mathbf{K}}^{\lambda}$. As $d_{\mathbf{K}}^{\lambda}$ is determined by the learning problem (i.e., the marginal distribution $P_x$), we can consider several different cases. In the best case, where the number of positive eigenvalues is finite, implying that $d_{\mathbf{K}}^{\lambda}$ does not grow with $n$, we then have that even with a constant number of features, we are able to achieve the $\mathcal{O}(1/\sqrt{n})$ learning rate. Next, if the eigenvalues of $\mathbf{K}$ exhibit a geometric/exponential decay, i.e., $\lambda_i \propto R_0r^{i}$ with a constant $R_0 > 0$ (this can happen in scenario where we have a Gaussian kernel and a sub-Gaussian marginal distribution $P_x$), we then know that $d_{\mathbf{K}}^{\lambda} \leq \log(R_0/\lambda)$ \citep{bach2017equivalence}, implying $s\geq \log n\log\log n$. Hence, significant savings can be obtained with $\mathcal{O}(n\log^4n+\log^6n)$ computational and $\mathcal{O}(n\log^2n)$ storage complexities of linear ridge regression over random Fourier features, as opposed to $\mathcal{O}(n^{3})$ and $\mathcal{O}(n^2)$ costs (respectively) in the kernel ridge regression setting.

In the case of a slower decay (e.g., $\mathcal{H}$ is a Sobolev space of order $t\geq 1$) with $\lambda_i \propto R_0 i^{-2t}$, we have $d_{\mathbf{K}}^{\lambda} \leq (R_0/\lambda)^{1/(2t)}$ and $s\geq n^{1/(4t)}\log n$. Hence, even in this case, a substantial computational savings can be achieved. Furthermore, in the worst case with $\lambda_i$ close to $R_0 i^{-1}$, our bound implies that $s \geq \sqrt{n}\log n$ features is sufficient, recovering the result from \citet{rudi2017generalization}.

\begin{table}[t]
	\centering\fontsize{10}{20}\selectfont
	\begin{tabular}{l|c|c|c}
		\hline
		
		\hline
		
		\textsc{sampling scheme}&\textsc{spectrum} & \textsc{number of features} & \textsc{learning rate}\\\cline{1-4}
		
		\multirow{4}{8em}{\textsc{weighted rff}} & finite rank & $s \in \Omega(1)$& \multirow{4}{4em}{$\mathcal{O}(1/\sqrt{n})$}\\\cline{2-3}
		
		&$\lambda_i \propto A^{i}$ & $s \in \Omega (\log n \cdot \log \log n)$ & 	\\\cline{2-3}
		
		&$\lambda_i \propto i^{-2t} $ ($t\geq 1$) & $s \in \Omega (n^{1/2t} \cdot \log n)$ & 	\\\cline{2-3}
		
		&$\lambda_i \propto i^{-1}$ & $s \in \Omega (\sqrt{n} \cdot \log n)$ & 	\\\cline{2-3}
		
		\hline
		
		\hline
		\multirow{4}{6em}{\textsc{plain rff}} & finite rank & $s \in \Omega(\sqrt{n})$& \multirow{4}{4em}{$\mathcal{O}(1/\sqrt{n})$}\\\cline{2-3}
		
		&$\lambda_i \propto A^{i}$ & $s \in \Omega (\sqrt{n} \cdot \log \log n)$ & 	\\\cline{2-3}
		
		&$\lambda_i \propto i^{-2t} $ ($t\geq 1$) & $s \in \Omega (\sqrt{n} \cdot \log n)$& 	\\\cline{2-3}
		
		&$\lambda_i \propto  i^{-1}$ & $s \in \Omega (\sqrt{n} \cdot \log n)$ & 	\\\cline{2-3}
		
		\hline
		
		\hline
	\end{tabular}
	\caption{The trade-off in the worst case for the squared error loss.}\label{tab:squ-wor}
\end{table}

%Apart from providing an indication on how to choose the number of random Fourier features $s$, our results allow us to trade the statistical efficiency of a regression estimator (i.e., the convergence rate of the expected risk) for a reduction in computational cost determined by the choice of $s$.
%%reduction of the computational cost while having a clear understanding of the influence of the choice of $s$ on the statistical efficiency measured with the corresponding expected risk rate. 
%In particular, we could choose a specific number of features $s$ (e.g., $O(n\log^2 n)$, $O(n^{(4t+1)/(4t)}\log n)$ and $O(n^{4/3}\log n)$ in the three previously considered cases) in combination with $\lambda(n) \in O(n^{-1/3})$ and guarantee that the expected risk will converge at the rate $O(n^{-1/3})$, slightly slower than the optimal minimax rate. %%Thus, we can trade the convergence rate for the computational cost since $s$ is much smaller. 
%This is particularly useful in dealing with large datasets, where we may have a strict computational budget but could tolerate a slower convergence rate of the expected risk.
%%
%%While Corollary \ref{main:krr_risk_cor1} demonstrates the risk property with leverage weighted RFF, Corollary \ref{main:krr_risk_cor2} provides a similar results when plain RFF is used.

\begin{cor}\label{main:krr_risk_cor2}
If the probability density function from Theorem~\ref{main:krr_risk} is the spectral measure $p(v)$ from Eq.~(\ref{main:krl_dec}), then the upper bound on the learning risk from Eq.~(\ref{krr_risk_up}) holds for all $s \geq 5\nicefrac{z_0^2}{\lambda}\log\frac{16d_{\mathbf{K}}^{\lambda}}{\delta}$.
\end{cor}
\begin{proof}
For Corollary \ref{main:krr_risk_cor2}, we set $\tilde{l}(v) = p(v)\frac{z_0^2}{\lambda}$ and derive \[d_{\tilde{l}} = \int_{\mathcal{V}}p(v)\frac{z_0^2}{\lambda}dv = \frac{z_0^2}{\lambda}.\]
\end{proof}
Corollary \ref{main:krr_risk_cor2} addresses plain random Fourier features and states that if $s$ is chosen to be greater than $\sqrt{n}\log d_{\mathbf{K}}^{\lambda}$ and $\lambda \propto \nicefrac{1}{\sqrt{n}}$ then the minimax risk convergence rate is guaranteed. In the case of finitely many positive eigenvalues, $s \geq \sqrt{n}$ features are needed to obtain $\mathcal{O}(\nicefrac{1}{\sqrt{n}})$ convergence rate. When the eigenvalues have an exponential decay, we obtain the same convergence rate with only $s \geq \sqrt{n}\log \log n$ features, which is an improvement compared to a result by~\citet{rudi2017generalization} where $s \geq \sqrt{n}\log n$ is needed. For the other two cases, we derive $s \geq \sqrt{n}\log n$ and recover the results from \citet{rudi2017generalization}. Table \ref{tab:squ-wor} provides a summary of the trade-offs between computational complexity and accuracy for the worst case scenario.

\subsubsection{Refined Analysis}\label{sec:sq_ref_ana}
%\vspace{-.5em}

In this section, we provide a more refined analysis with learning risk convergence rates faster than $\mathcal{O}(\nicefrac{1}{\sqrt{n}})$, depending on the spectrum decay of the kernel function and/or the complexity of the target regression function. 

\begin{theorem}\label{main:sharp_risk_theo}
%Let $k$, $y$, $\lambda_1,\cdot,\lambda_n$, $\tilde{l}$, $d_{\tilde{l}}$ and $q$ be the same as in Theorem \ref{main:krr_risk}. 
Under Assumption A.1, suppose that the conditions on sampling measure $\tilde{l}$ from Theorem \ref{main:krr_risk} apply and
let \[s\ \geq\  5d_{\tilde{l}}\log\frac{16d_{\mathbf{K}}^{\lambda}}{\delta}.\] Then, for all $D>1$ and $ \delta \in (0,1)$, with probability $1-\delta$, the excess risk of $\tilde{f}_{\beta}^{\lambda}$ can be upper bounded as
\begin{IEEEeqnarray}{rCl}
\mathbb{E}(l_{\tilde{f}_{\beta}^{\lambda}})-\mathbb{E}(l_{f_{\mathcal{H}}}) \ \leq \   \frac{12D}{B}\hat{r}^*_{\mathcal{H}} + 4\frac{D}{D-1}\lambda +  \mathcal{O}\left(\frac{1}{n}\right) + \mathbb{E}(l_{\hat{f}^{\lambda}})-\mathbb{E}(l_{f_\mathcal{H}}). \label{main:sharp_risk}  
\end{IEEEeqnarray}
Furthermore, denoting the eigenvalues of the normalized kernel matrix $(1/n) \mathbf{K}$ with $\{\hat{\lambda}_i\}_{i=1}^n$, we have that %$\hat{r}^*_{\mathcal{H}}\leq \min_{0\leq h\leq n}\Big(\frac{h}{n}*\frac{e_4}{n^2\lambda^2} + \sqrt{\frac{1}{n}\sum_{i>h}\hat{\lambda}_i}\Big)$, $\{\hat{\lambda}_i\}_{i=1}^n$ are the eigenvalues of the normalized kernel matrix $(1/n) \mathbf{K}$, and $e_4$ and $D>1$ are constants
\begin{IEEEeqnarray}{rCl}
\hat{r}^*_{\mathcal{H}}\leq \min_{0\leq h\leq n}\Big(e_0 \frac{h}{n}+ \sqrt{\frac{1}{n}\sum_{i>h}\hat{\lambda}_i}\Big) ,
\end{IEEEeqnarray}
where $B,e_0>0$ are some constant and $\hat{\lambda}_1\geq \dots \geq \hat{\lambda}_n$.
\end{theorem}

Theorem \ref{main:sharp_risk_theo} covers a wide range of cases and can provide sharper risk convergence rates. In particular, note that $\hat{r}^*_{\mathcal{H}}$ is at least of order $\mathcal{O}(1/\sqrt{n})$, which happens when we let $h = 0$ and the spectrum decays polynomially as $\mathcal{O}(1/n^t), t > 1$. On the other hand, if the eigenvalues decay exponentially, then setting $h = \ceil{\log n}$ implies that $\hat{r}^*_{\mathcal{H}} \leq \mathcal{O}(\log n /n)$. In the best case, when the kernel function has only finitely many positive eigenvalues, we have that $\hat{r}^*_{\mathcal{H}} \leq \mathcal{O}(1/n)$ by letting $h$ be any fixed value larger than the number of positive eigenvalues. These different upper bounds provide various computation and accuracy trade-off. We now split the discussion into two scenario: the weighted sampling with empirical leverage score and the plain sampling.

Under weighted sampling scheme, if the eigenvalues decay polynomially, i.e., $\lambda_i \propto i^{-t}$ for $t > 1$, then the learning rate is upper bounded by $\mathcal{O}(1/\sqrt{n})$. In this case, we have $d_{\mathbf{K}}^{\lambda} \leq (R_0/ \lambda)^{1/t} \leq n^{1/2t}$. We hence have $s \geq n^{1/2t}\log n$. On the other hand, if the eigenvalues decay exponentially, we have $d_{\mathbf{K}}^{\lambda} \leq \log (R_0/ \lambda)^{1/t} \leq \log n$. Hence, if $s \geq \log n \log \log n$, we achieve $\mathcal{O}(\log n/n)$ learning rate. In the best case, where we have finitely many positive eigenvalues, then with a constant number of features, we achieve $\mathcal{O}(1/n)$ learning rate.

On the other hand, if we choose the plain sampling strategy, then the learning rate and required number of features for the three above cases are: $\mathcal{O}(1/\sqrt{n})$ and $s \geq \sqrt{n} \log n$ (polynomial decay), $\mathcal{O}(\log n/n)$ and $s \geq n$ (exponential decay), and $\mathcal{O}(1/n)$ and $s \geq n$ (finite many positive eigenvalues). Table \ref{tab:squ-ref} summarizes our results for the refined case.

\paragraph{Remark:} In \citet{caponnetto2007optimal}, the convergence rate of the excess risk has been linked to the two constants $(b,c)$ where $b\in (1,\infty)$ represents the eigenvalue decay and the $c \in [1,2]$ measures the complexity of the target function $f_{\mathcal{H}}$. Essentially, $c$ determines how fast the coefficients $\alpha_i$ of $f_{\mathcal{H}}$ decay, where $\alpha_i$ represents the coefficient of the expansion of $f_{\mathcal{H}}$ along the eigenfunctions of the integral operator defined by the kernel $k$ and the data generating distribution $P(x,y)$. While $c = 1$ is equivalent to assuming $f_{\mathcal{H}}$ exists, in literature, it is typical that we have to assume benign cases (i.e., $c >1$) to obtain fast learning rates.
    
Our analysis is different from that in \citet{caponnetto2007optimal} in the sense that we only consider the worst case $c = 1$. Under this assumption, we compute excess learning risk of the random Fourier features estimator under various eigenvalue decays (the values of constant $b$). Our results demonstrate that even if we only consider case $c = 1$, we are still able to obtain the rate $\mathcal{O}(1/\sqrt{n})$ in Theorem 1. This is aligned with the worst case rate in \citet{caponnetto2007optimal}. In the refined analysis, the local Rademacher complexity technique allows us to obtain better convergence rates without further assumptions on the constant such as $c > 1$ (e.g., an improvement from $\mathcal{O}(1/\sqrt{n})$ to $\mathcal{O}(1/n)$). Moreover, our fast rate range matches that in \citet{caponnetto2007optimal}.

%Sampling strategies are essentially the same as in the worst case scenario where we have \emph{leverage weighted} and \emph{plain RFF} schemes. The difference is that in the plain sampling scheme, the required the number of features might be higher in the cases of faster convergence rate as we now need to let $\lambda$ to be order $O(\log n/n)$ or even $O(1/n)$.

\begin{table}[t]
	\centering\fontsize{10}{20}\selectfont
	\begin{tabular}{l|c|c|c}
		
		\hline
		
		\hline
		
		\textsc{sampling scheme}&\textsc{spectrum} & \textsc{number of features} & \textsc{learning rate}\\\cline{1-4}
		
		\multirow{3}{8em}{\textsc{weighted rff}} & finite rank & $s \in \Omega(1)$& $\mathcal{O}(1/n)$ \\\cline{2-4}
		
		&$\lambda_i \propto A^{i}$ & $s \in \Omega (\log n \cdot \log \log n)$ & 	$\mathcal{O}(\log n / n)$\\\cline{2-4}
		
		&$\lambda_i \propto i^{-t} $ ($t > 1$) & $s \in \Omega (n^{1/2t} \cdot \log n)$ & 	$\mathcal{O}(1/\sqrt{n})$\\\cline{2-4}

		\hline
		
		\hline
		\multirow{3}{6em}{\textsc{plain rff}} & finite rank & $s \in \Omega(n)$& $\mathcal{O}(1/n)$\\\cline{2-4}
		
		&$\lambda_i \propto A^{i}$ & $s \in \Omega (n)$ & $\mathcal{O}(\log n /n)$ 	\\\cline{2-4}
		
		&$\lambda_i \propto i^{-t} $ ($t> 1$) & $s \in \Omega (\sqrt{n} \cdot \log n)$& 	$\mathcal{O}(1/\sqrt{n})$\\\cline{2-4}

		\hline
		
		\hline
	\end{tabular}
	\caption{The trade-off in the refined case for the squared error loss.}\label{tab:squ-ref}
\end{table}

\subsection{Learning with a Lipschitz Continuous Loss}\label{main:orig_sampling}
%\vspace{-.5em}
We next consider kernel methods with Lipschitz continuous loss, examples of which include kernel support vector machines and kernel logistic regression. Similar to the squared error loss case, we approximate $y_i$ with $g_{\beta}(x_i) = \mathbf{z}_{q,x_i}(\mathbf{v})^{T}\beta$ and formulate the following learning problem
\begin{IEEEeqnarray}{rCl}
\beta^{\lambda} := \argmin_{\beta \in \mathbb{R}^s}~~\frac{1}{n}\sum_{i=1}^n l(y_i,\mathbf{z}_{q,x_i}(\mathbf{v})^{T}\beta) + \lambda s\|\beta\|_2^2. \nonumber
\end{IEEEeqnarray} 
We let $g_{\beta}^{\lambda}$ to be the prediction function defined based on $\beta^{\lambda}$ and state an additional assumption that is specific to the Lipshcitz continuous loss:
\begin{itemize}
    \item[B.$1$] We assume that $l$ is Lipschitz continuous with constant $L$:\[\forall g, g' \in \mathcal{H}, \forall x \in \mathcal{X},\; |l_g-l_{g'}| \leq L|g(x)-g'(x)|.\]
\end{itemize}

\subsubsection{Worst Case Analysis}
The following theorem describes the trade-off between the selected number of features $s$ and the learning risk of the estimator, providing an insight into the choice of $s$ for Lipschitz continuous loss.
\begin{theorem}\label{main:svm_risk}
Under Assumption B.1, suppose that the conditions on sampling measure $\tilde{l}$ from Theorem \ref{main:krr_risk} apply to the setting with a Lipschitz continuous loss. If \[s\ \geq\  5d_{\tilde{l}}\log\frac{16d_{\mathbf{K}}^{\lambda}}{\delta},\] then for all $\delta \in (0,1)$, with probability $1-\delta$, the learning risk of $g_{\beta}^{\lambda}$ can be upper bounded as
\begin{IEEEeqnarray}{rCl}
\mathbb{E}(l_{g_{\beta}^{\lambda}})\ \leq \ \mathbb{E}(l_{g_{\mathcal{H}}})+ \sqrt{2\lambda}+\mathcal{O}\left(\frac{1}{\sqrt{n}}\right).\label{svm_risk_up}
\end{IEEEeqnarray} 
\end{theorem}

This theorem, similar to Theorem \ref{main:krr_risk}, describes the relationship between $s$ and $\mathbb{E}(l_{g_{\beta}^{\lambda}})$ in the Lipschitz continuous loss case. However, a key difference here is that the learning risk can only be upper bounded by $\sqrt{\lambda}$, requiring $\lambda \propto \nicefrac{1}{n}$ in order to preserve the convergence properties of the risk.

Corollaries \ref{main:svm_risk_cor1}  and \ref{main:svm_risk_cor2} provide bounds for the cases of leverage weighted and plain RFF, respectively. The proofs are similar to the proofs of Corollaries \ref{main:krr_risk_cor1} and \ref{main:krr_risk_cor2}.

\begin{cor}\label{main:svm_risk_cor1}
If the probability density function from Theorem \ref{main:svm_risk} is the empirical ridge leverage score distribution $q^*(v)$, then the upper bound on the risk from Eq.~(\ref{svm_risk_up}) holds for all $s \geq  5d_{\mathbf{K}}^{\lambda}\log\frac{16d_{\mathbf{K}}^{\lambda}}{\delta}.$
\end{cor}

Similar to Theorem \ref{main:krr_risk}, we consider four different cases for the effective dimension of the problem $d_{\mathbf{K}}^{\lambda}$. Corollary \ref{main:svm_risk_cor1} states that the statistical efficiency is preserved if the leverage weighted RFF strategy is used with $s= \Omega(1)$, $s \geq \log n\log\log n$, $s \geq n^{1/(2t)}\log n$, and $s \geq n\log n$, respectively. Again, significant computational savings can be achieved if the kernel matrix $\mathbf{K}$ has a finite rank, as well as geometrically/exponentially or polynomially decaying eigenvalues.

\begin{cor}\label{main:svm_risk_cor2}
If the probability density function from Theorem \ref{main:svm_risk} is the spectral measure $p(v)$ from Eq.~(\ref{main:krl_dec}), then the upper bound on the risk from Eq.~(\ref{svm_risk_up}) holds for all  $s \geq 5\nicefrac{z_0^2}{\lambda}\log\frac{(16d_{\mathbf{K}}^{\lambda})}{\delta}.$
\end{cor}

Corollary \ref{main:svm_risk_cor2} states that $n\log n$ features are required to attain $\mathcal{O}(n^{-1/2})$ convergence rate of the learning risk with plain RFF, recovering results from \citet{rahimi2009weighted}. Similar to the analysis in the squared error loss case, Theorem \ref{main:svm_risk} together with Corollaries \ref{main:svm_risk_cor1} and \ref{main:svm_risk_cor2} allows theoretically motivated trade-offs between the statistical and computational efficiency of the estimator $g_{\beta}^{\lambda}$. Table \ref{tab:lip-wor} summarizes the computation and statistical accuracy trade-off for the worst case scenario.

\begin{table}[t]
	\centering\fontsize{10}{20}\selectfont
	\begin{tabular}{l|c|c|c}
		
		\hline
		
		\hline
		
		\textsc{sampling scheme}&\textsc{spectrum} & \textsc{number of features} & \textsc{learning rate}\\\cline{1-4}
		
		\multirow{4}{8em}{\textsc{weighted rff}} & finite rank & $s \in \Omega(1)$& \multirow{4}{4em}{$\mathcal{O}(1/\sqrt{n})$}\\\cline{2-3}
		
		&$\lambda_i \propto A^{i}$ & $s \in \Omega (\log n \cdot \log \log n)$ & 	\\\cline{2-3}
		
		&$\lambda_i \propto i^{-2t} $ ($t\geq 1$) & $s \in \Omega (n^{1/2t} \cdot \log n)$ & 	\\\cline{2-3}
		
		&$\lambda_i \propto i^{-1}$ & $s \in \Omega (n \cdot \log n)$ & 	\\\cline{2-3}
		
		\hline
		
		\hline
		\multirow{4}{6em}{\textsc{plain rff}} & finite rank & \multirow{4}{8em}{$s \in \Omega (n \cdot \log n)$}& \multirow{4}{4em}{$\mathcal{O}(1/\sqrt{n})$}\\\cline{2-2}
		
		&$\lambda_i \propto A^{i}$ & & 	\\\cline{2-2}
		
		&$\lambda_i \propto i^{-2t} $ ($t\geq 1$) & & 	\\\cline{2-2}
		
		&$\lambda_i \propto i^{-1}$ & & 	\\\cline{2-2}
		
		\hline
		
		\hline
	\end{tabular}
	\caption{The trade-off in the worst case for Lipschitz continuous loss.}\label{tab:lip-wor}
\end{table}

\subsubsection{Refined Analysis}
In general, it is hard for classification problems to achieve faster learning rates. However, as pointed out by \cite{bartlett2006convexity} and \cite{steinwart2008support}, in some benign conditions, it is possible to obtain $\mathcal{O}(1/n)$ convergence rate. Hence, in this section, by adding an extra assumption, we derive a faster learning rate for classification problems under random Fourier feautes setting. Specifically, we make the following assumption:
\begin{itemize}
    \item[B.$2$] Recall $g_{\mathcal{H}}$ is the estimator such that $g_{\mathcal{H}} = \arginf_{g\in \mathcal{H}} \mathbb{E}(l_g)$, where $P$ is a probability distribution over $\mathcal{X}\times \mathcal{Y}$. We assume that there is a constant B such that \[\mathbb{E}(g - g_{\mathcal{H}})^2 \leq B\mathbb{E}(l_g - l_{g_{\mathcal{H}}}).\] 
\end{itemize}
Assumption B.$2$ is a condition for classification problems to obtain faster learning rates. It typically requires that the function space $\mathcal{H}$ is convex and uniformly bounded as well as a uniform convexity condition on the loss function $l$. It can be shown that many loss functions satisfy this assumption, including squared loss \citep{bartlett2005local} and hinge loss \citep[Chapter 8.5]{steinwart2008support}. Other loss function examples are discussed in \citet{bartlett2006convexity} and \citet{mendelson2002improving}. In addition, because of the Lipschitz continuity of $l$, we have \[\mathbb{E}(l_g-l_{g_{\mathcal{H}}})^2 \leq L^2\mathbb{E}(g - g_{\mathcal{H}})^2 \leq BL^2\mathbb{E}(l_g - l_{g_{\mathcal{H}}}).\] This is the variance condition described in \citet[Chapter 7.3]{steinwart2008support}, required to achieve faster convergence rates. The variance condition is also linked to the Massart's low noise condition or more generally to the Tsybakov condition \citep{sun2018but}, which intuitively speaking, requires that $P(Y=1|X=x)$ is not close to $1/2$. For more details, please refer to \cite{tsybakov2004optimal} and \cite{koltchinskii2011oracle}. 

\begin{theorem}\label{main:svm_refine}
Under Assumptions B.1-2, suppose the conditions on sampling measure $\tilde{l}$ from Theorem \ref{main:krr_risk} apply to the setting with a Lipschitz continuous loss. If \[s\ \geq\  5d_{\tilde{l}}\log\frac{16d_{\mathbf{K}}^{\lambda}}{\delta},\]then we have for all $D >1$ and $\delta \in (0,1)$ with probability greater than $1-\delta$, 
\begin{IEEEeqnarray}{rCl}
\mathbb{E}(l_{g_{\beta}^{\lambda}}) \leq \frac{12D}{B}\hat{r}_{\mathcal{H}}^* + \frac{D}{D-1}\sqrt{2\lambda} + \mathcal{O}(1/n) + \mathbb{E}(l_{g_{\mathcal{H}}}).
\end{IEEEeqnarray} 
Here, $\hat{r}^*_{\mathcal{H}}$ can be upper bounded by
\begin{IEEEeqnarray}{rCl}
\hat{r}^*_{\mathcal{H}}\leq \text{min}_{0\leq h\leq n}\Big(b_0\frac{h}{n} + \sqrt{\frac{1}{n}\sum_{i>h}\hat{\lambda}_i}\Big),
\end{IEEEeqnarray}
where $B$ and $b_0$ are some constants.
\end{theorem}
Theorem \ref{main:svm_refine} provides a sharper learning rate compared to Theorem \ref{main:svm_risk}. Similar to Theorem \ref{main:sharp_risk_theo}, $\hat{r}^*_{\mathcal{H}}$ can be upper bounded by $\mathcal{O}(1/n)$ (Gram-matrix is of finite rank), $\mathcal{O}(\log n/n)$ (eigenvalues decay exponentially) and $\mathcal{O}(1/\sqrt{n})$ (eigenvalues decay proportional to $1/n$). This has various implications on the trade-offs between computational cost and statistical accuracy. Just as in previous sections, we split the discussion into two parts according to the two sampling strategies.

We first discuss the scenario with empirical leverage score sampling. In a finite rank setting, if we choose $\lambda \in \mathcal{O}(1/n^2)$, we can see that the learning rate is of the order of $\mathcal{O}(1/n)$. In addition, since we use the weighted sampling strategy and the Gram-matrix has finitely many eigenvalues, random Fourier features learning only requires a constant number of features to achieve $\mathcal{O}(1/n)$ learning rate. To the best of our knowledge, this is the first result that achieves this.
In an exponential decay setting, the learning rate can be bounded with $\hat{r}^*_{\mathcal{H}} \leq \nicefrac{\log n}{n}$ by setting $\lambda \in \mathcal{O}(\nicefrac{\log^2 n}{n^2})$. The number of required features is $s \geq \log n\log \log n$ as $d_{\mathbf{K}}^{\lambda} \leq \log (R^2/\lambda) \leq \log n$.
If the eigenvalue decay at the rate $\lambda_i \propto \mathcal{O}(i^{-t}), t > 1$, then the learning rate is $\mathcal{O}(1/\sqrt{n})$ by setting $\lambda \in \mathcal{O}(1/n)$, with the requirement on the number of features given by $d_{\mathbf{K}}^{\lambda} \leq (R^2/ \lambda)^{1/t} \leq n^{1/t}$. Since $t > 1$, one can see that with fewer than $n$ features, we could obtain fast $\mathcal{O}(1/n)$ learning rate.

On the other hand, in the plain sampling scheme, if we would like to achieve the fast $\mathcal{O}(1/n)$ learning rate, we need to set $\lambda \in \mathcal{O}(1/n^2)$, implying that the required number of features has to be $s \geq n^2$. This is undesirable as it does not provide any computation savings. The bottleneck here is that in the Lipschitz continuous case, learning rate is upper bounded by $\mathcal{O}(\sqrt{\lambda})$. This limits the learning rate that can be achieved. %How to further improve this condition would be an interesting future direction.  

\begin{table}[t]
	\centering\fontsize{10}{20}\selectfont
	\begin{tabular}{l|c|c|c}
		
		\hline
		
		\hline
		
		\textsc{sampling scheme}&\textsc{spectrum} & \textsc{number of features} & \textsc{learning rate}\\\cline{1-4}
		
		\multirow{3}{8em}{\textsc{weighted rff}} & finite rank & $s \in \Omega(1)$& $\mathcal{O}(1/n)$ \\\cline{2-4}
		
		&$\lambda_i \propto A^{i}$ & $s \in \Omega (\log n \cdot \log \log n)$ & 	$\mathcal{O}(\log n / n)$\\\cline{2-4}
		
		&$\lambda_i \propto i^{-t} $ ($t > 1$) & $s \in \Omega (n^{1/t} \cdot \log n)$ & 	$\mathcal{O}(1/\sqrt{n})$\\\cline{2-4}

		\hline
		
		\hline
		\multirow{3}{6em}{\textsc{plain rff}} & finite rank & $s \in \Omega(n^2)$& $\mathcal{O}(1/n)$\\\cline{2-4}
		
		&$\lambda_i \propto A^{i}$ & $s \in \Omega (n^2)$ & $\mathcal{O}(\log n /n)$ 	\\\cline{2-4}
		
		&$\lambda_i \propto i^{-t} $ ($t> 1$) & $s \in \Omega (n \cdot \log n)$& 	$\mathcal{O}(1/\sqrt{n})$\\\cline{2-4}

		\hline
		
		\hline
	\end{tabular}
	\caption{The trade-off in the refined case for Lipschitz continuous loss.}\label{tab:lip-ref}
\end{table}

% In a realizable case, i.e., the Bayes classifier belongs to the RKHS spanned by the features, Theorem \ref{main:svm_refine} describes a refined relationship between the learning risk and the number of features. In addition, it also implicitly states how the complexity of $g_l^*$ can affects the learning risk convergence rate. Basically, if choosing $s = O(\sqrt{n})$ is sufficient to make the RKHS $\tilde{\mathcal{H}}$ large enough to include $g_l^*$, and we let $\lambda = O(1/\sqrt{n})$, we can then achieve the learning rate of $O(\frac{\log n}{\sqrt{n}})$ with only $O(\sqrt{n})$ features. To our knowledge, this is the first result that shows we can obtain computational savings in Lipschitz continuous loss case. Furthermore, if $g_l^*$ has low complexity in the sense that with only finitely many features $c_s, c_s < \infty$, $\tilde{\mathcal{H}}$ can include $g_l^*$, then we can achieve $O(1/n)$ convergence rate with only $c_s$ features. That being said, however, we are in the realizable case which is somewhat limiting. Also, we lack a specific way to describe the exact complexity of $g_l^*$ in terms of the number of features. Hence, an interesting area for future work is to extend our analysis to the unrealizable case and to analyze the complexity of $g_l^*$.\\

\renewcommand{\algorithmicrequire}{\textbf{Input:}}
\renewcommand{\algorithmicensure}{\textbf{Output:}}

\begin{algorithm}[t]
%\algsetup{linenosize=\tiny}
\caption{\textsc{Approximate Leverage Weighted RFF}} 
\label{alo:opm_sampling}
\begin{algorithmic}[1]
{\fontsize{11}{12}\selectfont
\REQUIRE %{\ }
sample of examples $\{(x_i, y_i)\}_{i=1}^n$, shift-invariant kernel function $k$, and regularization parameter $\lambda$
\ENSURE set of features $\{(v_1,p_1),\cdots,(v_m,p_m)\}$ with $m$ and each $p_i$ computed as in lines 3--4
\STATE sample $s$ features $\{v_1, \dots, v_s\}$ from $p(v)$
\STATE create a feature matrix $\mathbf{Z}_s$ such that the $i$th row of $\mathbf{Z}_s$ is\vspace{-5pt}
\begin{align*}
[z(v_1,x_i),\cdots,z(v_s,x_i)]^T
\end{align*}\vspace{-15pt}
\STATE associate with each feature $v_i$ a real number $p_i$ such that $p_i$ is equal to the $i$th diagonal element of the matrix\vspace{-5pt}
\begin{align*}
\small
\mathbf{Z}_{s}^T\mathbf{Z}_{s}((1/s)\mathbf{Z}_{s}^T\mathbf{Z}_{s}+n\lambda I)^{-1}
\end{align*}\vspace{-15pt}
\STATE $m \leftarrow \sum_{i=1}^s p_i$ and $M \leftarrow \{(v_i,p_i/m)\}_{i=1}^s$
\STATE sample $m$ features from $M$ using the multinomial distribution given by the vector $(p_1/m,\cdots,p_s/m)$
}
\end{algorithmic}
\end{algorithm}\vspace*{-1ex}

\section{A Fast Approximation of Leverage Weighted RFF}
%\vspace{-.5em}
As discussed in Sections~\ref{sec:theo_analysis}, sampling according to the empirical ridge leverage score distribution (i.e., leverage weighted RFF) could speed up kernel methods. However, computing ridge leverage scores is as costly as inverting the Gram matrix. To address this computational shortcoming, we propose a simple algorithm to approximate the empirical ridge leverage score distribution and the leverage weights.
%\vspace{-1em}
%\begin{alo}\label{alo:opm_sampling}
%
%\end{alo}
%\vspace{-.5em}
%\begin{itemize}[noitemsep]
%\item \textbf{Input}: i.i.d data $\{x_i,y_i\}_{i=1}^n$, kernel $k$, hyper-parameter $\lambda$.
%\item \textbf{Output}: feature set $\{(v_1,p_1),\cdots,(v_l,p_l)\}$, where $l$ and each $p_i$ are computed in step 4.
%\end{itemize}
%\vspace{-.5em}
%\begin{enumerate}[noitemsep]
%\item Draw $s$ features $\{v_1,\cdots,v_s\}$ from $p(v)$.
%\item Form the feature matrix $\mathbf{Z}_s$, where the $i$th row of $\mathbf{Z}_s$ is $[z(v_1,x_i),\cdots,z(v_s,x_i)]^T$.
%\item Associate each feature $v_i$ a value $p_i$, where $p_i$ is the $i$th diagonal element of the matrix $\mathbf{Z}_{s}^T\mathbf{Z}_{s}((1/s)\mathbf{Z}_{s}^T\mathbf{Z}_{s}+n\lambda I)^{-1}$.
%\item Let $l = \sum_{i=1}^s p_i$ and $M =\{(v_i,p_i/l)\}_{i=1}^s$. We draw $l$ features from $M$ according to the multinomial distribution $(p_1/l,\cdots,p_s/l)$.
%\end{enumerate}
In particular, we propose to first sample a pool of $s$ features from the spectral measure $p(\cdot)$ and form the feature matrix $\mathbf{Z}_{s} \in \mathbb{R}^{n\times s}$ (Algorithm~\ref{alo:opm_sampling}, lines $1$-$2$). Then, the algorithm associates an approximate empirical ridge leverage score to each feature (Algorithm~\ref{alo:opm_sampling}, lines $3$-$4$) and samples a set of $m \ll s$ features from the pool proportional to the computed scores (Algorithm~\ref{alo:opm_sampling}, line $5$). The output of the algorithm can be compactly represented via the feature matrix $\mathbf{Z}_{m} \in \mathbb{R}^{n\times m}$ such that the $i$th row of $\mathbf{Z}_{m}$ is given by $\mathbf{z}_{x_i}(\mathbf{v}) = [\sqrt{\nicefrac{m}{p_1}}z(v_1,x_i),\cdots,\sqrt{\nicefrac{m}{p_m}}z(v_m,x_i)]^T$.
%\vspace{-.5em}

The computational cost of Algorithm~\ref{alo:opm_sampling} is dominated by the operations in step $3$. As $\mathbf{Z}_s \in \mathbb{R}^{n\times s}$, the multiplication of matrices $\mathbf{Z}_s^T\mathbf{Z}_s$ costs $\mathcal{O}(ns^2)$ and inverting $\mathbf{Z}_s^T\mathbf{Z}_s + n\lambda I$ costs only $\mathcal{O}(s^3)$. Hence, for $s \ll n$, the overall runtime is only $\mathcal{O}(ns^2+s^3)$. Moreover, $\mathbf{Z}_s^T\mathbf{Z}_s = \sum_{i=1}^n \mathbf{z}_{x_i}(\mathbf{v})\mathbf{z}_{x_i}(\mathbf{v})^T$ and it is possible to store only the rank-one matrix $\mathbf{z}_{x_i}(\mathbf{v})\mathbf{z}_{x_i}(\mathbf{v})^T$ into the memory. Thus, the algorithm only requires to store an $s \times s$ matrix and can avoid storing $\mathbf{Z}_s$, which would incur a storage cost of $\mathcal{O}(n\times s)$.

The following theorem gives the convergence rate for the learning risk of Algorithm \ref{alo:opm_sampling} in the kernel ridge regression setting.

\begin{theorem}\label{main:algo1_conv}
Under Assumption A.1, consider regression problem defined with a shift-invariant kernel $k$, a sample of examples $\{(x_i, y_i)\}_{i=1}^n$, and a regularization parameter $\lambda$. Let $s$ be the number of random Fourier features in the pool of features from Algorithm~\ref{alo:opm_sampling}, sampled using the spectral measure $p(\cdot)$ from Eq.~(\ref{main:krl_dec}) and the regularization parameter $\lambda$. Denote with $\tilde{f}_m^{\lambda^*}$ the ridge regression estimator obtained using a regularization parameter $\lambda^*$ and a set of random Fourier features $\{v_i\}_{i=1}^m$ returned by Algorithm~\ref{alo:opm_sampling}. If 
\[s\ \geq\ \frac{7z_0^2}{\lambda}\log\frac{(16d_{\mathbf{K}}^{\lambda})}{\delta} \quad \text{and} \quad m\ \geq\  5d_{\mathbf{K}}^{\lambda^*}\log\frac{(16d_{\mathbf{K}}^{\lambda^*})}{\delta} ,\]
then for all $\delta \in (0,1)$, with probability $1-\delta$, the learning risk of $\tilde{f}_m^{\lambda^*}$ can be upper bounded as
\begin{IEEEeqnarray}{rCl}
\mathbb{E}(l_{\tilde{f}_m^{\lambda^*}})\ \leq \ \mathbb{E}(l_{f_{\mathcal{H}}}) + 4\lambda + 4\lambda^* +\mathcal{O}\left(\frac{1}{\sqrt{n}}\right) . \nonumber
\end{IEEEeqnarray} 
Moreover, this upper bound holds for $m \in \Omega(\nicefrac{s}{n\lambda})$. 
\end{theorem}

%\begin{rem}
Theorem \ref{main:algo1_conv} bounds the learning risk of the ridge regression estimator over random features generated by Algorithm~\ref{alo:opm_sampling}. We can now observe that using the minimax choice of the regularization parameter for kernel ridge regression $\lambda, \lambda^* \propto n^{-1/2}$, the number of features that Algorithm~\ref{alo:opm_sampling} needs to sample from the spectral measure of the kernel $k$ is $s \in \Omega(\sqrt{n}\log n)$. Then, the ridge regression estimator $\tilde{f}_m^{\lambda^*}$ converges with the minimax rate to the hypothesis $f_{\mathcal{H}} \in \mathcal{H}$ for $m \in \Omega (\log n \cdot \log \log n) $. 
%
%we can see that the learning risk converges to $f_{\mathcal{H}}$ at the minimax rate and the number of features is $l = O(1)$. 
This is a significant improvement compared to the spectral measure sampling (plain RFF), which requires $\Omega(n^{3/2})$ features for in-sample training and $\Omega(\sqrt{n}\log n)$ for out-of-sample test predictions. 
%Whereas here, we only need $O(n)$ and $O(1)$ time respectively. 

Theorem~\ref{main:algo1_conv_lipschitz} provides a convergence bound for kernel support vector machines and logistic regression. Compared to the previous result, the convergence rate of the learning risk, however, is at a slower $\mathcal{O}(\sqrt{\lambda} + \sqrt{\lambda^*})$ rate due to the difference in the employed loss function (similar to Section~\ref{main:orig_sampling}).
\begin{theorem}\label{main:algo1_conv_lipschitz}
Under Assumption B.1, consider learning problem with Lipschitz continuous loss, a shift-invariant kernel $k$, a sample of examples $\{(x_i, y_i)\}_{i=1}^n$, and a regularization parameter $\lambda$. Let $s$ be the number of random Fourier features in the pool of features from Algorithm~\ref{alo:opm_sampling}, sampled using the spectral measure $p(\cdot)$ from Eq.~(\ref{main:krl_dec}) and the regularization parameter $\lambda$. Denote with $\tilde{g}_m^{\lambda^*}$ the estimator obtained using a regularization parameter $\lambda^*$ and a set of random Fourier features $\{v_i\}_{i=1}^m$ returned by Algorithm~\ref{alo:opm_sampling}. If 
\[s\ \geq\ \frac{5z_0^2}{\lambda}\log\frac{(16d_{\mathbf{K}}^{\lambda})}{\delta} \quad \text{and} \quad m\ \geq \  5d_{\mathbf{K}}^{\lambda^*}\log\frac{(16d_{\mathbf{K}}^{\lambda^*})}{\delta} ,\]
then for all $\delta \in (0,1)$, with probability $1-\delta$, the learning risk of $\tilde{g}_m^{\lambda^*}$ can be upper bounded as
\begin{IEEEeqnarray}{rCl}
\mathbb{E}(l_{\tilde{g}_m^{\lambda^*}})\ \leq \ \mathbb{E}(l_{g_{\mathcal{H}}}) + \sqrt{2\lambda} + \sqrt{2\lambda^*} +\mathcal{O}\left(\frac{1}{\sqrt{n}}\right) . \nonumber
\end{IEEEeqnarray} 
%Moreover, this upper bound holds for $l \in \Omega(\frac{s}{n\lambda})$. 
\end{theorem}

We conclude by pointing out that the proposed algorithm provides an interesting new trade-off between the computational cost and prediction accuracy. In particular, one can pay an upfront cost (same as plain RFF) to compute the leverage scores, re-sample significantly fewer features and employ them in the training, cross-validation, and prediction stages. This can reduce the computational cost for predictions at test points from $\Omega(\sqrt{n}\log n)$ to $\Omega(\log n \cdot \log \log n)$. Moreover, in the case where the amount of features with approximated leverage scores utilized is the same as in plain RFF, the prediction accuracy would be significantly improved as demonstrated in our experiments.% section below.  
%\end{rem}

\begin{figure*}[t]
%\vspace{-.5em}
	\centering
	\begin{subfigure}{0.49\textwidth}
		\centering
		\includegraphics[width=7.cm,height=4.5cm]{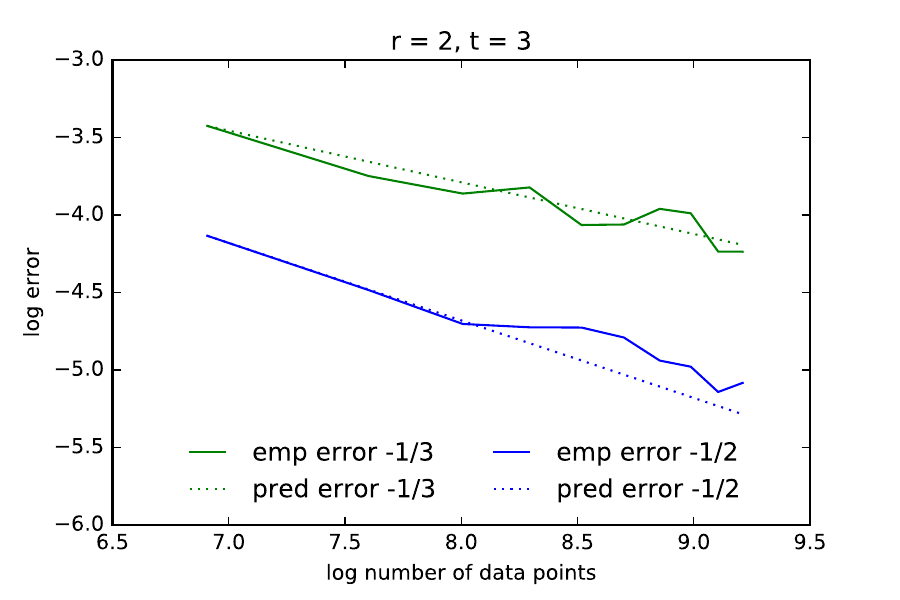}
		%\caption{Original Sampling Strategy}
		%\label{krr:fig 1}
	\end{subfigure}
	%\hfill
	\centering
	\begin{subfigure}{0.49\textwidth}
		\centering
		\includegraphics[width=7.cm,height=4.5cm]{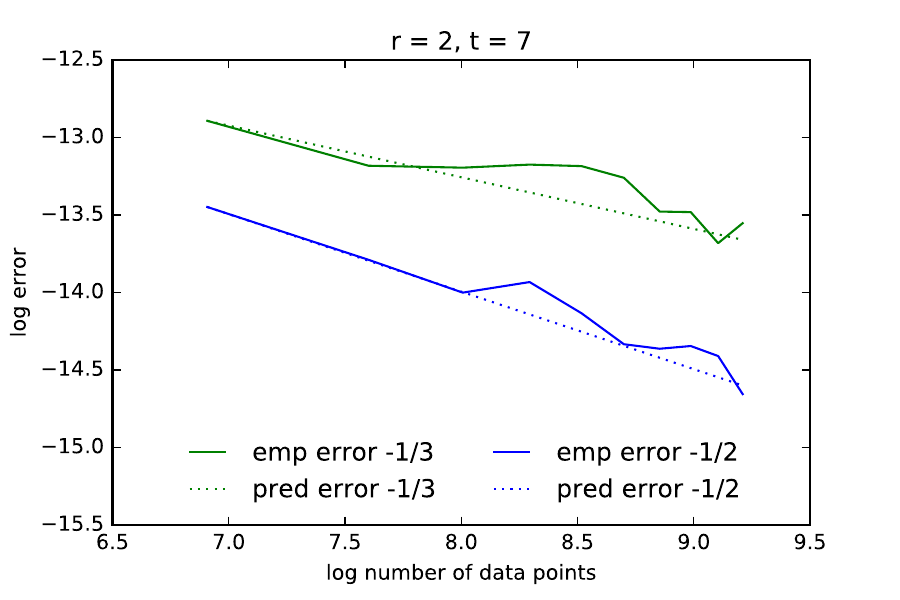}
		%\caption{Optimal Sampling Strategy}
		%\label{krr:fig 2}
	\end{subfigure}
	%\hfill
	\centering
	\begin{subfigure}{0.49\textwidth}
		\centering
		\includegraphics[width=7.cm,height=4.5cm]{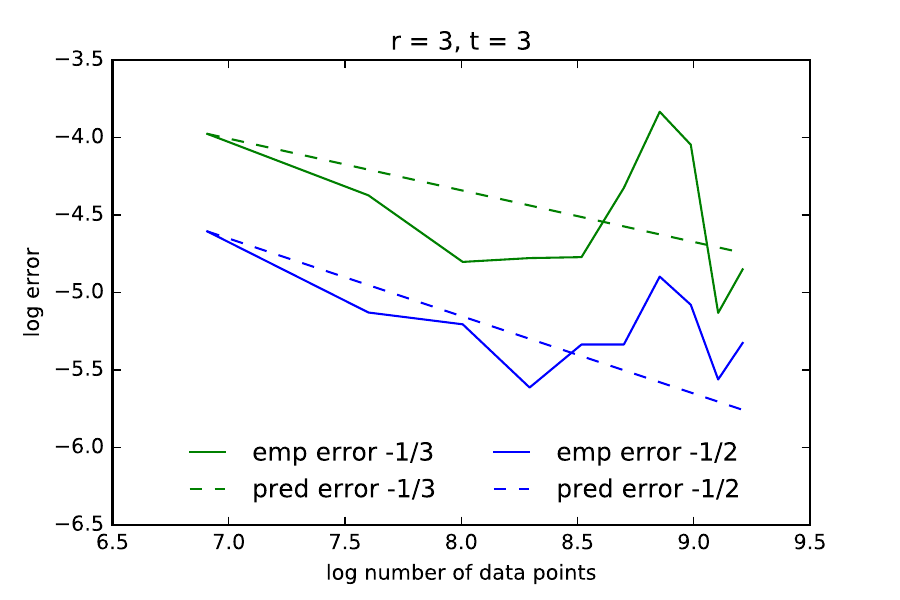}
		%\caption{Original Sampling Strategy}
		%\label{krr:fig 1}
	\end{subfigure}
	%\hfill
	\centering
	\begin{subfigure}{0.49\textwidth}
		\centering
		\includegraphics[width=7.cm,height=4.5cm]{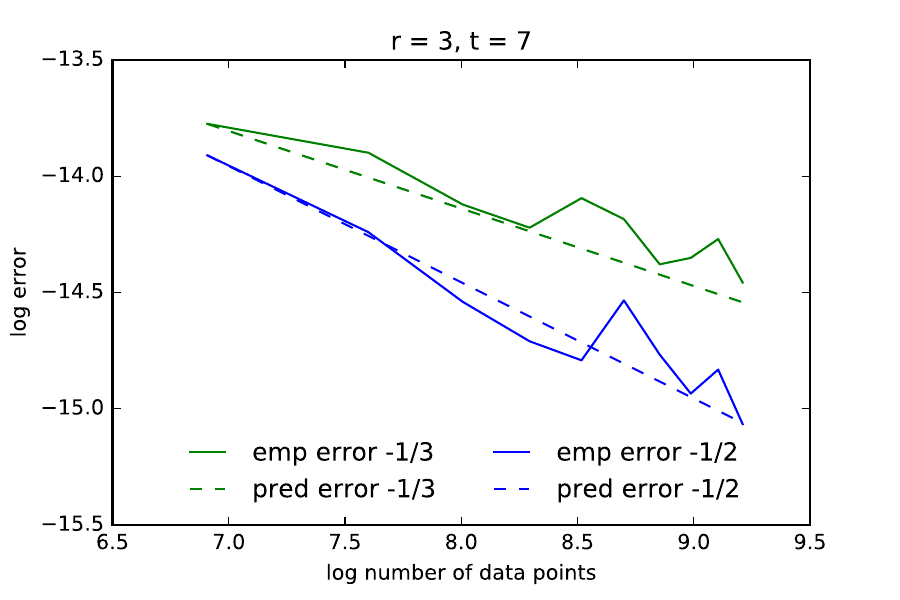}
		%\caption{Original Sampling Strategy}
		%\label{krr:fig 1}
	\end{subfigure}
	\caption{The log-log plot of the theoretical and simulated risk convergence rates, averaged over 100 repetitions.}
	\label{simulation_study}
\end{figure*}

\section{Numerical Experiments}
In this section, we report the results of our numerical experiments (on both simulated and real-world datasets) aimed at validating our theoretical results and demonstrating the utility of Algorithm \ref{alo:opm_sampling}. We first verify our results through a simulation experiment. Specifically, we consider a spline kernel of order $r$ where $k_{2r}(x,y) = 1 + \sum_{m>0}\frac{1}{m^{2r}}\cos2\pi m(x-y)$~\citep[also considered by][]{bach2017equivalence,rudi2017generalization}. If the marginal distribution of $X$ is uniform on $[0,1]$, we can show that $k_{2r}(x,y) = \int_{0}^{1}z(v,x)z(v,y)q^*(v)dv$, where $z(v,x) = k_{r}(v,x)$ and $q^*(v)$ is also uniform on $[0,1]$. We let $y$ be a Gaussian random variable with mean $f(x) = k_t(x,x_0)$ (for some $x_0 \in [0,1]$) and variance $\sigma^2$. We sample features according to $q^*(v)$ to estimate $f$ and compute the excess risk. By Theorem \ref{main:krr_risk} and Corollary \ref{main:krr_risk_cor1}, if the number of features is proportional to $d_{\mathbf{K}}^{\lambda}$ and $\lambda \propto n^{-1/2}$, we should expect the excess risk converging at $\mathcal{O}(n^{-1/2})$, or at $\mathcal{O}(n^{-1/3})$ if $\lambda \propto n^{-1/3}$. Figure \ref{simulation_study} demonstrates this with different values of $r$ and $t$.

\begin{figure*}[t]
	\centering
	\begin{subfigure}{0.49\textwidth}
		\centering
		\includegraphics[width=7.cm,height=5.5cm]{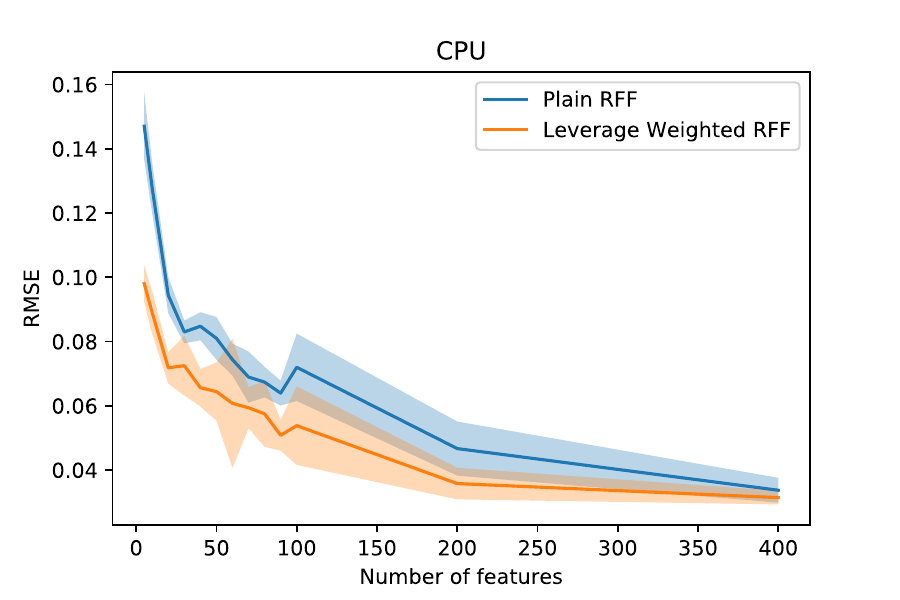}
	\end{subfigure}
	%\hfill
	\centering
	\begin{subfigure}{0.49\textwidth}
		\centering
		\includegraphics[width=7.cm,height=5.5cm]{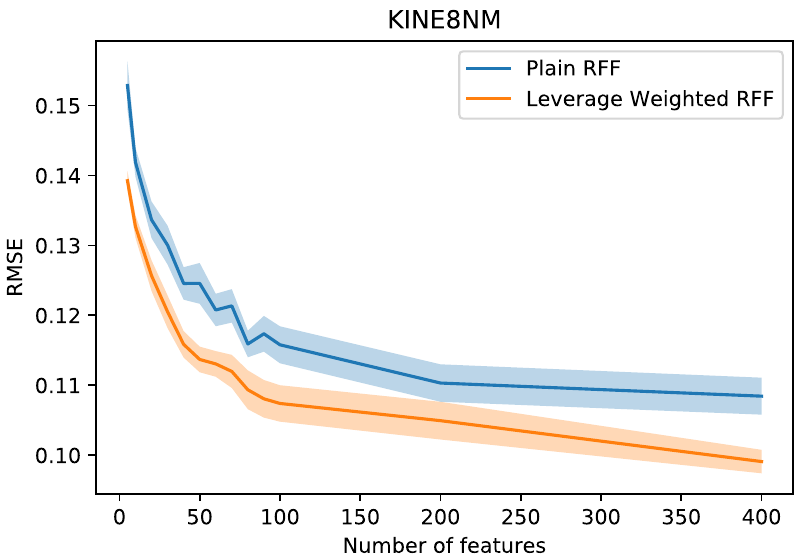}
		%\caption{Optimal Sampling Strategy}
		%\label{krr:fig 2}
	\end{subfigure}
	%\hfill
	\centering
	\begin{subfigure}{0.49\textwidth}
		\centering
		\includegraphics[width=7.cm,height=5.5cm]{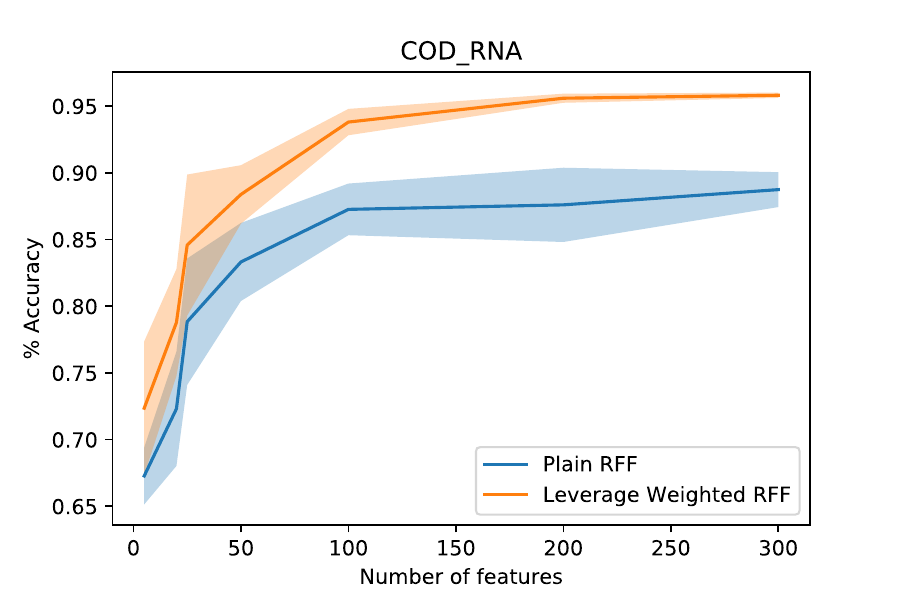}
		%\caption{Original Sampling Strategy}
		%\label{krr:fig 1}
	\end{subfigure}
	%\hfill
	\centering
	\begin{subfigure}{0.49\textwidth}
		\centering
		\includegraphics[width=7.cm,height=5.5cm]{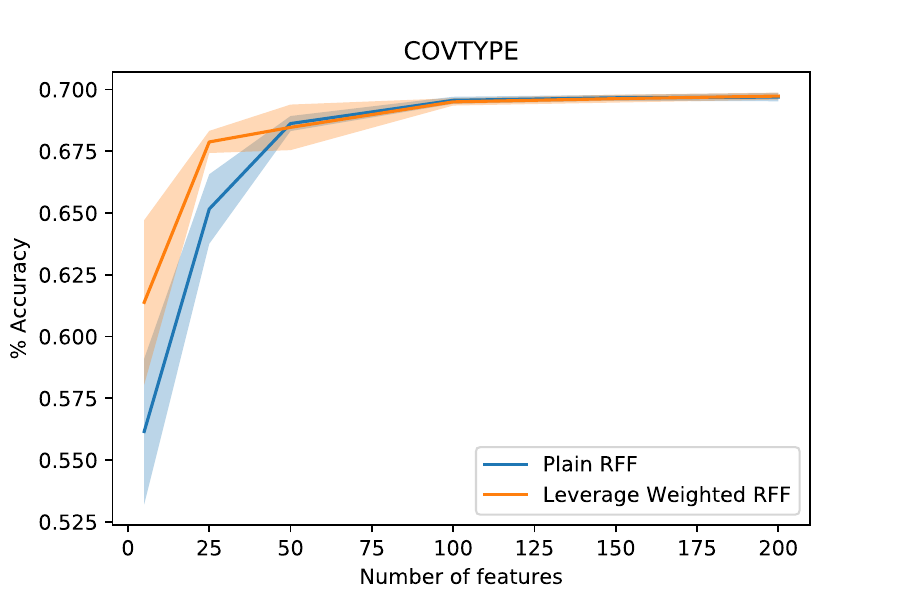}
		%\caption{Original Sampling Strategy}
		%\label{krr:fig 1}
	\end{subfigure}
	\caption{Comparison of leverage weighted and plain RFFs, with weights computed according to Algorithm \ref{alo:opm_sampling}.}
	\label{empirical_study}
%	\vspace*{-2ex}
\end{figure*}

Next, we make a comparison between the performances of leverage weighted (computed according to Algorithm \ref{alo:opm_sampling}) and plain RFF on real-world datasets. We use four datasets from \citet{CC01a} and~\citet{Dua:2017} for this purpose, including two for regression and two for classification: \texttt{CPU}, \texttt{KINEMATICS}, \texttt{COD-RNA} and \texttt{COVTYPE}. Except \texttt{KINEMATICS}, the other three datasets were used in \citet{yang2012nystrom} to investigate the difference between the Nystr{\"o}m method and plain RFF. We use the ridge regression and SVM package from \citet{scikit-learn} as a solver to perform our experiments. We evaluate the regression tasks using the root mean squared error and the classification ones using the average percentage of misclassified examples. The Gaussian/RBF kernel is used for all the datasets with hyper-parameter tuning via $5$-fold inner cross validation. We have repeated all the experiments $10$ times and reported the average test error for each dataset. Figure \ref{empirical_study} compares the performances of leverage weighted and plain RFF. In regression tasks, we observe that the upper bound of the confidence interval for the root mean squared error corresponding to leverage weighted RFF is below the lower bound of the confidence interval for the error corresponding to plain RFF. Similarly, the lower bound of the confidence interval for the classification accuracy of leverage weighted RFF is (most of the time) higher than the upper bound on the confidence interval for plain RFF. This indicates that leverage weighted RFFs perform statistically significantly better than plain RFFs in terms of the learning accuracy and/or prediction error.

In the final experiment, we investigate the effectiveness of our algorithm relative to plain RFFs and to that end construct the following synthetic dataset. We first generate samples $w^*$ from a multimodal Gaussian distribution where the modes are at $(-2, -2), (-2, 2), (2, -2)$, and $(2, 2)$ and where each mode has a diagonal covariance matrix of 0.5. These samples are going to be our frequencies for a RFF mapping. Next, we sample our covariates $x$ from $\mathcal{N}(0, 5*I)$. In order to generate our response variables, we map the covariates $x$ through a RFF map where the frequencies are given by samples $w^*$. We then randomly sample regression weights $\alpha_r$ from $\mathcal{N}(0, 1)$. Hence, the data generating process can be described as follows:
\[y = \alpha_r^T \phi_{w^*}(x) + \epsilon,\]
where $\epsilon \sim \mathcal{N}(0, \sigma)$ and $\phi_{w^*}$ is a RFF map with $w^*$ as the frequencies. By setting up our data generating process in this way, we are able to systematically investigate how well the proposed algorithm works. In particular, we consider learning the above described hypothesis using RFFs that correspond to a Gaussian kernel. Such a kernel corresponds to a uni-modal Gaussian distribution in the frequency domain and we will show that leverage weighted sampling is capable of selecting a sub-set of plain RFF sampled from that distribution, which covers the modes of the multimodal distribution that characterizes the data generating process.

In our experiments we have opted for the following setting: we use $50\ 000$ data points, 400 features for $w^*$, and $\sigma=0.1$.
We then run plain RFF as well as our leverage weighted RFF on this dataset. For plain RFF, we run the experiments with $1\ 000$ and $50\ 000$ frequencies/features, while with our leverage weighted RFF we only use $1\ 000$ features which have been selected from a pool consisting of $10\ 000$ plain random RFFs (i.e., a sub-sample from the original $50\ 000$ features). We carefully cross-validate both methods across a grid of hyper parameters and report the results in Table~\ref{tab:exp1}. The results confirm our theoretical findings and illustrate that learning with $1\ 000$ leverage weighted RFFs is as effective as learning with a complete pool of $50\ 000$ plain RFFs.

\begin{table}[t]
\centering\fontsize{10}{20}\selectfont
\begin{tabular}{c|c|c}
\hline

\hline
\textsc{\# of features}       & \textsc{plain rff}      & \textsc{leverage weighted rff} \\ 
\hline
$1\ 000$ & $0.13 \pm 0.06$ & $0.04 \pm 0.01$        \\
$50\ 000$ & $0.04 \pm 0.02$ & NA                    \\ 
\hline

\hline
\end{tabular}
\caption{The table summarizes the RMSE of our experiment on a synthetic dataset and illustrates the effectiveness of the proposed algorithm (i.e., leverage weighted RFF) relative to plain RFF. The reported numbers are the root mean squared error along with a confidence interval.}
\label{tab:exp1}
\end{table}

\section{Proofs}

\subsection{Proof of Proposition \ref{apn:func_norm}} \label{pf:func_norm}
\begin{proof}
Let us define a space of functions as \[\mathcal{H}_1 \coloneqq \{f \mid f(x) = \alpha z(v,x), \alpha \in \mathbb{R}\}.\] We now show that $\mathcal{H}_1$ is a reproducing kernel Hilbert space with kernel defined as $k_1(x,y)= (1/s)z(v,x)z(v,y)$, where $s$ is a constant. Define a map $M : \mathbb{R} \rightarrow \mathcal{H}_1$ such that $M\alpha = \alpha z(v,\cdot),\forall \alpha \in \mathbb{R}$. The map $M$ is a bijection, i.e. for any $f\in \mathcal{H}_1$ there exists a unique $\alpha_f \in \mathbb{R}$ such that $M^{-1}f = \alpha_f$. Now, we define an inner product on $\mathcal{H}_1$ as $$\langle f,g \rangle_{\mathcal{H}_1} = \langle \sqrt{s}M^{-1}f,\sqrt{s}M^{-1}g \rangle_{\mathbb{R}} = s\alpha_f\alpha_g.$$ It is easy to show that this is a well defined inner product and, thus, $\mathcal{H}_1$ is a Hilbert space.\\

For any instance $y$, $k_1(\cdot,y) = (1/s)z(v,\cdot) z(v,y) \in \mathcal{H}_1$, since $(1/s)z(v,y) \in \mathbb{R}$ by definition. Take any $f \in \mathcal{H}_1$ and observe that 
\begin{IEEEeqnarray}{rCl}
\langle f, k_1(\cdot,y) \rangle_{\mathcal{H}_1} &=& \langle \sqrt{s}M^{-1}f, \sqrt{s}M^{-1}k_1(\cdot,y) \rangle_{\mathbb{R}}\nonumber\\
&=& s \langle\alpha_f,1/sz(v,y) \rangle_{\mathbb{R}}\nonumber\\
&=& \alpha_f z(v,y) = f(y) . \nonumber
\end{IEEEeqnarray}
Hence, we have demonstrated the reproducing property for $\mathcal{H}_1$ and $\|f\|_{\mathcal{H}_1}^2 = s\alpha_f^2$.\\

Now, suppose we have a sample of features $\{v_i\}_{i=1}^s$. For each $v_i$, we define the reproducing kernel Hilbert space
\begin{align*}
    \mathcal{H}_i \coloneqq \{f \mid f(x) =  \alpha z(v_i,x), \alpha \in \mathbb{R}\}  
\end{align*}
with the kernel $k_i(x,y)= (1/s)z(v_i,x) z(v_i,y)$. Denoting with 
\begin{align*}
    \tilde{\mathcal{H}} =\oplus_{i=1}^s\mathcal{H}_i = \{\tilde{f}:\tilde{f}= \sum_{i=1}^sf_i, f_i \in \mathcal{H}_i\}
\end{align*} 
and using the fact that the direct sum of reproducing kernel Hilbert spaces is another reproducing kernel Hilbert space~\citep{berlinet2011reproducing}, we have that $\tilde{k}(x,y) = \sum_{i=1}^s k_{i}(x,y) = (1/s) \sum_{i=1}^s z(v_i ,x) z(v_i,y)$ is the kernel of $\tilde{\mathcal{H}}$ and that the squared norm of $\tilde{f} \in \tilde{\mathcal{H}}$ is defined as 
\begin{align*}
\begin{aligned}
& \min_{f_i \in \mathcal{H}_i\ \mid\ \tilde{f}=\sum_{i=1}^s f_i}\ \sum_{i=1}^s \|f_i\|_{\mathcal{H}_i}^2 = &\\
& \min_{\alpha_i \in \mathbb{R} \ \mid \ \tilde{f} = \sum_{i=1}^s \alpha_i z(v_i, \cdot)}\ \sum_{i=1}^s s\alpha_i^2  = \min_{\alpha_i \in \mathbb{R} \ \mid \ \tilde{f} = \sum_{i=1}^s \alpha_i z(v_i, \cdot)}\ s\|\alpha\|_2^2. &
\end{aligned}
\end{align*} 
Hence, we have that $\|\tilde{f}\|_{\tilde{\mathcal{H}}}^2\leq s\|\alpha\|_2^2$.
\end{proof}

\subsection{Proof of Theorem \ref{main:krr_risk}}
To prove Theorem \ref{main:krr_risk}, we need Lemma \ref{func_appx_opm} and Lemma \ref{triangle_lma} (proof in Appendix 
\ref{krr_risk_prof} and \ref{apn:triangle_lma_pf}) to analyse the learning risk. In Lemma \ref{func_appx_opm}, we give a general result that provides an upper bound on the approximation error between any function $f\in \mathcal{H}$ and its estimator based on random Fourier features. As discussed in Section \ref{sec:background}, we would like to approximate a function $f \in \mathcal{H}$ at observation points using $\tilde{f} \in \tilde{\mathcal{H}}$ with preferably as small function norm ($\|\tilde{f}\|_{\tilde{\mathcal{H}}}$) as possible. As such, the estimation of $\mathbf{f}_{x} = [f(x_1), \dots, f(x_n)]^T$ with $\tilde{\mathbf{f}}_x = [\tilde{f}(x_1),\dots,\tilde{f}(x_n)]^T$ can be formulated as the following optimization problem: \[ \min\nolimits_{\|\tilde{f}\| \in \tilde{\mathcal{H}}}  \frac{1}{n}\|\mathbf{f}_{x}- \tilde{\mathbf{f}}_x\|_{2}^2 + \lambda \|\tilde{f}\|_{\tilde{\mathcal{H}}}.\]
\begin{restatable}{lma}{funcAppxOpmRe}\label{func_appx_opm}
Under Assumption A.1, suppose that the conditions on sampling measure $\tilde{l}$ from Theorem \ref{main:krr_risk} apply. If \[s \geq  5d_{\tilde{l}}\log\frac{16d_{\mathbf{K}}^{\lambda}}{\delta},\] 
then for all $ \delta \in (0,1)$ and $f \in \mathcal{H}$ with $\|f\|_{\mathcal{H}} \leq 1 $, with probability greater than $1-\delta$, we have that it holds
\begin{IEEEeqnarray}{rCl} 
\inf\nolimits_{\sqrt{s}\|\beta\|_2 \leq \sqrt{2}}\ \frac{1}{n}\|\mathbf{f}_{x}- \mathbf{Z}_q \beta\|_{2}^2 \leq 2\lambda. \nonumber
\end{IEEEeqnarray}
Equivalently, this can be rewritten as:
\begin{IEEEeqnarray}{rCl} 
\sup\nolimits_{\|f\|_{\mathcal{H}}\leq 1 }\inf\nolimits_{\|\tilde{f}\|_{\tilde{\mathcal{H}}} \leq \sqrt{2}}\ \frac{1}{n} \|\mathbf{f}_x- \tilde{\mathbf{f}}_x\|_2^2 \leq 2\lambda \nonumber.
\end{IEEEeqnarray}
\end{restatable}

Denote with $\hat{f}^{\lambda}$ the empirical estimator for the kernel ridge regression problem (see Eq.~\ref{main:krl_opm}) and let $\hat{\mathbf{f}}^{\lambda}_x = [\hat{f}^{\lambda}(x_1),\dots,\hat{f}^{\lambda}(x_n)]^T$ be its in-sample prediction. The next lemma is important in demonstrating the risk convergence rate and its proof is in Appendix \ref{apn:triangle_lma_pf}.

\begin{restatable}{lma}{TriangleLma}\label{triangle_lma}
Suppose that $\{v_i\}_{i=1}^s$ are independent samples selected according to a probability density function $q(v)$. $\{v_i\}_{i=1}^s$ forms the feature matrix $\mathbf{Z}_q$ and the corresponding RKHS $\tilde{\mathcal{H}}$. Define \[\tilde{f}^{\lambda} := \min\nolimits_{\tilde{f}\in \tilde{\mathcal{H}}} \frac{1}{n}\|\hat{\mathbf{f}}^{\lambda}_x - \tilde{\mathbf{f}}_x\|^2_2 + \lambda\|\tilde{f}\|_{\tilde{\mathcal{H}}}.\] Let $\tilde{\mathbf{f}}^{\lambda}_x$ be the in-sample prediction of $\tilde{f}^{\lambda}$, then we have
\begin{IEEEeqnarray}{rCl}
\frac{1}{n}\langle Y-\hat{\mathbf{f}}^{\lambda}_x, \hat{\mathbf{f}}^{\lambda}_x -\tilde{\mathbf{f}}^{\lambda}_x \rangle \leq \lambda.\nonumber
\end{IEEEeqnarray}
\end{restatable}

Equipped with Lemma \ref{func_appx_opm} and Lemma \ref{triangle_lma}, we are now ready to prove Theorem \ref{main:krr_risk}.
\begin{proof}
The proof relies on the decomposition of the learning risk of $\mathbb{E}(l_{\tilde{f}_{\beta}^{\lambda}})$ as follows
\begin{IEEEeqnarray}{rCl}
\mathbb{E}(l_{\tilde{f}_{\beta}^{\lambda}}) &=& \mathbb{E}(l_{\tilde{f}_{\beta}^{\lambda}})-\mathbb{E}_n(l_{\tilde{f}_{\beta}^{\lambda}}) \label{apn:krr_risk_rad}\\ 
&&+\mathbb{E}_n(l_{\tilde{f}_{\beta}^{\lambda}})-\mathbb{E}_n(l_{\hat{f}^{\lambda}} )\label{apn:krr_func_error}\\
&& + \mathbb{E}_n(l_{\hat{f}^{\lambda}} )-\mathbb{E}(l_{\hat{f}^{\lambda}})\label{apn:krr_risk_fh} \\
&&+\mathbb{E}(l_{\hat{f}^{\lambda}})- \mathbb{E}(l_{f_{\mathcal{H}}}) \label{apn:krr_risk} \\
&& + \mathbb{E}(l_{f_{\mathcal{H}}}).\nonumber
\end{IEEEeqnarray} 

For (\ref{apn:krr_risk_rad}), the bound is based on the Rademacher complexity of the reproducing kernel Hilbert space $\tilde{\mathcal{H}}$, where $\tilde{\mathcal{H}}$ corresponds to the approximated kernel $\tilde{k}$. We can upper bound the Rademacher complexity of this hypothesis space with Lemma \ref{apn:rad_rkhs}. As $l(y,f(x))$ is the squared error loss function with $y$ and $f(x)$ both bounded, we have that $l$ is a Lipschitz continuous function with some constant $L>0$. Hence,
\begin{IEEEeqnarray}{rCl}
(\ref{apn:krr_risk_rad}) &\leq& R_n(l\circ \tilde{\mathcal{H})} +\sqrt{\frac{8\log(2/\delta)}{n}} \nonumber\\
&\leq& \sqrt{2}L \frac{1}{n}\mathbb{E}_X\sqrt{\text{Tr}(\tilde{\mathbf{K}})} + \sqrt{\frac{8\log(2/\delta)}{n}}\nonumber \\
&\leq& \sqrt{2}L \frac{1}{n}\sqrt{\mathbb{E}_{X}\text{Tr}(\tilde{\mathbf{K}})}  + \sqrt{\frac{8\log(2/\delta)}{n}}\nonumber \\
&\leq&  \sqrt{2}L \frac{1}{n}\sqrt{nz_0^2}+\sqrt{\frac{8\log(2/\delta)}{n}}\nonumber\\
&=&  \frac{\sqrt{2}Lz_0}{\sqrt{n}} +\sqrt{\frac{8\log(2/\delta)}{n}} \in \mathcal{O}(\frac{1}{\sqrt{n}}),
\end{IEEEeqnarray} 
where in the first inequality we applied Lemma \ref{apn:risk_rad} to $\tilde{\mathcal{H}}$, which is a reproducing kernel Hilbert space with radius $\sqrt{2}$. In addition, the bound on $R_n(l\circ \tilde{\mathcal{H})}$ also utilize the Lipschitz composition property of the Rademacher complexity \cite{bartlett2002rademacher}.  For (\ref{apn:krr_risk_fh}), a similar reasoning can be applied to the unit ball in the reproducing kernel Hilbert space $\mathcal{H}$.

For (\ref{apn:krr_func_error}), we observe that 
\begin{IEEEeqnarray}{rCl}
\mathbb{E}_n(l_{\tilde{f}_{\beta}^{\lambda}})-\mathbb{E}_n(l_{\hat{f}^{\lambda}}) &=& \frac{1}{n}\|Y-\tilde{\mathbf{f}}_{\beta}^{\lambda}\|_2^2 - \frac{1}{n}\|Y-\hat{\mathbf{f}}^{\lambda}_x\|_2^2 \nonumber \\
&=& \frac{1}{n}\inf\nolimits_{\|\tilde{f}\|_{\tilde{\mathcal{H}}} \leq \sqrt{2}}\|Y-\tilde{\mathbf{f}}_x\|_2^2 - \frac{1}{n}\|Y-\hat{\mathbf{f}}^{\lambda}_x\|_2^2 \nonumber\\
&= & \frac{1}{n}\inf\nolimits_{\|\tilde{f}\|_{\tilde{\mathcal{H}}} \leq \sqrt{2}}\Big(\|Y-\hat{\mathbf{f}}^{\lambda}_x\|_2^2+\|\hat{\mathbf{f}}^{\lambda}_x-\tilde{\mathbf{f}}_x\|_2^2 \nonumber \\
&&+ 2 \langle Y- \hat{\mathbf{f}}^{\lambda}_x, \hat{\mathbf{f}}^{\lambda}_x- \tilde{\mathbf{f}}_x\rangle \Big)- \frac{1}{n}\|Y-\hat{\mathbf{f}}^{\lambda}_x\|_2^2 \nonumber\\
&=& \frac{1}{n}\inf\nolimits_{\|\tilde{f}\|_{\tilde{\mathcal{H}}} \leq \sqrt{2}}\left(\|\hat{\mathbf{f}}^{\lambda}_x-\tilde{\mathbf{f}}_x\|_2^2+ 2\langle Y- \hat{\mathbf{f}}^{\lambda}_x, \hat{\mathbf{f}}^{\lambda}_x- \tilde{\mathbf{f}}_x\rangle \right)\nonumber \\
&\leq & \frac{1}{n}\|\hat{\mathbf{f}}^{\lambda}_x-\tilde{\mathbf{f}}^{\lambda}_x\|_2^2 + \frac{2}{n}\langle Y- \hat{\mathbf{f}}^{\lambda}_x, \hat{\mathbf{f}}^{\lambda}_x- \tilde{\mathbf{f}}^{\lambda}_x\rangle  \nonumber \\
&\leq& \inf\nolimits_{\|\tilde{f}\|_{\tilde{\mathcal{H}}} \leq \sqrt{2}}\frac{1}{n}\|\hat{\mathbf{f}}^{\lambda}_x-\tilde{\mathbf{f}}_x\|_2^2 + 2\lambda ~~~(\text{by~Lemma~\ref{triangle_lma}})\nonumber\\
&\leq & \sup\nolimits_{\|f\|_{\mathcal{H}} \leq 1}\inf\nolimits_{\|\tilde{f}\|_{\tilde{\mathcal{H}}} \leq \sqrt{2}}\frac{1}{n}\|\mathbf{f}_x-\tilde{\mathbf{f}}_x\|_2^2 +  2\lambda \nonumber \\ &\leq & 4 \lambda, \nonumber
\end{IEEEeqnarray} 
Note that in the last step we have employed Lemma \ref{func_appx_opm}.
%For Eq.(\ref{apn:krr_risk}), by noticing that $\hat{f}^{\lambda}$ is the kernel ridge regression estimation for $Y$, it can be bound by $O(1/\sqrt{n})$ according to \cite{caponnetto2007optimal} in the worst case scenario.
Combining the three results, we derive 
\begin{align}
\mathbb{E}(l_{\tilde{f}_{\beta}^{\lambda}})-\mathbb{E}(l_{f_{\mathcal{H}}}) \leq 4\lambda +\mathcal{O}(\frac{1}{\sqrt{n}})+\mathbb{E}(l_{\hat{f}^{\lambda}})- \mathbb{E}(l_{f_{\mathcal{H}}}).
\end{align} 
\end{proof}

\subsection{Proof of Theorem \ref{main:sharp_risk_theo}} \label{sec:squ_refine_pf}
To prove Theorem \ref{main:sharp_risk_theo}, we rely on the notion of local Rademacher complexity. In general, the reason Theorem \ref{main:krr_risk} is not sharp is because when analysing Eqs.(\ref{apn:krr_risk_rad} and \ref{apn:krr_risk_fh}), we used the global Rademacher complexity of the whole RKHS. However, we could just analyse the local space around $f_\mathcal{H}$. In particular, we can apply Lemma \ref{apn:emp_local_rade} to Eqs.(\ref{apn:krr_risk_rad} and \ref{apn:krr_risk_fh}). 

%In order to do so, we need to find a proper \textit{sub-root} function. Below, we first show how to find the proper sub-root function $\hat{\psi}_n(r)$, followed by computing its fixed point by Lemma \ref{apn:local_kernel_risk}. This then leads to our desired Theorem \ref{main:sharp_risk_theo}.

%We also abbreviate the notation and denote with $l_f = l(f(x),y)$ and $\mathbb{E}_n (f) = 1/n \sum_{i=1}^n f(x_i)$. For the reproducing kernel Hilbert space $\mathcal{H}$, we denote the solution of the kernel ridge regression problem by $\hat{f}$. \\

To this end, we define the transformed function class as $l_{\mathcal{H}} \coloneqq \{(x,y) \rightarrow l(f(x),y) \mid  f\in \mathcal{H}\}$, for any reproducing kernel Hilbert space $\mathcal{H}$ and a loss function $l$. We now would like to apply Lemma \ref{apn:emp_local_rade} to the function class $l_{\mathcal{H}}$. First, it is easy to see that $\mathbb{E}(l_f^2) \leq B \mathbb{E}(l_f)$ for some constant $B$ since $l_f$ is bounded. Now if we assume that there exists a sub-root function $\hat{\psi}_n(r)$ such that it satisfies: \[\hat{\psi}_n(r) \geq c_1 \hat{R}_n\{l_f\in \mathrm{star}(l_{\mathcal{H}},0) \mid \mathbb{E}_n(l_f^2) \leq r\} + \frac{c_2}{n}\log\frac{1}{\delta} \ ,\] then with high probability, we have \[\mathbb{E}(l_f) \leq \frac{D}{D-1}\mathbb{E}_n(l_f) + \frac{6D}{B}\hat{r}^* + \frac{c_3}{n}\log\frac{1}{\delta} \ ,\] where $r^*$ is the fixed point of $\hat{\psi}_n(r)$.

Hence, our job now is to find a proper $\hat{\psi}_n(r)$ such that we can compute its fixed point $r^*$. To this end, we define $\hat{f}= \inf_{f\in \mathcal{H}}\mathbb{E}_n(l_f) =\inf_{f\in \mathcal{H}} \frac{1}{n}\sum_{i=1}^n (f(x_i)-y_i)^2 $ for given training sample $\{(x_i,y_i)\}_{i=1}^n$. We observe that for all $l_f \in l_{\mathcal{H}}$ it holds that
\begin{IEEEeqnarray}{rCl}
\mathbb{E}_n(l_f^2) &\geq & (\mathbb{E}_n(l_f))^2 ~~(x^2~ \text{is convex} ) \nonumber \\
&\geq &  (\mathbb{E}_nl_f)^2 - (\mathbb{E}_nl_{\hat{f}})^2 \nonumber \\
& \geq & 2\mathbb{E}_nl_{\hat{f}}\ \mathbb{E}_n(l_f -l_{\hat{f}})\nonumber ~~(\text{since}~a^2 -b^2 \geq 2b(a-b), \forall a,b \geq 0)\\
& \geq & \frac{2}{B}\mathbb{E}_nl_{\hat{f}} \mathbb{E}_n(f-\hat{f})^2 .\label{apn:emp_loss_bound}~~ \text{(Lemma \ref{apn:square_loss} in Appendix \ref{apn:sub_root_sec})}
\end{IEEEeqnarray}
Hence, to obtain a lower bound on $\mathbb{E}_nl_f^2$ expressed solely in terms of $\mathbb{E}_n(f -\hat{f})^2$, we need to find a lower bound of $\mathbb{E}_nl_{\hat{f}}$. Since $\mathbb{E}_n(l_{\hat{f}}) = \frac{1}{n}\sum_{i=1}^n (\hat{f}(x_i)-y_i)^2 $, we have $\mathbb{E}_P \left(\mathbb{E}_{n}l_{\hat{f}}\right) = \mathbb{E}(l_{f_{\mathcal{H}}}) \geq \sigma^2$, where we recall $\sigma^2$ is the variance of $\epsilon$ defined in Assumption A.1. In addition, for each pair of $(x_i,y_i)$, $l(\hat{f}(x_i), y_i)$ is bounded and i.i.d. Applying Hoeffding lemma, we can see that with probability greater than $1- \delta$ with $\delta \in (0,1)$, $\mathbb{E}_{n}(l_{\hat{f}})$ is lower bounded, we denote its lower bound as some constant $e_0$. Hence, with probability greater than $1-\delta$, Eq.~(\ref{apn:emp_loss_bound}) becomes \[\mathbb{E}_nl_f^2 \geq \frac{e_1}{B}\mathbb{E}_n(f-\hat{f})^2=:e_2\mathbb{E}_n(f-\hat{f})^2 . \]

% First, observe that it holds
% \begin{align*}
%   \mathbb{E}_nl_{\hat{f}} = \frac{1}{n} \|Y- \mathbf{K}(\mathbf{K} + n\lambda I)^{-1}Y\|^2 .
% \end{align*}
% Then, using this expression we derive
% \begin{IEEEeqnarray}{rCl}
% \mathbb{E}_nl_{\hat{f}} &= & \frac{1}{n} \|Y- \mathbf{K}(\mathbf{K} + n\lambda I)^{-1}Y\|^2 \nonumber \\
% & = & n\lambda^2Y^T(\mathbf{K} + n\lambda I)^{-2}Y \nonumber\\
% & \geq & \frac{n\lambda^2}{(\lambda_1 + n\lambda)^2}Y^TY\nonumber\\
% & =& \Bigg(\frac{n\lambda}{\lambda_1 + n\lambda}\Bigg)^2 \frac{1}{n}\sum_{i=1}^n y_i^2 \nonumber\\
% & \geq& \Bigg(\frac{n\lambda}{\lambda_1 + n\lambda}\Bigg)^2\sigma_{y}^2 \quad (\text{with}~ \frac{1}{n}\sum_{i=1}^n y_i^2 \geq \sigma_{y}^2 ) \nonumber\\
% & =&  \sigma_y^2 \Bigg(\frac{1}{1+\frac{\lambda_1}{n\lambda}}\Bigg)^2 \nonumber \\
% & \geq & \sigma_y^2 \Bigg(\frac{1}{\frac{\lambda_1}{n\lambda}+\frac{\lambda_1}{n\lambda}}\Bigg)^2 \nonumber \\
% & = & \frac{\sigma_y^2}{4}\Bigg(\frac{n\lambda}{\lambda_1}\Bigg)^2 = e_4 (n\lambda)^2 ,
% \end{IEEEeqnarray}
% where $e_4 =(\frac{\sigma_y}{2\lambda_1})^2$ is a constant.

As a result of this, we have the following inequality for the two function classes
\begin{IEEEeqnarray}{rCl}
\{l_f \in \mathrm{star}(l_{\mathcal{H}},0) \mid \mathbb{E}_n l_f^2 \leq r \} \subseteq \{l_f \in \mathrm{star}(l_{\mathcal{H}},0) \mid  \mathbb{E}_n (f-\hat{f})^2 \leq \frac{r}{e_2} \} .\nonumber 
\end{IEEEeqnarray}

Recall that for a function class $\mathcal{H}$, we denote its empirical Rademacher complexity by $\hat{R}_n(\mathcal{H})$. Then, we have the following inequality
\begin{align}
\begin{aligned}
& \hat{R}_n\{l_f \in \mathrm{star}(l_{\mathcal{H}},0) \mid \mathbb{E}_n l_f^2 \leq r \} \leq  &\\
&\hat{R}_n\{l_f \in \mathrm{star}(l_{\mathcal{H}},0) \mid \mathbb{E}_n (f-\hat{f})^2 \leq \frac{r}{e_2}\} = & \\
& \hat{R}_n\{l_f-l_{\hat{f}} \mid \mathbb{E}_n (f-\hat{f})^2 \leq \frac{r}{e_2} \ \wedge \ l_f \in \mathrm{star}(l_{\mathcal{H}},0) \} \leq & \\
&  L\hat{R}_n\{f-\hat{f} \mid \mathbb{E}_n (f-\hat{f})^2 \leq \frac{r}{e_2} \ \wedge \ f \in \mathcal{H} \} \leq &\\
&  L\hat{R}_n\{f-g \mid \mathbb{E}_n (f-g)^2 \leq \frac{r}{e_2}\  \wedge \ f, g \in \mathcal{H}\} \leq &\\
& 2L\hat{R}_n\{f \in \mathcal{H} \mid \mathbb{E}_nf^2 \leq \frac{1}{4}\frac{r}{e_2}\} = &\\
& 2L\hat{R}_n \{f \in \mathcal{H} \mid \mathbb{E}_nf^2 \leq e_3 r\} , & 
\end{aligned}\label{apn:rade_bound}
\end{align}
where the last inequality was proved in \citet[Corollary 6.7]{bartlett2005local}\footnote{The results come from the first three lines from the proof of Corollary 6.7.}. Now, since $\mathcal{H}$ is a reproducing kernel Hilbert space with kernel $k$, applying Lemma~\ref{apn:local_kernel} gives an upper bound of Eq.~(\ref{apn:rade_bound}). We can then derive the following theorem which gives us the proper sub-root function $\hat{\psi}_n$. The theorem is proved in Appendix \ref{apn:sub_root_reg}. 

\begin{restatable}{lma}{LocalKernelRisk}\label{apn:local_kernel_risk}
Assume $\{x_i,y_i\}_{i=1}^n$ is an independent sample from a probability measure $P$ defined on $\mathcal{X} \times \mathcal{Y}$, where $\mathcal{Y}$ has bounded range. Let $k$ be a positive definite kernel with the reproducing kernel Hilbert space $\mathcal{H}$ and let $\hat{\lambda}_1 \geq \cdots,\geq \hat{\lambda}_n$ be the eigenvalues of the normalized kernel Gram-matrix. Denote the squared error loss function by $l(f(x),y) = (f(x)-y)^2$ and fix $\delta \in (0,1)$. If \[\hat{\psi}_n(r) = 2Lc_1 \Bigg(\frac{2}{n}\sum_{i=1}^n\min\{r,\hat{\lambda}_i\}\Bigg)^{1/2} + \frac{c_2}{n}\log(\frac{1}{\delta}),\]
then for all $l_f \in l_\mathcal{H}$ and $D > 1$, with probability $1-\delta$,\[\mathbb{E}(l_f) \leq \frac{D}{D-1}\mathbb{E}_nl_f + \frac{6D}{B}\hat{r}^* + \frac{c_3}{n}\log(\frac{1}{\delta}).\] Moreover, the fixed point $\hat{r}^*$ defined with $\hat{r}^* = \hat{\psi}_n(\hat{r}^*)$ can be upper bounded by \[\hat{r}^* \leq \min_{0\leq h\leq n}\Big(e_0\frac{h}{n} + \sqrt{\frac{1}{n}\sum_{i>h}\hat{\lambda}_i}\Big),\]
where $e_0$ is a constant.
\end{restatable}

We are now ready to deliver the proof of Theorem \ref{main:sharp_risk_theo}.
\begin{proof}
We decompose $\mathbb{E}(l_{\tilde{f}_{\beta}^{\lambda}})$ with $D > 1$ as follows:
\begin{IEEEeqnarray}{rCl}
\mathbb{E}(l_{\tilde{f}_{\beta}^{\lambda}})-\mathbb{E}(l_{f_\mathcal{H}}) &\leq& \mathbb{E}(l_{\tilde{f}_{\beta}^{\lambda}}) - \frac{D}{D-1}\mathbb{E}_n(l_{\tilde{f}_{\beta}^{\lambda}}) \label{apn:sharp_rff_risk}\\
&&+ \frac{D}{D-1}(\mathbb{E}_n(l_{\tilde{f}_{\beta}^{\lambda}})- \mathbb{E}_n(l_{\hat{f}^{\lambda}}))\label{apn:sharp_rff_krr}\\
&& + \frac{D}{D-1}\mathbb{E}_n(l_{\hat{f}^{\lambda}})-\mathbb{E}(l_{\hat{f}^{\lambda}})\label{apn:sharp_krr_risk}\\
&& + \mathbb{E}(l_{\hat{f}^{\lambda}})-\mathbb{E}(l_{f_\mathcal{H}}) .\label{apn:sharp_gen_error}
\end{IEEEeqnarray}

We have already demonstrated that \[\text{Eq.}\ (\ref{apn:sharp_rff_krr}) \leq 4\frac{D}{D-1}\lambda .\]
For Eqs.(\ref{apn:sharp_rff_risk})  we apply Lemma \ref{apn:local_kernel_risk}. For Eq.~(\ref{apn:sharp_krr_risk}) use the opposite side of Lemma \ref{apn:local_kernel_risk}. The proof of the opposite side is similar to Lemma \ref{apn:local_kernel_risk} and is a direct consequence of the second part of \cite[Theorem 4.1]{bartlett2005local}. However, note that $\tilde{f}_{\beta}^{\lambda}$ and $\hat{f}^{\lambda}$ belong to different reproducing kernel Hilbert spaces. As a result, we have
\begin{IEEEeqnarray}{rCl}
\text{Eq.}\ (\ref{apn:sharp_rff_risk}) &\leq& \frac{6D}{B}\hat{r}^*_{\tilde{\mathcal{H}}} + \mathcal{O}(1/n) \nonumber\\
\text{Eq.}\ (\ref{apn:sharp_krr_risk}) &\leq& \frac{6D}{B}\hat{r}^*_{\mathcal{H}} + \mathcal{O}(1/n)\nonumber 
\end{IEEEeqnarray}
Now, combining these inequalities together we deduce
\begin{IEEEeqnarray}{rCl}
\mathbb{E}(l_{\tilde{f}_{\beta}^{\lambda}})-\mathbb{E}(l_{f_\mathcal{H}})& \leq&  \frac{6D}{B}\hat{r}^*_{\tilde{\mathcal{H}}} + \frac{6D}{B}\hat{r}^*_{\mathcal{H}} +4\frac{D}{D-1}\lambda +  \mathcal{O}(1/n) \nonumber\\
&&+ \mathbb{E}(l_{\hat{f}^{\lambda}})-\mathbb{E}(l_{f_\mathcal{H}})\nonumber\\
&\leq& \frac{12D}{B}\hat{r}^*_{\mathcal{H}} + 4\frac{D}{D-1}\lambda +  \mathcal{O}(1/n) \nonumber\\
&&+ \mathbb{E}(l_{\hat{f}^{\lambda}})-\mathbb{E}(l_{f_\mathcal{H}}) . \nonumber
\end{IEEEeqnarray}
The last inequality holds because the eigenvalues of the Gram-matrix for the reproduing kernel Hilbert space $\tilde{\mathcal{H}}$ decay faster than the eigenvalues of $\mathcal{H}$. As a result of this, we have that $\hat{r}^*_{\tilde{\mathcal{H}}} \leq \hat{r}^*_{\mathcal{H}}$. 

Now, Lemma \ref{apn:local_kernel_risk} implies that
\begin{IEEEeqnarray}{rCl}
\hat{r}^*_{\mathcal{H}}\leq \text{min}_{0\leq h\leq n}\Big(e_0\frac{h}{n} + \sqrt{\frac{1}{n}\sum_{i>h}\hat{\lambda}_i}\Big).
\end{IEEEeqnarray}
There are two cases worth discussing here. On the one hand, if the eigenvalues of $\mathbf{K}$ decay exponentially, we have \[\hat{r}^*_{\mathcal{H}}\leq O\Bigg(\frac{\log n}{n}\Bigg)\] by substituting $h = \ceil{\log n}$. Now, according to \citet{caponnetto2007optimal}
\begin{align*}
    \mathbb{E}(l_{\hat{f}^{\lambda}})-\mathbb{E}(l_{f_\mathcal{H}}) \in O\Bigg( \frac{\log n}{n}\Bigg), 
\end{align*}
and, thus, if we set $\lambda \propto \frac{\log n}{n} $ then the learning risk rate can be upper bounded by \[\mathbb{E}(l_{\tilde{f}_{\beta}^{\lambda}})-\mathbb{E}(l_{f_\mathcal{H}}) \in O\Bigg(\frac{\log n}{n}\Bigg).\]

On the other hand, if $\mathbf{K}$ has finitely many non-zero eigenvalues ($t$), we then have that \[\hat{r}^*_{\mathcal{H}} \in O\Bigg(\frac{1}{n}\Bigg),\] by substituting $h\geq t$. Moreover, in this case, $\mathbb{E}(l_{\hat{f}^{\lambda}})-\mathbb{E}(l_{f_\mathcal{H}}) \in \mathcal{O}( \frac{1}{n})$ and setting $\lambda \propto \frac{1}{n}$, we deduce that \[\mathbb{E}(l_{\tilde{f}_{\beta}^{\lambda}})-\mathbb{E}(l_{f_\mathcal{H}}) \leq O\Bigg(\frac{1}{n} \Bigg).\]
\end{proof}

\subsection{Proof of Theorem \ref{main:svm_risk}}\label{sec:risk_decom}
\begin{proof}
The proof is similar to Theorem \ref{main:krr_risk}. In particular, we decompose the expected learning risk as
\begin{IEEEeqnarray}{rCl}
\mathbb{E}(l_{g_{\beta}^{\lambda}}) &=& \mathbb{E}(l_{g_{\beta}^{\lambda}})-\mathbb{E}_n(l_{g_{\beta}^{\lambda}}) \label{apn:svm_risk_rad}\\ 
&&+\mathbb{E}_n(l_{g_{\beta}^{\lambda}})-\mathbb{E}_n(l_{g_{\mathcal{H}}})\label{apn:svm_fun_error}\\
&& + \mathbb{E}_n(l_{g_{\mathcal{H}}})-\mathbb{E}(l_{g_{\mathcal{H}}})\label{apn:svm_fh} \\
&&+ \mathbb{E}(l_{g_{\mathcal{H}}}).\nonumber
\end{IEEEeqnarray} 
Now, (\ref{apn:svm_risk_rad}) and (\ref{apn:svm_fh}) can be upper bounded similar to Theorem \ref{main:krr_risk}, through the Rademacher complexity bound from Lemma \ref{apn:risk_rad}. For (\ref{apn:svm_fun_error}), we have
\begin{align*}
\begin{aligned}
& \mathbb{E}_n(l_{g_{\beta}^{\lambda}})-\mathbb{E}_n(l_{g_{\mathcal{H}}}) =& \\ 
& \frac{1}{n}\sum_{i=1}^n l(y_i,g_{\beta}^{\lambda}(x_i))-\frac{1}{n}\sum_{i=1}^n l(y_i,g_{\mathcal{H}}(x_i)) =& \\
& \frac{1}{n}\inf\nolimits_{\|g_{\beta}\|}\sum_{i=1}^n l(y_i,g_{\beta}(x_i)) -\frac{1}{n}\sum_{i=1}^n l(y_i,g_{\mathcal{H}}(x_i)) &\\
%&=&\frac{1}{n}\inf_{\|g_{\beta}\|}\sum_{i=1}^n\Big[l(y_i,g_{\beta}(x_i))-l(y_i,g_{\mathcal{H}}(x_i)) \Big]\nonumber\\
&\leq \inf\nolimits_{\|g_{\beta}\|}\frac{1}{n}\sum_{i=1}^n |g_{\beta}(x_i)-g_{\mathcal{H}}(x_i)| &\\
&\leq \inf\nolimits_{\|g_{\beta}\|}\sqrt{\frac{1}{n}\sum_{i=1}^n |g_{\beta}(x_i)-g_{\mathcal{H}}(x_i)|^2} &\\
&\leq  \sup\nolimits_{\|g\|}\inf\nolimits_{\|g_{\beta}\|}\sqrt{\frac{1}{n}\|g-g_{\beta}\|_2^2} &\\
&\leq  \sqrt{2\lambda}. &
\end{aligned}
\end{align*} 
\end{proof}

\subsection{Proof of Theorem \ref{main:svm_refine}}
To prove Theorem \ref{main:svm_refine}, we adopt the similar strategy from the proof of Theorem \ref{main:sharp_risk_theo}, where we utilize the properties of the local Rademacher complexities by applying Lemma \ref{apn:emp_local_rade} to the decomposition of the learning risk in the Lipschitiz continuous loss case, namely Eqs. (\ref{apn:svm_risk_rad}, \ref{apn:svm_fh}). In order to do that, we need two steps. The first step is to find a proper sub-root function $\hat{\psi}_n(r)$. The second step is to find the fixed point of $\hat{\psi}_n(r)$. Hence, the following is devoted to solving these two problems. \\

First we recall that we define the transformed function class as $g_{\mathcal{H}} \coloneqq \{(x,y) \rightarrow g(f(x),y) \mid  g\in \mathcal{H}\}$ for a Lipschitz continuous loss function $l$ and $\hat{g} = \inf_{g\in\mathcal{H}} \frac{1}{n}\sum_{i=1}^nl(g(x_i), y_i)$. We observe that for any $l_g \in l_{\mathcal{H}}$, we have 
\begin{IEEEeqnarray}{rCl}
\mathbb{E}_n(l_g^2) &\geq & (\mathbb{E}_n(l_g))^2 ~~(x^2~ \text{is convex} ) \nonumber \\
&\geq &  (\mathbb{E}_nl_g)^2 - (\mathbb{E}_nl_{\hat{g}})^2 \nonumber \\
& \geq & 2\mathbb{E}_nl_{\hat{g}}\ \mathbb{E}_n(l_g -l_{\hat{g}})\nonumber ~~(\text{since}~a^2 -b^2 \geq 2b(a-b), \forall a,b \geq 0)\\
& \geq & \frac{2}{B}\mathbb{E}_nl_{\hat{g}} \mathbb{E}_n(g-\hat{g})^2 .\label{apn:emp_loss_bound_lip}~~ (\text{Assumption B.$4$})
\end{IEEEeqnarray}

By the similar reason for Eq.~(\ref{apn:emp_loss_bound}) in Section \ref{sec:squ_refine_pf}, we can see that with probability greater than $1- \delta$ with $\delta \in (0,1)$, $\mathbb{E}_n(l_{\hat{g}})$ is lower bounded by a constant, denoted as $b_0$. Hence, Eq.~(\ref{apn:emp_loss_bound_lip}) becomes \[\mathbb{E}_n (l_g^2) \geq b_1\mathbb{E}_n(g-\hat{g})^2.\]

Similar to Section \ref{sec:squ_refine_pf}, we have 
\begin{IEEEeqnarray}{rCl}
\{l_g \in l_{\mathcal{H}} \mid \mathbb{E}_n l_g^2 \leq r \} \subseteq \{l_g \in l_{\mathcal{H}} \mid  \mathbb{E}_n (g-\hat{g})^2 \leq \frac{r}{b_1} \} .\nonumber 
\end{IEEEeqnarray}

This further implies that
\begin{IEEEeqnarray}{rCl}
\hat{R}_n\{l_g \in l_{\mathcal{H}} \mid \mathbb{E}_n l_g^2 \leq r \} \leq  2L\hat{R}_n \{g \in \mathcal{H} \mid \mathbb{E}_ng^2 \leq b_2 r\}, \nonumber
\end{IEEEeqnarray}
where we recall $L$ is the Lipschitz constant of the loss function $l$. By appealing to Lemma \ref{apn:local_kernel}, we obtain an upper bound of $\hat{R}_n \{g \in \mathcal{H} \mid \mathbb{E}_ng^2 \leq b_2 r\}$. Apply Lemma \ref{apn:local_kernel_risk} to the function class $l_{\mathcal{H}}$, we have with probability greater than $1 - \delta$, for all $l_g \in l_\mathcal{H}$ and $D > 1$, 
\begin{IEEEeqnarray}{rCl}
\mathbb{E}(l_g) \leq \frac{D}{D-1}\mathbb{E}_nl_g + \frac{12D}{B}\hat{r}^* + \frac{c_1}{n}\log(\frac{1}{\delta}). \label{eqn:local_risk_lip}
\end{IEEEeqnarray}
Moreover, the fixed point $\hat{r}^*$ can be upper bounded by \[\hat{r}^* \leq \min_{0\leq h\leq n}\Big(b_0\frac{h}{n} + \sqrt{\frac{1}{n}\sum_{i>h}\hat{\lambda}_i}\Big),\]
where $b_0$ is a constant. 

With this result in mind, we analyse the risk decomposition as follows:
In particular, we decompose the expected learning risk as
\begin{IEEEeqnarray}{rCl}
\mathbb{E}(l_{g_{\beta}^{\lambda}}) &=& \mathbb{E}(l_{g_{\beta}^{\lambda}})-\frac{D}{D-1} \mathbb{E}_n(l_{g_{\beta}^{\lambda}}) \label{apn:svm_risk_rad_refin}\\ 
&&+\frac{D}{D-1} \mathbb{E}_n(l_{g_{\beta}^{\lambda}})-\frac{D}{D-1}\mathbb{E}_n(l_{g_{\mathcal{H}}})\label{apn:svm_fun_error_refin}\\
&& + \frac{D}{D-1}\mathbb{E}_n(l_{g_{\mathcal{H}}})-\mathbb{E}(l_{g_{\mathcal{H}}})\label{apn:svm_fh_refin} \\
&&+ \mathbb{E}(l_{g_{\mathcal{H}}}).\nonumber
\end{IEEEeqnarray} 

For Eq.~(\ref{apn:svm_fun_error_refin}), we can upper bound this by $\frac{D}{D-1}\sqrt{2\lambda}$ according to Section \ref{sec:risk_decom}. For Eq.~(\ref{apn:svm_risk_rad_refin}), since $g_{\beta}^{\lambda} \in \tilde{\mathcal{H}}$, we can upper bound using Eq.~(\ref{eqn:local_risk_lip}) applied to $\tilde{\mathcal{H}}$. We repeat the same procedure for Eq.~(\ref{apn:svm_fh_refin}) using Eq.~(\ref{eqn:local_risk_lip}) applied to $\mathcal{H}$. Combing all of the results, we obtain that with probability greater than $1-\delta$, 
\begin{IEEEeqnarray}{rCl}
\mathbb{E}(l_{g_{\beta}^{\lambda}}) \leq \frac{12D}{B_0}\hat{r}_{\mathcal{H}}^* + \frac{D}{D-1}\sqrt{2\lambda} + \mathcal{O}(1/n) + \mathbb{E}(l_{g_{\mathcal{H}}}).
\end{IEEEeqnarray} 
where $\hat{r}^*_{\mathcal{H}}$ can be upper bounded as:
\begin{IEEEeqnarray}{rCl}
\hat{r}^*_{\mathcal{H}}\leq \text{min}_{0\leq h\leq n}\Big(b_0\frac{h}{n} + \sqrt{\frac{1}{n}\sum_{i>h}\hat{\lambda}_i}\Big).
\end{IEEEeqnarray}

\subsection{Proof of Theorem \ref{main:algo1_conv}}
\begin{proof}
Suppose the examples $\{x_i,y_i\}_{i=1}^n$ are independent and identically distributed and that the kernel $k$ can be decomposed as in Eq.~(\ref{main:krl_dec}). Let $\{v_i\}_{i=1}^s$ be an independent sample selected according to $p(v)$. Then, using these $s$ features we can approximate the kernel as 
\begin{IEEEeqnarray}{rCl}
\tilde{k}(x,y) &=& \frac{1}{s}\sum_{i=1}^sz(v_i,x)z(v_i,y) \nonumber\\
&=& \int_{V}z(v,x)z(v,y)d\hat{P}(v), \label{apn:krl_emp_dec}
\end{IEEEeqnarray} 
where $\hat{P}$ is the empirical measure on $\{v_i\}_{i=1}^s$. Denote the reproducing kernel Hilbert space associated with kernel $\tilde{k}$ by $\tilde{\mathcal{H}}$ and suppose that kernel ridge regression was performed with the approximate kernel $\tilde{k}$. 
From Theorem~\ref{main:krr_risk} and Corollary~\ref{main:krr_risk_cor2}, it follows that if \[s \geq  \frac{7z_0^2}{\lambda}\log\frac{16d_{\mathbf{K}}^{\lambda}}{\delta},\] then for all $\delta \in (0,1)$, with probability $1-\delta$, the risk convergence rate of the kernel ridge regression estimator based on random Fourier features can be upper bounded by
\begin{IEEEeqnarray}{rCl}
\mathbb{E}(l_{f_{\alpha}^{\lambda}})&\leq&4\lambda+\mathcal{O}\Bigg(\frac{1}{\sqrt{n}}\Bigg) + \mathbb{E}(l_{f_{\mathcal{H}}}). \label{apn:1_lev_risk}
\end{IEEEeqnarray}
Note that in Eq.~(\ref{apn:1_lev_risk}) we have used the fact that $\mathbb{E}(l_{f_{\mathcal{H}}})$ differs with $\mathbb{E}(l_{\hat{f}^{\lambda}})$ by at most $\mathcal{O}(1/\sqrt{n})$. Let $f_{\tilde{\mathcal{H}}}$ be the function in the reproducing kernel Hilbert space $\tilde{\mathcal{H}}$ achieving the minimal risk, i.e., $\mathbb{E}(l_{f_{\tilde{\mathcal{H}}}}) = \inf_{f\in \tilde{\mathcal{H}}}\mathbb{E}(l_f)$. We now treat $\tilde{k}$ as the actual kernel that can be decomposed via the expectation with respect to the empirical measure in Eq.~(\ref{apn:krl_emp_dec}) and re-sample features from the set $\{v_i\}_{i=1}^s$, but this time the sampling is performed using the optimal ridge leverage scores. As $\tilde{k}$ is the actual kernel, it follows from Eq.~(\ref{main:lev_fun}) that the leverage function in this case can be defined by
\[l_{\lambda}(v) = p(v)\mathbf{z}_{v}(\mathbf{x})^{T}(\tilde{\mathbf{K}}+n\lambda I)^{-1}\mathbf{z}_{v}(\mathbf{x}).\]
Now, observe that
\begin{IEEEeqnarray}{rCl}
l_{\lambda}(v_i) = p(v_i)[\mathbf{Z}_s^T(\tilde{\mathbf{K}}+n\lambda I)^{-1}\mathbf{Z}_s]_{ii}\nonumber
\end{IEEEeqnarray}
where $[A]_{ii}$ denotes the $i$th diagonal element of matrix $A$. As $\tilde{\mathbf{K}} = (1/s)\mathbf{Z}_s\mathbf{Z}_s^T$, then the Woodbury inversion lemma implies that
\begin{IEEEeqnarray}{rCl}
l_{\lambda}(v_i) = p(v_i)[\mathbf{Z}_s^T\mathbf{Z}_s(\frac{1}{s}\mathbf{Z}_s^T\mathbf{Z}_s+n\lambda I)^{-1}]_{ii}.\nonumber
\end{IEEEeqnarray}
If we let $l_{\lambda}(v_i) = p_i$, then the optimal distribution for $\{v_i\}_{i=1}^s$ is multinomial with individual probabilities $q(v_i) = p_i/(\sum_{j=1}^sp_j)$. Hence, we can re-sample $m$ features according to $q(v)$ and perform linear ridge regression using the sampled leverage weighted features. Denoting this estimator with $\tilde{f}_m^{\lambda^*}$ and the corresponding number of degrees of freedom with $ d_{\tilde{\mathbf{K}}}^{\lambda} = \text{Tr}\tilde{\mathbf{K}}(\tilde{\mathbf{K}}+n\lambda)^{-1}$, we deduce (using Theorem~\ref{main:krr_risk} and Corollary~\ref{main:krr_risk_cor1})
\begin{IEEEeqnarray}{rCl}
\mathbb{E}(l_{\tilde{f}_m^{\lambda^*}})&\leq& 4\lambda^*+\mathcal{O}\Bigg(\frac{1}{\sqrt{n}}\Bigg)+ \mathbb{E}(l_{f_{\tilde{\mathcal{H}}}}), \label{apn:2_lev_risk}
\end{IEEEeqnarray}
with the number of features $l \propto d_{\tilde{\mathbf{K}}}^{\lambda}$, and we again used the fact that $\mathbb{E}(l_{f_{\tilde{\mathcal{H}}}})$ differs with $\mathbb{E}(l_{{\tilde{f}^{\lambda*}}_m})$ by at most $\mathcal{O}(1/\sqrt{n})$. 

As $f_{\tilde{\mathcal{H}}}$ is the function achieving the minimal risk over $\tilde{\mathcal{H}}$, we can conclude that $\mathbb{E}(l_{f_{\tilde{\mathcal{H}}}}) \leq \mathbb{E}(l_{f_{\alpha}^{\lambda}})$. Now, combining Eq.~(\ref{apn:1_lev_risk}) and (\ref{apn:2_lev_risk}), we obtain the final bound on $\mathbb{E}(l_{\tilde{f}_m^{\lambda^*}})$.
\end{proof}

\section{Discussion}
We have investigated the generalization properties of learning with random Fourier features in the context of different kernel methods: kernel ridge regression, support vector machines, and kernel logistic regression. In particular, we have given generic bounds on the number of features required for consistency of learning with two sampling strategies: \emph{leverage weighted} and \emph{plain random Fourier features}. The derived convergence rates account for the complexity of the target hypothesis and the structure of the reproducing kernel Hilbert space with respect to the marginal distribution of a data-generating process. In addition to this, we have also proposed an algorithm for fast approximation of empirical leverage scores and demonstrated its superiority in both theoretical and empirical analyses.

For kernel ridge regression, \citet{avron2017random} and \citet{rudi2017generalization} have extensively analyzed the performance of learning with random Fourier features. In particular, \citet{avron2017random} have shown that $o(n)$ features are enough to guarantee a good estimator in terms of its\emph{ empirical risk}. The authors of that work have also proposed a modified data-dependent sampling distribution and demonstrated that a further reduction in the number of random Fourier features is possible for leverage weighted sampling. However, their results do not provide a convergence rate for the \emph{learning risk} of the estimator which could still potentially imply that computational savings come at the expense of statistical efficiency. Furthermore, the modified sampling distribution can only be used in the $1$D Gaussian kernel case. While~\citet{avron2017random} focus on bounding the empirical risk of an estimator, \citet{rudi2017generalization} give a comprehensive study of the generalization properties of random Fourier features for kernel ridge regression by bounding the learning risk of an estimator. The latter work for the first time shows that $\Omega(\sqrt{n}\log n)$ features are sufficient to guarantee the (kernel ridge regression) minimax rate and observes that further improvements to this result are possible by relying on a data-dependent sampling strategy. However, such a distribution is defined in a complicated way and it is not clear how one could devise a practical algorithm by sampling from it. While in our analysis of learning with random Fourier features we also bound the learning risk of an estimator, the analysis is not restricted to kernel ridge regression and covers other kernel methods such as support vector machines and kernel logistic regression. In addition to this, our derivations are much simpler compared to~\citet{rudi2017generalization} and provide sharper bounds in some cases. More specifically, we have demonstrated that $\Omega(\sqrt{n}\log\log n)$ features are sufficient to attain the minimax rate in the case where eigenvalues of the Gram matrix have a geometric/exponential decay. In other cases, we have recovered the results from \citet{rudi2017generalization}. Another important difference with respect to this work is that we consider a data-dependent sampling distribution based on empirical ridge leverage scores, showing that it can further reduce the number of features and in this way provide a more effective estimator. %In addition, by properly defining the empirical leverage function, we provide a simple formula for the empirical leverage distribution and reduce the size of $s$ when this modified data-dependent distribution is used.
%Recently, \citet{sun2018but} have provided novel bounds for random Fourier features in the SVM setting using local Rademacher complexity. However, the results are restricted to hinge loss only, lacking a clear means to generalize to other loss functions. 

In addition to the squared error loss, we also investigate the properties of learning with random Fourier features using the Lipschitz continuous loss functions. Both \citet{rahimi2009weighted} and \citet{bach2017equivalence} have studied this problem setting and obtained that $\Omega(n)$ features are needed to ensure $\mathcal{O}(1/\sqrt{n})$ learning risk convergence rate. Moreover, \citet{bach2017equivalence} has defined an optimal sampling distribution by referring to the leverage score function based on the integral operator and shown that the number of features can be significantly reduced when the eigenvalues of a Gram matrix exhibit a fast decay. The $\Omega(n)$ requirement on the number of features is too restrictive and precludes any computational savings. Also, the optimal sampling distribution is typically intractable. In our analysis, through assuming the realizable case, we have demonstrated that for the first time, $\mathcal{O}(\sqrt{n})$ features are possible to guarantee $\mathcal{O}(\frac{1}{\sqrt{n}})$ risk convergence rate. In extreme cases, where the complexity of target function is small, constant features is enough to guarantee fast risk convergence. Moreover, we also provide a much simpler form of the empirical leverage score distribution and demonstrate that the number of features can be significantly smaller than $n$, without incurring any loss of statistical efficiency.%, recovering the result from \cite{bach2017equivalence}. In the case of plain RFFs, we also recover results from \cite{rahimi2009weighted}.

Having given risk convergence rates for learning with random Fourier features, we provide a fast and practical algorithm for sampling them in a data-dependent way, such that they approximate the ridge leverage score distribution. In the kernel ridge regression setting, our theoretical analysis demonstrates that, compared to spectral measure sampling, significant computational savings can be achieved while preserving the statistical properties of the estimators. Furthermore, we verify our findings empirically on simulated and real-world datasets. An interesting extension of our empirical analysis would be a thorough and comprehensive comparison of the proposed leverage weighted sampling scheme to other recently proposed data-dependent strategies for selecting good features~\citep[e.g.,][]{rudi2018fast}, as well as a comparison to the Nystr\"om method. %A thorough comparison would shed light on how to effectively sample according the degree of freedom of the learning problem.

%An open question for the future work is whether a further reduction in the number of features is possible. 
%An interesting direction for further research would be a comprehensive theoretical comparison of the proposed leverage weighted random Fourier feature sampler to the Nystr{\"o}m method~\citep{yang2012nystrom}. Namely, both methods use data-dependent features and this has been previously considered an argument in favor of the latter.
%\newpage
%\newpage
%\section*{References}
%\medskip

%\section*{Acknowledgements}

%\textbf{Do not} include acknowledgements in the initial version of
%the paper submitted for blind review.

%If a paper is accepted, the final camera-ready version can (and
%probably should) include acknowledgements. In this case, please
%place such acknowledgements in an unnumbered section at the
%end of the paper. Typically, this will include thanks to reviewers
%who gave useful comments, to colleagues who contributed to the ideas,
%and to funding agencies and corporate sponsors that provided financial
%support.
\vspace*{2ex}
{
\small\selectfont
\noindent \textbf{Acknowledgments:} We thank Fadhel Ayed, Qinyi Zhang and Anthony Caterini for fruitful discussion on some of the results as well as for proofreading of this paper. This work was supported by the EPSRC and MRC through the OxWaSP CDT programme (EP/L016710/1). Dino Oglic was supported in part by EPSRC grant EP/R012067/1. Zhu Li was supported in part by Huawei UK.
}

\nocite{OglicG16}

{
\small
\bibliography{ka}

\begin{thebibliography}{39}
\providecommand{\natexlab}[1]{#1}
\providecommand{\url}[1]{\texttt{#1}}
\expandafter\ifx\csname urlstyle\endcsname\relax
  \providecommand{\doi}[1]{doi: #1}\else
  \providecommand{\doi}{doi: \begingroup \urlstyle{rm}\Url}\fi

\bibitem[Alaoui and Mahoney(2015)]{alaoui2015fast}
Ahmed Alaoui and Michael~W Mahoney.
\newblock Fast randomized kernel ridge regression with statistical guarantees.
\newblock In \emph{Advances in Neural Information Processing Systems}, pages
  775--783, 2015.

\bibitem[Avron et~al.(2017)Avron, Kapralov, Musco, Musco, Velingker, and
  Zandieh]{avron2017random}
Haim Avron, Michael Kapralov, Cameron Musco, Christopher Musco, Ameya
  Velingker, and Amir Zandieh.
\newblock Random {F}ourier features for kernel ridge regression: Approximation
  bounds and statistical guarantees.
\newblock In \emph{International Conference on Machine Learning}, pages
  253--262, 2017.

\bibitem[Bach(2013)]{bach2013sharp}
Francis Bach.
\newblock Sharp analysis of low-rank kernel matrix approximations.
\newblock In \emph{Conference on Learning Theory}, pages 185--209, 2013.

\bibitem[Bach(2017{\natexlab{a}})]{bach2017breaking}
Francis Bach.
\newblock Breaking the curse of dimensionality with convex neural networks.
\newblock \emph{Journal of Machine Learning Research}, 18\penalty0
  (19):\penalty0 1--53, 2017{\natexlab{a}}.

\bibitem[Bach(2017{\natexlab{b}})]{bach2017equivalence}
Francis Bach.
\newblock On the equivalence between kernel quadrature rules and random feature
  expansions.
\newblock \emph{Journal of Machine Learning Research}, 18\penalty0
  (21):\penalty0 1--38, 2017{\natexlab{b}}.

\bibitem[Bartlett and Mendelson(2002)]{bartlett2002rademacher}
Peter~L Bartlett and Shahar Mendelson.
\newblock Rademacher and {G}aussian complexities: Risk bounds and structural
  results.
\newblock \emph{Journal of Machine Learning Research}, 3\penalty0
  (Nov):\penalty0 463--482, 2002.

\bibitem[Bartlett et~al.(2005)Bartlett, Bousquet, Mendelson,
  et~al.]{bartlett2005local}
Peter~L Bartlett, Olivier Bousquet, Shahar Mendelson, et~al.
\newblock Local {R}ademacher complexities.
\newblock \emph{The Annals of Statistics}, 33\penalty0 (4):\penalty0
  1497--1537, 2005.

\bibitem[Bartlett et~al.(2006)Bartlett, Jordan, and
  McAuliffe]{bartlett2006convexity}
Peter~L Bartlett, Michael~I Jordan, and Jon~D McAuliffe.
\newblock Convexity, classification, and risk bounds.
\newblock \emph{Journal of the American Statistical Association}, 101\penalty0
  (473):\penalty0 138--156, 2006.

\bibitem[Berlinet and Thomas-Agnan(2011)]{berlinet2011reproducing}
Alain Berlinet and Christine Thomas-Agnan.
\newblock \emph{Reproducing kernel Hilbert spaces in probability and
  statistics}.
\newblock Springer Science \& Business Media, 2011.

\bibitem[Bochner(1932)]{Bochner32}
Salomon Bochner.
\newblock Vorlesungen \"uber {F}ouriersche {I}ntegrale.
\newblock In \emph{Akademische Verlagsgesellschaft}, 1932.

\bibitem[Caponnetto and De~Vito(2007)]{caponnetto2007optimal}
Andrea Caponnetto and Ernesto De~Vito.
\newblock Optimal rates for the regularized least-squares algorithm.
\newblock \emph{Foundations of Computational Mathematics}, 7\penalty0
  (3):\penalty0 331--368, 2007.

\bibitem[Chang and Lin(2011)]{CC01a}
Chih-Chung Chang and Chih-Jen Lin.
\newblock {LIBSVM}: A library for support vector machines.
\newblock \emph{ACM Transactions on Intelligent Systems and Technology},
  2:\penalty0 27:1--27:27, 2011.
\newblock Software available at \url{http://www.csie.ntu.edu.tw/~cjlin/libsvm}.

\bibitem[Dheeru and Karra~Taniskidou(2017)]{Dua:2017}
Dua Dheeru and Efi Karra~Taniskidou.
\newblock {UCI} machine learning repository, 2017.
\newblock URL \url{http://archive.ics.uci.edu/ml}.

\bibitem[Hastie(2017)]{hastie2017generalized}
Trevor~J Hastie.
\newblock Generalized additive models.
\newblock In \emph{Statistical models in S}, pages 249--307. Routledge, 2017.

\bibitem[Koltchinskii(2011)]{koltchinskii2011oracle}
Vladimir Koltchinskii.
\newblock \emph{Oracle Inequalities in Empirical Risk Minimization and Sparse
  Recovery Problems: Ecole d’Et{\'e} de Probabilit{\'e}s de Saint-Flour
  XXXVIII-2008}, volume 2033.
\newblock Springer Science \& Business Media, 2011.

\bibitem[Mahoney and Drineas(2009)]{mahoney2009cur}
Michael~W Mahoney and Petros Drineas.
\newblock C{UR} matrix decompositions for improved data analysis.
\newblock \emph{Proceedings of the National Academy of Sciences}, 106\penalty0
  (3):\penalty0 697--702, 2009.

\bibitem[Mendelson(2002)]{mendelson2002improving}
Shahar Mendelson.
\newblock Improving the sample complexity using global data.
\newblock \emph{IEEE transactions on Information Theory}, 48\penalty0
  (7):\penalty0 1977--1991, 2002.

\bibitem[Nystr{\"{o}}m(1930)]{Nystrom30}
Evert~J. Nystr{\"{o}}m.
\newblock {\"{U}}ber die praktische {A}ufl{\"{o}}sung von {I}ntegralgleichungen
  mit {A}nwendungen auf {R}andwertaufgaben.
\newblock \emph{Acta Mathematica}, 1930.

\bibitem[Oglic and G{\"{a}}rtner(2016)]{OglicG16}
Dino Oglic and Thomas G{\"{a}}rtner.
\newblock Greedy feature construction.
\newblock In \emph{Advances in Neural Information Processing Systems 29}, pages
  3945--3953. Curran Associates, Inc., 2016.

\bibitem[Pedregosa et~al.(2011)Pedregosa, Varoquaux, Gramfort, Michel, Thirion,
  Grisel, Blondel, Prettenhofer, Weiss, Dubourg, Vanderplas, Passos,
  Cournapeau, Brucher, Perrot, and Duchesnay]{scikit-learn}
F.~Pedregosa, G.~Varoquaux, A.~Gramfort, V.~Michel, B.~Thirion, O.~Grisel,
  M.~Blondel, P.~Prettenhofer, R.~Weiss, V.~Dubourg, J.~Vanderplas, A.~Passos,
  D.~Cournapeau, M.~Brucher, M.~Perrot, and E.~Duchesnay.
\newblock Scikit-learn: Machine learning in {P}ython.
\newblock \emph{Journal of Machine Learning Research}, 12:\penalty0 2825--2830,
  2011.

\bibitem[Rahimi and Recht(2007)]{rahimi2007random}
Ali Rahimi and Benjamin Recht.
\newblock Random features for large-scale kernel machines.
\newblock In \emph{Advances in neural information processing systems}, pages
  1177--1184, 2007.

\bibitem[Rahimi and Recht(2009)]{rahimi2009weighted}
Ali Rahimi and Benjamin Recht.
\newblock Weighted sums of random kitchen sinks: Replacing minimization with
  randomization in learning.
\newblock In \emph{Advances in neural information processing systems}, pages
  1313--1320, 2009.

\bibitem[Rudi and Rosasco(2017)]{rudi2017generalization}
Alessandro Rudi and Lorenzo Rosasco.
\newblock Generalization properties of learning with random features.
\newblock In \emph{Advances in Neural Information Processing Systems}, pages
  3218--3228, 2017.

\bibitem[Rudi et~al.(2015)Rudi, Camoriano, and Rosasco]{rudi2015less}
Alessandro Rudi, Raffaello Camoriano, and Lorenzo Rosasco.
\newblock Less is more: Nystr{\"o}m computational regularization.
\newblock In \emph{Advances in Neural Information Processing Systems}, pages
  1657--1665, 2015.

\bibitem[Rudi et~al.(2017)Rudi, Carratino, and Rosasco]{rudi2017falkon}
Alessandro Rudi, Luigi Carratino, and Lorenzo Rosasco.
\newblock Falkon: An optimal large scale kernel method.
\newblock In \emph{Advances in Neural Information Processing Systems}, pages
  3891--3901, 2017.

\bibitem[Rudi et~al.(2018)Rudi, Calandriello, Carratino, and
  Rosasco]{rudi2018fast}
Alessandro Rudi, Daniele Calandriello, Luigi Carratino, and Lorenzo Rosasco.
\newblock On fast leverage score sampling and optimal learning.
\newblock In \emph{Advances in Neural Information Processing Systems}, pages
  5672--5682, 2018.

\bibitem[Rudin(2017)]{rudin2017fourier}
Walter Rudin.
\newblock \emph{Fourier analysis on groups}.
\newblock Courier Dover Publications, 2017.

\bibitem[Sch{\"o}lkopf and Smola(2001)]{Scholkopf01}
Bernhard Sch{\"o}lkopf and Alexander~J. Smola.
\newblock \emph{Learning with kernels: {S}upport vector machines,
  regularization, optimization, and beyond}.
\newblock MIT Press, 2001.

\bibitem[Sch{\"o}lkopf et~al.(2004)Sch{\"o}lkopf, Tsuda, and
  Vert]{scholkopf2004kernel}
Bernhard Sch{\"o}lkopf, Koji Tsuda, and Jean-Philippe Vert.
\newblock \emph{Kernel methods in computational biology}.
\newblock MIT press, 2004.

\bibitem[Smola and Sch\"olkopf(2000)]{Smola00}
Alexander~J. Smola and Bernhard Sch\"olkopf.
\newblock Sparse greedy matrix approximation for machine learning.
\newblock In \emph{Proceedings of the 17th International Conference on Machine
  Learning}, 2000.

\bibitem[Sriperumbudur and Szab{\'o}(2015)]{sriperumbudur2015optimal}
Bharath Sriperumbudur and Zolt{\'a}n Szab{\'o}.
\newblock Optimal rates for random {F}ourier features.
\newblock In \emph{Advances in Neural Information Processing Systems}, pages
  1144--1152, 2015.

\bibitem[Steinwart and Christmann(2008)]{steinwart2008support}
Ingo Steinwart and Andreas Christmann.
\newblock \emph{Support vector machines}.
\newblock Springer Science \& Business Media, 2008.

\bibitem[Sun et~al.(2018)Sun, Gilbert, and Tewari]{sun2018but}
Yitong Sun, Anna Gilbert, and Ambuj Tewari.
\newblock But how does it work in theory? linear svm with random features.
\newblock In \emph{Advances in Neural Information Processing Systems}, pages
  3379--3388, 2018.

\bibitem[Sutherland and Schneider(2015)]{sutherland2015error}
Dougal~J Sutherland and Jeff Schneider.
\newblock On the error of random {F}ourier features.
\newblock In \emph{Proceedings of the Thirty-First Conference on Uncertainty in
  Artificial Intelligence}, pages 862--871. AUAI Press, 2015.

\bibitem[Tropp(2015)]{tropp2015introduction}
Joel~A Tropp.
\newblock An introduction to matrix concentration inequalities.
\newblock \emph{Foundations and Trends{\textregistered} in Machine Learning},
  8\penalty0 (1-2):\penalty0 1--230, 2015.

\bibitem[Tsybakov et~al.(2004)]{tsybakov2004optimal}
Alexander~B Tsybakov et~al.
\newblock Optimal aggregation of classifiers in statistical learning.
\newblock \emph{The Annals of Statistics}, 32\penalty0 (1):\penalty0 135--166,
  2004.

\bibitem[Williams and Seeger(2001)]{Williams01}
Christopher K.~I. Williams and Matthias Seeger.
\newblock Using the {N}ystr\"{o}m method to speed up kernel machines.
\newblock In \emph{Advances in Neural Information Processing Systems 13}. 2001.

\bibitem[Yang et~al.(2012)Yang, Li, Mahdavi, Jin, and Zhou]{yang2012nystrom}
Tianbao Yang, Yu-Feng Li, Mehrdad Mahdavi, Rong Jin, and Zhi-Hua Zhou.
\newblock Nystr{\"o}m method vs random {F}ourier features: A theoretical and
  empirical comparison.
\newblock In \emph{Advances in neural information processing systems}, pages
  476--484, 2012.

\bibitem[Zhang et~al.(2015)Zhang, Duchi, and Wainwright]{zhang2015divide}
Yuchen Zhang, John Duchi, and Martin Wainwright.
\newblock Divide and conquer kernel ridge regression: A distributed algorithm
  with minimax optimal rates.
\newblock \emph{The Journal of Machine Learning Research}, 16\penalty0
  (1):\penalty0 3299--3340, 2015.

\end{thebibliography}
}
%\cleardoublepage

\newpage
\appendix
\section{Bernstein Inequality}
The next lemma is the matrix Bernstein inequality, cited from \cite[Lemma 27]{avron2017random} which is a restatement of Corollary 7.3.3 in \cite{tropp2015introduction} with some fix in the typos.
\begin{lma}\citep[Bernstein inequality,][Corollary 7.3.3]{tropp2015introduction}\label{apn:matx_con}
Let $\mathbf{R}$ be a fixed $d_1 \times d_2$ matrix over the set of complex/real numbers. Suppose that $\{\mathbf{R}_1,\cdots,\mathbf{R}_n\}$ is an independent and identically distributed sample of $d_1 \times d_2$ matrices such that \[\mathbb{E}[\mathbf{R}_i] = \mathbf{R} \qquad \text{and} \qquad \|\mathbf{R}_i\|_2 \leq L,\]
where $L>0$ is a constant independent of the sample.
Furthermore, let $\mathbf{M}_1, \mathbf{M}_2$ be semidefinite upper bounds for the matrix-valued variances 
\begin{align*}
\begin{aligned}
& \mathrm{Var}_1[\mathbf{R}_i] \preceq \mathbb{E}[\mathbf{R}_i\mathbf{R}_i^{T}] \preceq \mathbf{M}_1 & \\
& \mathrm{Var}_2[\mathbf{R}_i] \preceq \mathbb{E}[\mathbf{R}_i^{T}\mathbf{R}_i]\preceq \mathbf{M}_2. &
\end{aligned}
\end{align*}
Let $m = \max(\|\mathbf{M}_1\|_2,\|\mathbf{M}_2\|_2)$ and $d =\frac{\text{Tr}(\mathbf{M}_1)+ \text{Tr}(\mathbf{M}_2)}{m}.$ 
Then, for $\epsilon \geq \sqrt{m/n}+2L/3n$, we can bound \[\bar{\mathbf{R}}_n = \frac{1}{n}\sum_{i=1}^{n}\mathbf{R}_i\] around its mean using the concentration inequality \[P(\|\bar{\mathbf{R}}_n - \mathbf{R}\|_2 \geq \epsilon) \leq 4d\exp\Bigg(\frac{-n\epsilon^2/2}{m+2L\epsilon/3}\Bigg).\]
\end{lma}

%Often in analysing the learning risk, we need the Hoeffding lemma due to \cite{hoeffding1963probability} to provide an upper bound, we state it next: 
% \begin{lma}\citep{hoeffding1963probability}\label{apn:hoef_lma}
% Let ${x_1,\cdots,x_n}$ be independent and bounded random variables with $x_i \in [a,b]$ for all $i$, where $-\infty < a \leq b < \infty$. If we denote with $s_n = \frac{\sum_{i=1}^{n}x_i}{n}$, then for all $\delta \in (0,1)$, with probability $1-\delta$, it holds that 
% \begin{IEEEeqnarray}{rCl}
% s_n \leq \mathbb{E}[s_n] + (b-a)\sqrt{\frac{\log(1/\delta)}{2n}}. \nonumber
% \end{IEEEeqnarray} 
% \end{lma}

%\section{Upper bound on the approximation function norm}\label{apn:upp_bound}

\section{Proof of Lemma \ref{func_appx_opm}}\label{krr_risk_prof}
The following two lemmas are required for our proof of Lemma \ref{func_appx_opm}, presented subsequently.

\begin{lma}\label{apn:func_con}
Suppose that the assumptions from Lemma \ref{func_appx_opm} hold and let $\epsilon \geq \sqrt{\frac{m}{s}}+\frac{2L}{3s}$ with constants $m$ and $L$ (see the proof for explicit definition). If the number of features
\begin{IEEEeqnarray}{rCl}
s \geq  d_{\tilde{l}}(\frac{1}{\epsilon^2}+\frac{2}{3\epsilon})\log\frac{16d_{\mathbf{K}}^{\lambda}}{\delta}, \nonumber
\end{IEEEeqnarray}
then for all $\delta \in (0,1)$, with probability greater than $1-\delta$,
\begin{IEEEeqnarray}{rCl}
-\epsilon\mathbf{I}\preceq (\mathbf{K}+n\lambda \mathbf{I})^{-\frac{1}{2}}(\tilde{\mathbf{K}}-\mathbf{K})(\mathbf{K}+n\lambda \mathbf{I})^{-\frac{1}{2}} \preceq \epsilon\mathbf{I}. \nonumber
\end{IEEEeqnarray}
\end{lma}
\begin{proof}
Following the derivations in~\citet{avron2017random}, we utilize the matrix Bernstein concentration inequality to prove the result. More specifically, we observe that
\begin{align*}
\begin{aligned}
&(\mathbf{K}+n\lambda \mathbf{I})^{-\frac{1}{2}}\tilde{\mathbf{K}}(\mathbf{K}+n\lambda \mathbf{I})^{-\frac{1}{2}} =&\\
&\frac{1}{s}\sum_{i=1}^s(\mathbf{K}+n\lambda \mathbf{I})^{-\frac{1}{2}}\mathbf{z}_{q,v_i}(\mathbf{x})\mathbf{z}_{q,v_i}(\mathbf{x})^{T}(\mathbf{K}+n\lambda \mathbf{I})^{-\frac{1}{2}}=&\\
&\frac{1}{s}\sum_{i=1}^s\mathbf{R}_i =: \bar{\mathbf{R}}_s, &
\end{aligned}
\end{align*}
with
\begin{IEEEeqnarray}{rCl}
\mathbf{R}_i = (\mathbf{K}+n\lambda \mathbf{I})^{-\frac{1}{2}}\mathbf{z}_{q,v_i}(\mathbf{x})\mathbf{z}_{q,v_i}(\mathbf{x})^{T}(\mathbf{K}+n\lambda \mathbf{I})^{-\frac{1}{2}}.\nonumber
\end{IEEEeqnarray}
Now, observe that 
$$\mathbf{R} = \mathbb{E}[\mathbf{R}_i] =(\mathbf{K}+n\lambda \mathbf{I})^{-\frac{1}{2}}\mathbf{K}(\mathbf{K}+n\lambda \mathbf{I})^{-\frac{1}{2}}.$$ The operator norm of $\mathbf{R}_i$ is equal to%\|\mathbf{R}_i\|_2 = 
\begin{IEEEeqnarray}{rCl}
\|(\mathbf{K}+n\lambda \mathbf{I})^{-\frac{1}{2}}\mathbf{z}_{q,v_i}(\mathbf{x})\mathbf{z}_{q,v_i}(\mathbf{x})^{T}(\mathbf{K}+n\lambda \mathbf{I})^{-\frac{1}{2}}\|_2. \nonumber
\end{IEEEeqnarray}
As $\mathbf{z}_{q,v_i}(\mathbf{x})\mathbf{z}_{q,v_i}(\mathbf{x})^{T}$ is a rank one matrix, we have that the operator norm of this matrix is equal to its trace, i.e.,
\begin{align*}
\begin{aligned}
& \|\mathbf{R}_i\|_2 =&\\
&\text{Tr}((\mathbf{K}+n\lambda \mathbf{I})^{-\frac{1}{2}}\mathbf{z}_{q,v_i}(\mathbf{x})\mathbf{z}_{q,v_i}(\mathbf{x})^{T}(\mathbf{K}+n\lambda \mathbf{I})^{-\frac{1}{2}}) =&\\
&\frac{p(v_i)}{q(v_i)}\text{Tr}((\mathbf{K}+n\lambda \mathbf{I})^{-\frac{1}{2}}\mathbf{z}_{v_i}(\mathbf{x})\mathbf{z}_{v_i}(\mathbf{x})^{T}(\mathbf{K}+n\lambda \mathbf{I})^{-\frac{1}{2}}) = &\\
& \frac{p(v_i)}{q(v_i)}\text{Tr}(\mathbf{z}_{v_i}(\mathbf{x})^{T}(\mathbf{K}+n\lambda \mathbf{I})^{-1}\mathbf{z}_{v_i}(\mathbf{x})) =& \\
& \frac{l_{\lambda}(v_i)}{q(v_i)} =: L_i \quad \text{ and } \quad L_q \coloneqq \sup_i\  L_i.&
\end{aligned}
\end{align*}
Observe that $L_q= \sup_{i} L_i = \sup_{i}\frac{l_{\lambda}(v_i)}{q(v_i)} \leq \sup_i \frac{\tilde{l}(v_i)}{q(v_i)} = d_{\tilde{l}}$. On the other hand,
\begin{align*}
\begin{aligned}
& \mathbf{R}_i\mathbf{R}_i^{T} =&\\
&(\mathbf{K}+n\lambda \mathbf{I})^{-\frac{1}{2}}\mathbf{z}_{q,v_i}(\mathbf{x})\mathbf{z}_{q,v_i}(\mathbf{x})^{T}(\mathbf{K}+n\lambda \mathbf{I})^{-1}\mathbf{z}_{q,v_i}(\mathbf{x}) &\\
&\cdot\mathbf{z}_{q,v_i}(\mathbf{x})^{T}(\mathbf{K}+n\lambda \mathbf{I})^{-\frac{1}{2}} = &\\
&\frac{p(v_i)l_{\lambda}(v_i)}{q^2(v_i)}(\mathbf{K}+n\lambda \mathbf{I})^{-\frac{1}{2}}\mathbf{z}_{v_i}(\mathbf{x})\mathbf{z}_{v_i}(\mathbf{x})^{T}(\mathbf{K}+n\lambda \mathbf{I})^{-\frac{1}{2}} \preceq &\\
& \frac{\tilde{l}(v_i)}{q(v_i)}\frac{p(v_i)}{q(v_i)}(\mathbf{K}+n\lambda \mathbf{I})^{-\frac{1}{2}}\mathbf{z}_{v_i}(\mathbf{x})\mathbf{z}_{v_i}(\mathbf{x})^{T}(\mathbf{K}+n\lambda \mathbf{I})^{-\frac{1}{2}} =&\\
&d_{\tilde{l}}\frac{p(v_i)}{q(v_i)}(\mathbf{K}+n\lambda \mathbf{I})^{-\frac{1}{2}}\mathbf{z}_{v_i}(\mathbf{x})\mathbf{z}_{v_i}(\mathbf{x})^{T}(\mathbf{K}+n\lambda \mathbf{I})^{-\frac{1}{2}} . &
\end{aligned}
\end{align*}
From the latter inequality, we obtain that
\begin{IEEEeqnarray}{rCl}
\mathbb{E}[\mathbf{R}_i\mathbf{R}_i^{T}] &\preceq& d_{\tilde{l}}(\mathbf{K}+n\lambda \mathbf{I})^{-\frac{1}{2}}\mathbf{K}(\mathbf{K}+n\lambda \mathbf{I})^{-\frac{1}{2}}\nonumber  =: \mathbf{M}_1.\nonumber
\end{IEEEeqnarray}
We also have the following two equalities
\begin{IEEEeqnarray}{rCl}
m &=& \|\mathbf{M}_1\|_2 =  d_{\tilde{l}}\frac{\lambda_1}{\lambda_1+n\lambda} =:  d_{\tilde{l}}d_1\nonumber\\
d &=& \frac{2~\text{Tr}(\mathbf{M}_1)}{m} = 2\frac{\lambda_1+n\lambda}{\lambda_1}d_{\mathbf{K}}^{\lambda} = 2d_1^{-1}d_{\mathbf{K}}^{\lambda}.\nonumber
\end{IEEEeqnarray}
We are now ready to apply the matrix Bernstein concentration inequality \citep[Corollary 7.3.3]{tropp2015introduction}. More specifically, for $\epsilon \geq \sqrt{m/s}+2L/3s$ and for all $\delta \in(0,1)$, with probability $1-\delta$, we have that
\begin{IEEEeqnarray}{rCl}
\text{P}(\|\bar{\mathbf{R}}_{s}-\mathbf{R}\|_2 \geq \epsilon) &\leq& 4d \exp\left( \frac{-s\epsilon^2/2}{m+2L\epsilon/3} \right) \nonumber\\
&\leq& 8d_1^{-1}d_{\mathbf{K}}^{\lambda}\exp\left( \frac{-s\epsilon^2/2}{ d_{\tilde{l}}d_1+ d_{\tilde{l}}2\epsilon/3} \right)\nonumber\\
&\leq& 16 d_{\mathbf{K}}^{\lambda} \exp\left(\frac{-s\epsilon^2}{ d_{\tilde{l}}(1+2\epsilon/3)}\right) \leq \delta .\nonumber
\end{IEEEeqnarray}
In the third line, we have used the assumption that $n\lambda \leq \lambda_1$ and, consequently, $d_1 \in [1/2,1)$.
\end{proof}

\textbf{Remark}: We note here that the two considered sampling strategies lead to two different results. In particular, if we let $\tilde{l}(v) = l_{\lambda}(v)$ then $q(v) = l_{\lambda}(v)/d_{\mathbf{K}}^{\lambda}$, i.e., we are sampling proportional to the ridge leverage scores. Thus, the leverage weighted random Fourier features sampler requires
\begin{IEEEeqnarray}{rCl} 
s \geq d_{\mathbf{K}}^{\lambda}(\frac{1}{\epsilon^2}+\frac{2}{3\epsilon})\log\frac{16d_{\mathbf{K}}^{\lambda}}{\delta}.\label{apn:fea_opm}
\end{IEEEeqnarray}
Alternatively, we can opt for the plain random Fourier feature sampling strategy by taking $\tilde{l}(v) = z_0^2p(v)/\lambda$, with $l_{\lambda}(v)\leq z_0^2p(v)/\lambda$. Then, plain random Fourier features sampling requires
\begin{IEEEeqnarray}{rCl}
s \geq \frac{z_0^2}{\lambda}(\frac{1}{\epsilon^2}+\frac{2}{3\epsilon})\log\frac{16d_{\mathbf{K}}^{\lambda}}{\delta}.\label{apn:fea_orig}
\end{IEEEeqnarray}

Thus, the leverage weighted random Fourier features sampling scheme can dramatically change the number of features required to achieve a predefined approximation error in the operator norm.

\begin{lma}\label{apn:fun_krl}
Let $f \in \mathcal{H}$, where $\mathcal{H}$ is the reproducing kernel Hilbert space associated with a kernel $k$. Recall we have assumed that $\|f\|_{\mathcal{H}} \leq 1, \forall f$ and $\mathbf{f}_x = [f(x_1),\cdots,f(x_n)]^T$. Let $\mathbf{K}$ be the Gram-matrix of the kernel $k$ given by the provided set of instances. Then,
\begin{IEEEeqnarray}{rCl}
\mathbf{f}_x^{T}\mathbf{K}^{-1}\mathbf{f}_x \leq 1. \nonumber
\end{IEEEeqnarray}
\end{lma}

\begin{proof}
Recall that a function $f \in \mathcal{H}$ can be expressed as:
\begin{IEEEeqnarray}{rCl}
f(x) = \int_{\mathcal{V}}g(v)z(v,x)p(v)dv \qquad (\forall x \in \mathcal{X}),
\end{IEEEeqnarray}
where $g \in L_2(d\tau)$ is a real-valued function with $\|f\|_{\mathcal{H}}$ equal to the minimum of $\|g\|_{L_2(d\tau)}$, over all possible decompositions of $f$. For a vector $\mathbf{a} \in \mathbb{R}^n$, we have that
\begin{IEEEeqnarray}{rCl}
\mathbf{a}^{T}\mathbf{f}_x\mathbf{f}_x^T\mathbf{a} &=& \Big(\mathbf{f}_x^{T}\mathbf{a}\Big)^2 = \Big(\sum_{i=1}^n a_if(x_i) \Big)^2 \nonumber \\
&=& \Big(\sum_{i=1}^n a_i\int_{\mathcal{V}}g(v)z(v,x_i)d\tau(v) \Big)^2\nonumber \\
&=&\Big(\int_{\mathcal{V}}g(v)\mathbf{z}_v(\mathbf{x}) ^{T}\mathbf{a}~d\tau(v)\Big)^2\nonumber\\
&\leq& \int_{\mathcal{V}}g(v)^2 d\tau(v)\int_{\mathcal{V}}(\mathbf{z}_v(\mathbf{x}) ^T\mathbf{a})^2~d\tau(v)\nonumber\\
&=& \int_{\mathcal{V}}\mathbf{a}^{T}\mathbf{z}_v(\mathbf{x})\mathbf{z}_v(\mathbf{x})^{T}\mathbf{a}~d\tau(v)\nonumber\\
&=& \mathbf{a}^{T}\ \int_{\mathcal{V}}\mathbf{z}_v(\mathbf{x})\mathbf{z}_v(\mathbf{x})^{T}~d\tau(v)\ \mathbf{a}\nonumber\\
&=& \mathbf{a}^{T}\mathbf{K}\mathbf{a}.\nonumber
\end{IEEEeqnarray}
The third equality is due to the fact that, for all $f \in \mathcal{H}$, we have that $f(x) = \int_{\mathcal{V}}g(v)z(v,x)p(v)dv$ ($\forall x\in \mathcal{X}$) and $$\|f\|_{\mathcal{H}} = \min_{\Big\{g \ \mid\ f(x)=\int_{\mathcal{V}}g(v)z(v,x)p(v)dv \Big\}}\ \|g\|_{L_2(d\tau)}.$$ The first inequality, on the other hand, follows from the Cauchy-Schwarz inequality. The bound implies that $\mathbf{f}_x\mathbf{f}_x^{T} \preceq \mathbf{K}$ and, consequently, we derive $\mathbf{f}_x^{T}\mathbf{K}^{-1}\mathbf{f}_x \leq 1$.
\end{proof}
Now we are ready to prove Lemma \ref{func_appx_opm}.
\funcAppxOpmRe*
\begin{proof}
For any $f \in \mathcal{H}$ with $\|f\|_{\mathcal{H}} \leq 1$, we write the following optimization problem:
\begin{IEEEeqnarray}{rCl}
\frac{1}{n}\|\mathbf{f}_x- \mathbf{Z}_q\beta\|_2^2 + s\lambda\|\beta\|_2^2. \label{apn:op_goal}
\end{IEEEeqnarray}
The minimizer can be computed as:
\begin{IEEEeqnarray}{rCl}
\beta &=&\frac{1}{s}(\frac{1}{s}\mathbf{Z}_q^T\mathbf{Z}_q+ n\lambda \mathbf{I})^{-1}\mathbf{Z}_q^T\mathbf{f}_x\nonumber\\
&=&\frac{1}{s}\mathbf{Z}_q^{T}(\frac{1}{s}\mathbf{Z}_q\mathbf{Z}_q^{T}+ n\lambda \mathbf{I})^{-1}\mathbf{f}_x \nonumber \\
&=& \frac{1}{s}\mathbf{Z}_q^{T}(\tilde{\mathbf{K}}+n\lambda\mathbf{I})^{-1}\mathbf{f}_x, \nonumber 
\end{IEEEeqnarray}
where the second equality follows from the Woodbury inversion lemma.

Substituting $\beta$ into Eq.~(\ref{apn:op_goal}), we transform the first part as
\begin{IEEEeqnarray}{rCl}
\frac{1}{n}\|\mathbf{f}_x- \mathbf{Z}_q\beta\|_2^2 &=& \frac{1}{n}\|\mathbf{f}_x-\frac{1}{s}\mathbf{Z}_q\mathbf{Z}_q^{T}(\tilde{\mathbf{K}}+n\lambda\mathbf{I})^{-1}\mathbf{f}_x \|_2^2 \nonumber\\
&=&\frac{1}{n}\|\mathbf{f}_x-\tilde{\mathbf{K}}(\tilde{\mathbf{K}}+n\lambda\mathbf{I})^{-1}\mathbf{f}_x \|_2^2\nonumber\\
&=& \frac{1}{n}\|n\lambda(\tilde{\mathbf{K}}+n\lambda\mathbf{I})^{-1}\mathbf{f}_x\|_2^2\nonumber \\
&=& n\lambda^2\mathbf{f}_x^{T}(\tilde{\mathbf{K}}+n\lambda\mathbf{I})^{-2}\mathbf{f}_x.\nonumber
\end{IEEEeqnarray}
On the other hand, the second part can be transformed as
\begin{IEEEeqnarray}{rCl}
s\lambda\|\beta\|_2^2 &=& s\lambda\frac{1}{s^2}\mathbf{f}_x^{T}(\tilde{\mathbf{K}}+n\lambda\mathbf{I})^{-1}\mathbf{Z}_q\mathbf{Z}_q^{T}(\tilde{\mathbf{K}}+n\lambda\mathbf{I})^{-1}\mathbf{f}_x\nonumber\\
&=& \lambda\mathbf{f}_x^{T}(\tilde{\mathbf{K}}+n\lambda\mathbf{I})^{-1}\tilde{\mathbf{K}}(\tilde{\mathbf{K}}+n\lambda\mathbf{I})^{-1}\mathbf{f}_x\nonumber\\
&=& \lambda\mathbf{f}_x^{T}(\tilde{\mathbf{K}}+n\lambda\mathbf{I})^{-1}(\tilde{\mathbf{K}}+n\lambda\mathbf{I})(\tilde{\mathbf{K}}+n\lambda\mathbf{I})^{-1}\mathbf{f}_x - n\lambda^2\mathbf{f}_x^{T}(\tilde{\mathbf{K}}+n\lambda\mathbf{I})^{-2}\mathbf{f}_x\nonumber\\
&=&\lambda\mathbf{f}_x^{T}(\tilde{\mathbf{K}}+n\lambda\mathbf{I})^{-1}\mathbf{f}_x-n\lambda^2\mathbf{f}_x^{T}(\tilde{\mathbf{K}}+n\lambda\mathbf{I})^{-2}\mathbf{f}_x. \nonumber
\end{IEEEeqnarray}
Now, summing up the first and the second part, we deduce
\begin{align*}
\begin{aligned}
&\frac{1}{n}\|\mathbf{f}_x- \mathbf{Z}_q\beta\|_2^2 + s\lambda\|\beta\|_2^2    =\lambda\mathbf{f}_x^{T}(\tilde{\mathbf{K}}+n\lambda\mathbf{I})^{-1}\mathbf{f}_x &\\
&  =\lambda\mathbf{f}_x^{T}(\mathbf{K} + n\lambda\mathbf{I}+\tilde{\mathbf{K}}-\mathbf{K})^{-1}\mathbf{f}_x &\\
&  =\lambda\mathbf{f}_x^{T}(\mathbf{K} + n\lambda\mathbf{I})^{-\frac{1}{2}}\left(\mathbf{I}+(\mathbf{K} + n\lambda\mathbf{I})^{-\frac{1}{2}}(\tilde{\mathbf{K}}-\mathbf{K})(\mathbf{K}+n\lambda\mathbf{I})^{-\frac{1}{2}}\right)^{-1}(\mathbf{K} + n\lambda\mathbf{I})^{-\frac{1}{2}}\mathbf{f}_x. &
\end{aligned}
\end{align*}

From Lemma \ref{apn:func_con}, it follows that when
\begin{IEEEeqnarray}{rCl}
s \geq  d_{\tilde{l}}(\frac{1}{\epsilon^2}+\frac{2}{3\epsilon})\log\frac{16d_{\mathbf{K}}^{\lambda}}{\delta} \nonumber
\end{IEEEeqnarray}
then $(\mathbf{K} + n\lambda\mathbf{I})^{-\frac{1}{2}}(\tilde{\mathbf{K}}-\mathbf{K})(\mathbf{K}+n\lambda\mathbf{I})^{-\frac{1}{2}} \succeq -\epsilon \mathbf{I}$. %where we let $\epsilon = 1/2$.

We can now upper bound the objective function as follows (with $\epsilon = 1/2$):
\begin{align*}
\begin{aligned}
& \lambda\mathbf{f}_x^{T}(\tilde{\mathbf{K}}+n\lambda\mathbf{I})^{-1}\mathbf{f}_x && \leq \lambda\mathbf{f}_x^{T}(\mathbf{K} + n\lambda\mathbf{I})^{-\frac{1}{2}}(1-\epsilon)^{-1} (\mathbf{K} + n\lambda\mathbf{I})^{-\frac{1}{2}}\mathbf{f}_x  &\\
& && =(1-\epsilon)^{-1}\lambda\mathbf{f}_x^{T}(\mathbf{K}+n\lambda\mathbf{I})^{-1}\mathbf{f}_x \leq  (1-\epsilon)^{-1}\lambda\mathbf{f}_x^{T}\mathbf{K}^{-1}\mathbf{f}_x \leq 2\lambda , &
\end{aligned}
\end{align*}
where in the last inequality we have used Lemma \ref{apn:fun_krl}. Moreover, we have that 
\begin{align*}
\begin{aligned}
& s\|\beta\|_2^{2} && =  \mathbf{f}_x^{T}(\tilde{\mathbf{K}}+n\lambda\mathbf{I})^{-1}\mathbf{f}_x-n\lambda\mathbf{f}_x^{T}(\tilde{\mathbf{K}}+n\lambda\mathbf{I})^{-2}\mathbf{f}_x &\\
& && \leq \mathbf{f}_x^{T}(\tilde{\mathbf{K}}+n\lambda\mathbf{I})^{-1}\mathbf{f}_x \leq (1-\epsilon)^{-1} \mathbf{f}_x^{T}\mathbf{K}^{-1}\mathbf{f}_x \leq  2 .&
\end{aligned}
\end{align*}
Hence, the squared norm of our approximated function is bounded by $\|\tilde{f}\|_{\tilde{\mathcal{H}}}^2 \leq s\|\beta\|_2^2 \leq 2$. As such, problem (\ref{apn:op_goal}) can now be written as $\min_{\beta}(1/n)\|\mathbf{f}_{x}-\tilde{\mathbf{f}}_x\|_2^2 $ subject to $\|\tilde{f}\|_{\tilde{\mathcal{H}}}^2 \leq s\|\beta\|_2^2 \leq 2$, which is equivalent to \[\ \inf\nolimits_{\|\tilde{f}\|_{\tilde{\mathcal{H}}} \leq \sqrt{2}}\ \frac{1}{n}\|\mathbf{f}_x- \tilde{\mathbf{f}}_x\|_2^2,\] and we have shown that this can be upper bounded by $2\lambda$. Since we are approximating any $f\in \mathcal{H}$ with $\|f\|_{\mathcal{H}} \leq 1$, this can further be written as\[\sup\nolimits_{\|f\|_{\mathcal{H}}\leq 1 }\inf\nolimits_{\|\tilde{f}\|_{\tilde{\mathcal{H}}} \leq \sqrt{2}}\ \frac{1}{n}\|\mathbf{f}_x- \tilde{\mathbf{f}}_x\|_2^2 \leq 2\lambda. \] 
\end{proof}

\section {Proof of Lemma \ref{triangle_lma}} \label{apn:triangle_lma_pf}
\TriangleLma*
\begin{proof}
By definition, $\tilde{\mathbf{f}}^{\lambda}_x$ has the format as $\tilde{\mathbf{f}}^{\lambda}_x = \mathbf{Z}_q \tilde{\beta}^{\lambda}$, where $\tilde{\beta}^{\lambda} \in \mathbb{R}^s$. In addition, definition of $\tilde{f}^{\lambda}$ can be reparametrized by the following optimization problem:
\begin{IEEEeqnarray}{rCl}
\tilde{\beta}^{\lambda} := \min\nolimits_{\tilde{\beta}}\frac{1}{n}\|\hat{\mathbf{f}}^{\lambda}_x - \mathbf{Z}_q\tilde{\beta}\|^2_2 + s\lambda \|\tilde{\beta}\|. \label{eqn:KRR_proj_RFF}
\end{IEEEeqnarray}
This gives the closed-form solution of $\tilde{\beta}^{\lambda} = \frac{1}{s}\mathbf{Z}_q^{T}(\frac{1}{s}\mathbf{Z}_q\mathbf{Z}_q^{T}+n\lambda\mathbf{I})^{-1}\hat{\mathbf{f}}^{\lambda}_x$. As a result, we have \[\tilde{\mathbf{f}}^{\lambda}_x = \frac{1}{s}\mathbf{Z}_q\mathbf{Z}_q^{T}(\frac{1}{s}\mathbf{Z}_q\mathbf{Z}_q^{T}+n\lambda\mathbf{I})^{-1}\hat{\mathbf{f}}^{\lambda}_x = \tilde{\mathbf{K}}(\tilde{\mathbf{K}} + n\lambda I)^{-1}\hat{\mathbf{f}}^{\lambda}_x.\] Now recall $\hat{\mathbf{f}}^{\lambda}_x$ is the in-sample prediction of the KRR estimator $\hat{f}^{\lambda}$, so it can be written as $\hat{\mathbf{f}}^{\lambda}_x = \mathbf{K}(\mathbf{K}+ n\lambda I)^{-1}Y$. As a result, we have the following:
\begin{IEEEeqnarray}{rCl}
\frac{1}{n}\langle Y-\hat{\mathbf{f}}^{\lambda}_x, \hat{\mathbf{f}}^{\lambda}_x -\tilde{\mathbf{f}}^{\lambda}_x \rangle  &= & \frac{1}{n}\langle Y - \hat{\mathbf{f}}^{\lambda}_x, \hat{\mathbf{f}}^{\lambda}_x-\tilde{\mathbf{K}}(\tilde{\mathbf{K}} + n\lambda I)^{-1}\hat{\mathbf{f}}^{\lambda}_x \rangle \nonumber \\
& =& \frac{1}{n}\langle Y - \mathbf{K}(\mathbf{K}+ n\lambda I)^{-1}Y, (I - \tilde{\mathbf{K}}(\tilde{\mathbf{K}} + n\lambda I)^{-1}) \hat{\mathbf{f}}^{\lambda}_x \rangle\nonumber\\
& =& \frac{1}{n}Y^T(I- \mathbf{K}(\mathbf{K}+ n\lambda I)^{-1})(I - \tilde{\mathbf{K}}(\tilde{\mathbf{K}} + n\lambda I)^{-1}) \hat{\mathbf{f}}^{\lambda}_x\nonumber\\
& \leq &  \frac{1}{n}Y^T(I- \mathbf{K}(\mathbf{K}+ n\lambda I)^{-1})\hat{\mathbf{f}}^{\lambda}_x \label{eqn:rff_proj_ineq}\\
& = & \lambda Y^T (\mathbf{K}+ n\lambda I)^{-1}\hat{\mathbf{f}}^{\lambda}_x\nonumber\\
& = & \lambda Y^T(\mathbf{K}+ n\lambda I)^{-1}\mathbf{K} \mathbf{K}^{-1}\hat{\mathbf{f}}^{\lambda}_x \nonumber\\
& = & \lambda \hat{\mathbf{f}}^{\lambda T}_x  \mathbf{K}^{-1}\hat{\mathbf{f}}^{\lambda}_x \leq \lambda \nonumber
\end{IEEEeqnarray}
Note that in Eq.~(\ref{eqn:rff_proj_ineq}), we have used the fact that\[\|I - \tilde{\mathbf{K}}(\tilde{\mathbf{K}} + n\lambda I)^{-1}\|_2  \leq 1.\] For the last inequality, since $\hat{f}^{\lambda} \in \mathcal{H}$, we employ Lemma \ref{apn:fun_krl}.
\end{proof}

\section{Property of Square Loss}\label{apn:sub_root_sec}
In this section, we state the property of square loss function.
\begin{lma}\label{apn:square_loss}\citep[Section 5.2]{bartlett2005local}
Let $l$ be the squared error loss function and $\mathcal{H}$ a convex and uniformly bounded hypothesis space. Assume that for every probability distribution $P$ in a class of data-generating distributions, there is an $f_{\mathcal{H}} \in \mathcal{H}$ such that $\mathbb{E}(l_{f_{\mathcal{H}}}) = \inf_{f \in \mathcal{H}}\ \mathbb{E}(l_f)$. Then, there exists a constant $B \geq 1$ such that for all $f \in \mathcal{H}$ and for every probability distribution $P$
\begin{IEEEeqnarray}{rCl}
\mathbb{E}(f-f_{\mathcal{H}})^2 \leq B \mathbb{E}(l_f-l_{f_{\mathcal{H}}}) \label{apn:loss_inq}
\end{IEEEeqnarray} 
\end{lma}

\section{Proof of Lemma \ref{apn:local_kernel_risk}} \label{apn:sub_root_reg}
\LocalKernelRisk*
\begin{proof}
It is easy to see that $\mathbb{E}(l_f^2) \leq \mathbb{E}(l_f)$. Hence, we can apply Lemma \ref{apn:emp_local_rade} to function class $l_{\mathcal{H}}$ and obtain that for all $l_f \in l_\mathcal{H}$ \[\mathbb{E}(l_f) \leq \frac{D}{D-1}\mathbb{E}_nl_f + \frac{6D}{B}\hat{r}^* + \frac{c_3}{n}\log(\frac{1}{\delta}),\] as long as there is a sub-root function $\hat{\psi}_n(r)$ such that
\begin{align}
\hat{\psi}_n(r) \geq c_1 \hat{R}_n\{l_f\in star(l_{\mathcal{H}},0) \mid \mathbb{E}_nl_f^2 \leq r\} + \frac{c_2}{n}\log(\frac{1}{\delta}). \label{apn:subroot_req}
\end{align}
We have previously demonstrated that 
\begin{align}
\begin{aligned}
&c_1 \hat{R}_n\{l_f\in star(l_{\mathcal{H}},0) \mid \mathbb{E}_nl_f^2 \leq r\} + \frac{c_2}{n}\log(\frac{1}{\delta}) \\
\leq ~& 2c_1L\hat{R}_n \Bigg\{f \in \mathcal{H} \mid \mathbb{E}_nf^2 \leq e_1 r\Bigg\} + \frac{c_2}{n}\log(\frac{1}{\delta})\\
\leq ~& 2c_1L\Bigg(\frac{2}{n}\sum_{i=1}^n\min\Bigg\{e_1 r,\hat{\lambda}_i\Bigg\}\Bigg)^{1/2}+\frac{c_2}{n}\log(\frac{1}{\delta}) ~~~~(\text{by Lemma~\ref{apn:local_kernel}}).
\end{aligned}\label{apn:subroot_lh}
\end{align}
Hence, if we choose $\hat{\psi}_n(r)$ to be equal to the right hand side of Eq.~(\ref{apn:subroot_lh}), then $\hat{\psi}_n(r)$ is a sub-root function that satisfies Eq.~(\ref{apn:subroot_req}). Now, the upper bound on the fixed point $\hat{r}^*$ follows from Corollary 6.7 in~\citet{bartlett2005local}.
\end{proof}

\section{Additional Experiments with more features}

We have also added extra experiments where we use more features for the experiments that have not yet converged i.e. \texttt{KINEMATICS} and \texttt{COD-RNA}. In the below we see that only when we increase the number of features up to 1000 we are able to attain comparable performance.

\begin{figure*}[htp]
	\centering
	\begin{subfigure}{0.49\textwidth}
		\centering
		\includegraphics[width=7.cm,height=5.5cm]{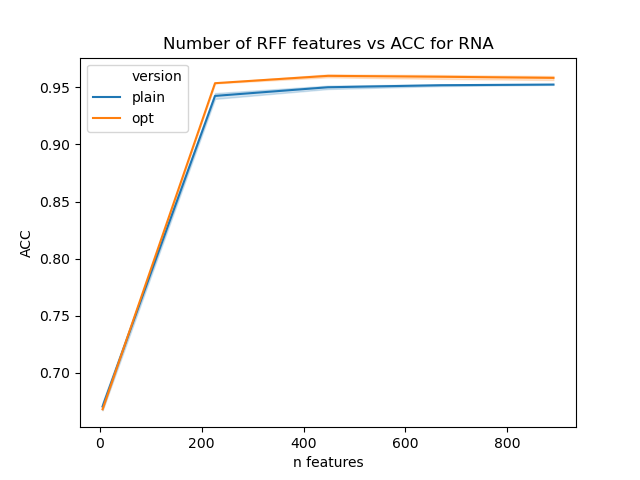}
	\end{subfigure}
	%\hfill
	\centering
	\begin{subfigure}{0.49\textwidth}
		\centering
		\includegraphics[width=7.cm,height=5.5cm]{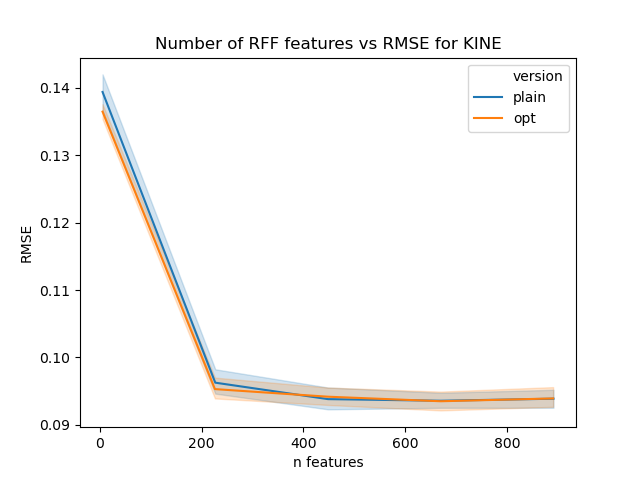}
		%\caption{Optimal Sampling Strategy}
		%\label{krr:fig 2}
	\end{subfigure}
	%\hfill
\end{figure*}

\clearpage
\onecolumn

\end{document}